\documentclass[a4paper, 11pt]{article}

\usepackage{float}
\usepackage[margin=1in]{geometry}

\usepackage[utf8]{inputenc} 
\usepackage[T1]{fontenc}    
\usepackage{hyperref}       
\usepackage{url}            
\usepackage{booktabs}       
\usepackage{amsfonts}       
\usepackage{nicefrac}       
\usepackage{microtype}      
\usepackage{xcolor}         

\usepackage{wrapfig}

\usepackage{graphicx}
\usepackage{amsmath}
\usepackage{amssymb}
\usepackage{booktabs}
\usepackage{times}
\usepackage{epsfig}
\usepackage{graphicx}
\usepackage{amsfonts, amsthm}
\usepackage{comment}
\usepackage{caption}
\usepackage{subcaption}
\usepackage{multirow}
\usepackage{multicol}
\usepackage{enumitem}
\usepackage{setspace}
\usepackage{natbib}
\newcommand{\RB}{\mathbb{R}}
\newcommand{\PB}{\mathbb{P}}

\newcommand{\NB}{\mathbb{N}}
\newcommand{\EB}{\mathbb{E}}

\newcommand{\supp}{\normalfont\text{supp}}

\theoremstyle{plain}
\newtheorem{theorem}{Theorem}[section]

\newtheorem{lemma}[theorem]{Lemma}
\newtheorem{ansatz}[theorem]{Ansatz}

\theoremstyle{definition}
\newtheorem{definition}[theorem]{Definition}

\theoremstyle{remark}

\title{On progressive sharpening, flat minima and generalisation}

\author{%
  Lachlan E. MacDonald\thanks{lemacdonald@protonmail.com} \\
  Mathematical Institute for Data Science\\
  Johns Hopkins University\\
  Baltimore, MD 21218, USA \\
  \and
Jack Valmadre \& Simon Lucey \\
Australian Institute for Machine Learning \\
University of Adelaide \\
Adelaide, SA 5000, Australia
}

\begin{document}

\maketitle

\begin{abstract}
    We present a new approach to understanding the relationship between loss curvature and input-output model behaviour in deep learning.  Specifically, we use existing empirical analyses of the spectrum of deep network loss Hessians to ground an ansatz tying together the loss Hessian and the input-output Jacobian over training samples during the training of deep neural networks. We then prove a series of theoretical results which quantify the degree to which the input-output Jacobian of a model approximates its Lipschitz norm over a data distribution, and deduce a novel generalisation bound in terms of the empirical Jacobian. We use our ansatz, together with our theoretical results, to give a new account of the recently observed progressive sharpening phenomenon, as well as the generalisation properties of flat minima. Experimental evidence is provided to validate our claims.
\end{abstract}

\section{Introduction}

In this paper, we attempt to clarify how the curvature of the loss landscape of a deep neural network is related to the input-output behaviour of the model. Via a single mechanism, our Ansatz \ref{ansatz}, we offer new explanations of both the as-yet unexplained progressive sharpening phenomenon \cite{coheneos} observed in the training of deep neural networks, and the long-speculative relationship between loss curvature and generalisation \cite{hochreiter, keskar, chaudhari, foret}.

The mechanism we propose in Ansatz \ref{ansatz} to mediate the relationship between loss curvature and input-output model behaviour is the input-output Jacobian of the model over the training sample, which is already understood to play a role in determining generalisation \cite{drucker, bartlett2, neyshabur, hoffman, bubeck, gouk, novak2018sensitivity, maetal}.  Our proposal is based on empirical observations made of the eigenspectrum of the Hessian of deep neural networks \cite{papyan1, papyan2, ghorbani}, whose outliers can be attributed to a summand of the Hessian known as the Gauss-Newton matrix. Crucially, the Gauss-Newton matrix is second-order only in the cost function, and solely first-order in the network layers.

The Gauss-Newton matrix is a Gram matrix (i.e. a product $A^TA$ for some matrix $A$), and its conjugate $AA^T$, closely related to the tangent kernel identified in \cite{jacot}, has the same nonzero eigenvalues. Expanding $AA^T$ reveals that it is determined in part by composite input-output layer Jacobians. Thus, insofar as the outlying Hessian eigenvalues are determined by those of the Gauss-Newton matrix, and insofar as input-output layer Jacobians determine the spectrum of the Gauss-Newton matrix via its isospectrality to its conjugate, one expects the largest singular values of the loss Hessian to be closely related to those of the model's input-output Jacobian.

Our contributions in this paper are as follows.
\begin{enumerate}
    \item Based on previous empirical work identifying outlying loss Hessian eigenvalues with those of the Gauss-Newton matrix, we propose an ansatz: that under certain conditions, the largest eigenvalues of the loss Hessian control the growth of the largest singular values of the model's input-output Jacobian.  We focus in particular on the \emph{largest} eigenvalue of the Hessian (the \emph{sharpness}) and the \emph{largest} singular value of the Jacobian (its \emph{spectral norm}, which we abbreviate to simply \emph{norm} in what follows).
    \item In Section \ref{sec:theory} we provide theorems which quantify the extent to which the maximum input-output Jacobian spectral norm of a model over a training set will approximate the Lipschitz norm of the model over the underlying data distribution during training for datasets of practical relevance, including those generated by a generative adversarial network (GAN) or implicit neural function.
    \item In Section \ref{sec:progressivesharpening}, we provide a theorem which gives a data-dependent, high probability lower bound on the empirical Jacobian norm of a model that grows during any effective training procedure. We combine this result with Ansatz \ref{ansatz} to give a new account of progressive sharpening, which enables us to change the severity of sharpening by scaling inputs and outputs. We report on classification experiments that validate our account.
    \item In Section \ref{sec:generalisation}, we provide a novel bound on the generalisation gap of the model in terms of the empirical Jacobian norm. Synthesising this result with Ansatz \ref{ansatz}, we argue that low loss curvature implies good generalisation only insofar as low loss curvature implies small empirical Jacobian norm. We report on experiments measuring the effect of hyperparameters such as learning rate and batch size on loss sharpness, Jacobian norm and generalisation, revealing the validity of our explanation. Finally, we present the results of experiments measuring the loss sharpness, Jacobian norm and generalisation gap of networks trained with a variety of regularisation measures, revealing in all cases the superior correlation of Jacobian norm and generalisation when compared with loss sharpness, whether or not the regularisation technique targets loss sharpness.
\end{enumerate}

\section{Related work}


\textbf{Flatness, Jacobians and generalisation:} The position that flatter minima generalise better dates back to \cite{hochreiter}. It has since become a staple concept in the toolbox of deep learning practitioners \cite{keskar, chaudhari, foret, chensam}, with its incorporation into training schemes yielding state-of-the-art performance in many tasks. Its effectiveness has motivated the study of the effect that learning hyperparameters such as learning rate and batch size have on the sharpness of minima to which the algorithm will converge \cite{wudynamicalstability,zhusgd,mulayoff}. The hypothesis has received substantial criticism. In \cite{dinh} it is shown that flatness is not \emph{necessary} for good generalisation in ReLU models, since such models are invariant to scaling symmetries of the parameters which arbitrarily sharpen their loss landscapes. This observation has motivated the consideration of scale-invariant measures of loss sharpness \cite{lyu1, jangetal}. In \cite{granziol}, the \emph{sufficiency} of loss flatness for good generalisation is disputed by showing that models trained with the cross-entropy cost generalise better with weight decay than without, despite the former ending up in sharper minima. Sufficiency of flatness for generalisation is further challenged for Gaussian-activated networks on regression tasks in \cite{sameera}.

Relatively little work has looked into the mechanism underlying the relationship between loss curvature and generalisation. PAC-Bayes bounds \cite{dziugaite, foret, fhe1} provide some theoretical support for the hypothesis that loss flatness is sufficient for good generalisation, however this hypothesis, taken unconditionally, is known to be false empirically \cite{granziol, sameera}.  The same is true of \cite{petzka2021relative}, which cleverly bounds the generalisation gap by a rescaled loss Hessian trace. Despite advocating for the empirical Jacobian as the link between loss curvature and generalisation as we do, the insightful papers \cite{maetal, gambaetal} consider the relationship only at a critical point of the square cost, and in particular cannot offer explanations for progressive sharpening or the generalisation benefits of training with an only an \emph{initially} large learning rate that is gradually decayed, as our account can.  \cite{leeetal} advocates empirically for the parameter-output Jacobian of the model as the mediating link between loss sharpness and generalisation, but provide no theoretical indication for why this should be the case. We provide evidence in Appendix \ref{app:leeetal} indicating that our proposal may be able to fill this gap in \cite{leeetal}.

Finally, we note that among other generalisation studies in the literature, ours is most closely related to \cite{maetal} which exploits properties of the data distribution to prove a generalisation bound. Our bound is also related to \cite{bartlett2, bubeck, sulam}, which recognise sensitivity to input perturbations as key to generalisation. Like \cite{weima}, our work in addition formalises the intuitive idea that the empirical Jacobian norm regulates generalisation \cite{drucker, hoffman, novak2018sensitivity}. Further study of our ansatz in relation to the edge of stability phenomenon \cite{coheneos} may be of utility to algorithmic stability approaches to generalisation \cite{bousquet, hardt, charles, kuzborskij, bassily}.

\textbf{Progressive sharpening and edge of stability:} The effect of learning rate $\eta$ on loss sharpness has been understood to some extent for several years \cite{wudynamicalstability}. In \cite{coheneos}, this relationship was attended to for deep networks with a rigorous empirical study which showed that, after an initial period of sharpness increase during training called \emph{progressive sharpening}, the sharpness increase halted at around $2/\eta$ and oscillated there while loss continued to decrease non-monotonically, a phase called \emph{edge of stability}.  These phenomena show that the typical assumptions made in theoretical work, namely that the learning rate is always smaller than twice the reciprocal of the sharpness \cite{zhu, du1, du2, macdonald}, do not hold in practice.  Significant work has since been conducted to understand these phenomena \cite{lyu1, leeeos, aroraeos, wangeos, zhueos, ahneos}, with \cite{damianeos} showing in particular that edge of stability is a universal phenomenon arising from a third-order stabilising effect that must be taken into account when the sharpness grows above $2/\eta$. In contrast, progressive sharpening is not universal; it is primarily observed only in genuinely deep neural networks. Although it has been correlated to growth in the norm of the output layer \cite{wangeos}, the cause of progressive sharpening has so far remained mysterious. The mechanism we propose is the first account of which we are aware for a \emph{cause} of progressive sharpening.

\section{Background and paper outline}

\subsection{Outliers in the spectrum}

We follow the formal framework introduced in \cite{macdonald}, which allows us to treat all neural network layers on a common footing.  Specifically, we consider a multilayer parameterised system $\{f_i:\RB^{p_i}\times\RB^{d_{i-1}}\rightarrow\RB^{d_i}\}_{i=1}^{L}$ (the layers), with a data matrix $X\in\RB^{d_0\times N}$ consisting of $N$ data vectors in $\RB^{d_0}$.  We denote by $F:\RB^{p_1+\dots+p_L}\rightarrow\RB^{d_{L}\times N}$ the associated parameter-function map defined by
\begin{equation}            F_X(\theta_1,\dots,\theta_L)_i:=f_L(\theta_L)\circ\cdots\circ f_1(\theta_1)(X)_i\in\RB^{d_L}\,\qquad i=1,\dots,N.
\end{equation}
Given a convex cost function $c:\RB^{d_L}\times\RB^{d_L}\rightarrow\RB$ and a target matrix $Y\in\RB^{d_L\times N}$, we consider the associated loss
\begin{equation}
    \ell(\vec{\theta}):=\gamma_Y\circ F_X(\vec{\theta}) = \frac{1}{N}\sum_{i=1}^{N}c\big(F_X(\vec{\theta})_i,Y_i\big),
\end{equation}
where $\gamma_Y:\RB^{d_L\times N}\rightarrow\RB$ is defined by $\gamma_Y(Z):=N^{-1}\sum_{i=1}^{N}c(Z_i,Y_i)$.

Using the chain rule and the product rule, one observes that the Hessian $D^2\ell$ of $\ell$ admits the decomposition
\begin{equation}\label{eq:hessian}
    D^2\ell = DF_X^T\,D^2\gamma_Y\,DF_X + D\gamma_Y\,D^2F_X.
\end{equation}
The first of these terms, often called the \emph{Gauss-Newton matrix}, is positive-semidefinite by convexity of $\gamma_Y$.  It has been demonstrated empirically in a vast number of practical settings that the largest, outlying eigenvalues of the Hessian throughout training correlate closely with those of the Gauss-Newton matrix \cite{papyan1, papyan2, coheneos}.  In attempting to understand the relationship between loss sharpness and model behaviour, therefore, empirical evidence invites us to devote special attention to the Gauss-Newton matrix. \textbf{In what follows we assume based on this evidence that the outlying eigenvalues of the Gauss-Newton matrix \emph{determine} those of the loss Hessian}.

\subsection{The Gauss-Newton matrix and input-output Jacobians}

Letting $C$ denote the square root $(D^2\gamma_Y)^{\frac{1}{2}}$, we see that the Gauss-Newton matrix $DF_X^T\,C^2\,DF_X$ has the same nonzero eigenvalues as its conjugate matrix $C\,DF_X\,DF_X^T\,C^T$, which is closely related to the \emph{tangent kernel} $DF_X\,DF_X^T$ identified in \cite{jacot}.  Letting $Jf_l$ and $Df_l$ denote the input-output and parameter derivatives of a layer $f_l:\RB^p\times\RB^{d_{l-1}}\rightarrow\RB^{d_{l}}$ respectively, one sees that
\begin{equation}\label{eq:tangentkernel}
    C\,DF_X\,DF_X^T\,C^T = C\,\bigg(\sum_{l=1}^{L}\bigg(Jf_L\cdots Jf_{l+1}\,Df_l\,Df_l^T\,Jf_{l+1}^T\cdots Jf_L^T\bigg)\bigg)C^T
\end{equation}
is a sum of positive-semidefinite matrices.  Each summand is determined in large part by the composites of input-output layer Jacobians. Note, however, that the input-output Jacobian of the \emph{first} layer does not appear; only its parameter derivative does.

It is clear then that insofar as the largest eigenvalues of the Hessian are determined by the Gauss-Newton matrix, these eigenvalues are determined in part by the input-output Jacobians of all layers following the first.  We judge the following ansatz to be intuitively clear from careful consideration of Equation \eqref{eq:tangentkernel}.

\begin{ansatz}\label{ansatz}
Under certain conditions, an increase in the magnitude of the input-output Jacobian of a deep neural network will cause an increase in the loss sharpness.  Conversely, a decrease in the sharpness will cause a decrease in the magnitude of the input-output Jacobian.
\end{ansatz}

Note that \cite{maetal, gambaetal} already contain a rigorous proofs of Ansatz \ref{ansatz} under special conditions: namely at critical points of the square cost. This restricted setting is not sufficient for our purposes. To understand the relationship between loss curvature and model behaviour adequately, we must invoke Ansatz \ref{ansatz} \emph{throughout training} and most frequently with the \emph{cross-entropy cost function}, where \cite{maetal, gambaetal} do not apply.

Importantly, \emph{we make no claim that Ansatz \ref{ansatz} holds unconditionally}. Generally, the relationship proposed in Ansatz \ref{ansatz} is \emph{necessarily} mediated by the following factors present in Equation \eqref{eq:tangentkernel}: (1) the square root $C$ of the second derivative of the cost function; (2) the presence of the parameter derivatives $Df_l$; (3) the complete absence of the Jacobian of the first layer. These mediating factors prevent a simple upper bound of input-output Jacobian by loss Hessian, and we spent a significant amount of time and computational power exposing the importance of these mediating factors empirically (Figure \ref{fig:weightdecay}, Appendix \ref{app:mediating}). In all cases where Ansatz \ref{ansatz} appears not to apply, we are able to account for its inadequacy in terms of these mediating factors using Equation \eqref{eq:tangentkernel} and further measurements. Moreover, we are able to to explain all known exceptions to the ``rules" of the superior generalisation of flat minima \cite{granziol, sameera} and of progressive sharpening \cite{coheneos} as consequences of unusual behaviour in these mediating factors. Ours is the only account we are aware of for these phenomena which can make this claim.

\section{General theory}\label{sec:theory}

All of what follows is motivated by the following idea: the Lipschitz norm of a differentiable function is upper-bounded by the supremum, over the relevant data distribution, of the spectral norm of its Jacobian. Intuitively, this supremum will itself be approximated by the maximum Jacobian norm over a finite sample of points from the distribution, which by Ansatz \ref{ansatz} can be expected to relate to the loss Hessian of the model.  It is this intuition we seek to formalise and, and whose practical relevance we seek to justify, in this section.


In what follows, we will use $\RB_{>0}$ to denote the positive real numbers. We will use $\PB$ to denote a probability measure on $\RB^d$, whose support $\supp(\PB)$ we assume to be a metric space with metric inherited from $\RB^d$. Given a Lipschitz function $g:\RB^d\rightarrow\RB^{d'}$, we use $\|g\|_{Lip,\PB}$ to denote the Lipschitz norm of $g|_{\supp(\PB)}$. Note the dual meaning of $\|g\|_2$: when $g$ is \emph{vector}-valued, it refers to the Euclidean norm, while when $g$ is \emph{operator}-valued (e.g., when $g = Jf$ is a Jacobian), $\|g\|_2$ refers to the spectral norm. We will frequently invoke pointwise evaluation of the derivative of Lipschitz functions, and in doing so always rely on the fact that our evaluations are probabilistic and Lipschitz functions are differentiable almost everywhere. We will use $B(x,\delta)\subset\RB^d$ to denote the \emph{closed} Euclidean ball of radius $\delta$ centered $x$. Finally, given a locally bounded (but not necessarily continuous) function $g:\RB^d\rightarrow\RB^{d'}$ and a compact set $S\subset\RB^d$, we define the \emph{local variation of $g$ over $S$} by
\begin{equation}\label{variation}
    V_{S}(g):=\sup_{x,y\in S}\|g(x)-g(y)\|_2.
\end{equation}
Thus, for instance, if $g$ is Lipschitz and $S$ is a ball of radius $\delta$, then $V_S(g)\leq2\delta\|g|_{S}\|_{Lip}$. All proofs are deferred to the appendix.

We begin by specifying the data distributions with which we will be concerned.

\begin{definition}\label{def:good}
    Let $\delta:\NB\times(0,1)\rightarrow\RB_{>0}$ be a function such that for each $\epsilon\in(0,1)$, the function $N\mapsto\delta(N,\epsilon)$ is decreasing and vanishes as $N\rightarrow\infty$. We say that a distribution $\PB$ is $\delta$-\textbf{good} if for every $N\in\NB$ and $\epsilon\in(0,1)$, with probability at least $1-\epsilon$ over i.i.d. samples $(x_1,\dots,x_N)$ from $\PB$, one has $\supp(\PB)\subset\bigcup_{i=1}^{N}B(x_i,\delta(N,\epsilon))$.
\end{definition}

A similar assumption on the data distribution is adopted in \cite{maetal}, however no \emph{proof} is given in \cite{maetal} that such an assumption is characteristic of data distributions of practical interest. Our first theorem, derived from \cite[Theorem 2.1]{reznikov}, is that many data distributions of practical interest in deep learning are indeed examples of Definition \ref{def:good}.

\begin{theorem}\label{thm:satisfy}
    Suppose that $\PB$ is the normalised Riemannian volume measure of a compact, connected, embedded, $d$-dimensional submanifold of Euclidean space (possibly with boundary and corners). Then there exists $\delta:\NB\times(0,1)\rightarrow\RB_{>0}$, with $\delta(N,\epsilon) = O\big((\log(N\epsilon^{-1})N^{-1})^{\frac{1}{d}}\big)$, such that $\PB$ is $\delta$-good.

    In particular, the uniform distribution on the unit hypercube $[0,1]^d$ in $\RB^d$ is $\delta_{[0,1]^d}$-good with
    \begin{equation}
        \delta_{[0,1]^d}(N,\epsilon)=\frac{4\,\Gamma(\frac{d}{2}+1)^{\frac{1}{d}}}{\sqrt{\pi}}\bigg(\frac{\log(N\epsilon^{-1})}{N}\bigg)^{\frac{1}{d}},
    \end{equation}
    and, for $0<\gamma<1$ and $N$ sufficiently large, the uniform distribution on the $d$-dimensional unit hypersphere $S^{d}\subset\RB^{d+1}$ is $\delta_{S^{d}}$-good with
    \begin{equation}
        \delta_{S^d}(N,\epsilon)=2^{1+\frac{1}{d}}(1-\gamma)^{-\frac{1}{2}}\bigg(\frac{\log(N\epsilon^{-1})}{N}\bigg)^{\frac{1}{d}}.
    \end{equation}
    
    Moreover, if $\PB$ is any $\delta$-good distribution, then the pushforward of $\PB$ by any Lipschitz function $g$ is $\|g\|_{Lip,\PB}\delta$-good.
\end{theorem}

Theorem \ref{thm:satisfy} says in particular that the distribution generated by any GAN whose latent space is the uniform distribution on a hypercube or on a sphere is an example of a good distribution, as is any distribution generated by an implicit neural function \cite{siren}.  Since a high dimensional isotropic Gaussian is close to a uniform distribution on a sphere, which is good by Theorem \ref{thm:satisfy}, we believe it likely that our theory can be strengthened to include Gaussian distributions, however we have not attempted to prove this and leave it to future work.

Our next theorem formalises our intuition that the maximum Jacobian is an approximate upper bound on the Lipschitz constant of a model throughout training. Its proof is implicit in the proofs of both of our following theorems which characterise progressive sharpening and generalisation in terms of the empirical Jacobian norm.

\begin{theorem}\label{thm:jacobianmaximum}
Suppose that $\PB$ is $\delta$-good, and let $N\in\NB$ and $0<\epsilon<1$. Then with probability at least $1-\epsilon$ over i.i.d. samples $(x_1,\dots,x_N)$ from $\PB$, for any Lipschitz function $f:\RB^d\rightarrow\RB^{d'}$, one has:\normalfont
\begin{equation}\label{eq:jacobianmaximum}
\|f\|_{Lip,\PB}\leq \max_i\big(\|Jf(x_i)\|_2+V_{B(x_i,\delta(N,\epsilon))}(Jf)\big).
\end{equation}
\end{theorem}

\section{Progressive sharpening}\label{sec:progressivesharpening}

In this section we give our account of progressive sharpening. We will assume that training of a model $f:\RB^d\rightarrow\RB^{d'}$ is undertaken using a cost function $c:\RB^{d'}\times\RB^{d'}\rightarrow\RB$, whose global minimum  we assume to be zero, and for which we assume that there is $\alpha>0$ such that $c(z_1,z_2)\geq \alpha\|z_1-z_2\|^2_2$ for all $z_1,z_2\in\RB^{d'}$. This is clearly always the case for the square cost; by Pinsker's inequality, it also holds for the Kullback-Leibler divergence on softmax outputs, and hence also for the cross-entropy cost provided one subtracts label entropy.

\begin{theorem}\label{thm:progressivesharpening}
    Suppose that $\PB$ is $\delta$-good, and let $N\in\NB$ and $0<\epsilon<1$. Let $f^*:\supp(\PB)\rightarrow\RB^{d'}$ be a target function and fix a real number $\ell>0$. Then with probability at least $1-\epsilon$ over i.i.d. samples $(x_1,\dots,x_N)$ drawn from $\PB$, for any Lipschitz function $f:\RB^d\rightarrow\RB^{d'}$ satisfying $c(f(x_i),f^*(x_i))\leq \ell^2\alpha$ for all $i$, one has
    \begin{equation}\label{eq:lipschitz}
        \max_i\|Jf(x_i)\|_2\geq \max_{i\neq j}\frac{\|f^*(x_i)-f^*(x_j)\|_2 - 2\ell}{\|x_i-x_j\|_2}  - \max_i V_{B(x_i,\delta(N,\epsilon))}(Jf).
    \end{equation}
\end{theorem}

\textbf{Progressive sharpening:} Theorem \ref{thm:jacobianmaximum} tells us that any training procedure that reduces the loss over all data points will thus also increase the sample-maximum Jacobian norm from a low starting point. Invoking Ansatz \ref{ansatz}, this increase in the sample-maximum Jacobian norm can be expected to cause a corresponding increase in the magnitude of the loss Hessian (see Appendix \ref{sec:bnjac}): progressive sharpening.  

Theorem \ref{thm:progressivesharpening} thus tells us that sharpening can be made more or less severe by scaling the distances between target values: indeed this is what we observe (Figure \ref{fig:labelspacing}, see also Appendix \ref{app:fullbatch}). One might also expect that for non-batch-normalised networks, scaling the inputs closer together would increase sharpening too. However, this is not the case, due to the mediating factors in Ansatz \ref{ansatz}. Specifically, while it is true that such scaling increases the growth rate of the Jacobian norm, this increase in Jacobian norm does \emph{not} necessarily increase sharpening, since it coincides with a \emph{decrease} in the magnitude of the parameter derivatives, which are key factors in relating the model Jacobian to the loss Hessian (see Appendix \ref{sec:param_deriv}). 

\textbf{Edge of stability:} Although the edge of stability mechanism explicated in \cite{damianeos} can be expected to put some downward pressure on the model Jacobian, due to the presence of the mediating factors discussed following Ansatz \ref{ansatz} it \emph{need not} cause the Jacobian to \emph{plateau} in the same way that loss sharpness does. Nonetheless, this downward pressure on Jacobian is important: it is to this that we attribute the better generalisation of models trained with large learning rate, even if that learning rate is decayed towards the end of training. We validate this empirically in the next section and Appendix \ref{app:leeetal}.

\textbf{Wide networks and linear models:} It has been observed empirically that wide networks exhibit less severe sharpening, and linear (kernel) models exhibit none at all \cite{coheneos}. It might be thought, since these models nonetheless must increase in Jacobian norm during training by Theorem \ref{thm:progressivesharpening}, that such models are therefore counterexamples to our explanation. They are in fact consistent: Theorem \ref{thm:progressivesharpening} only implies sharpening \emph{insofar as Ansatz \ref{ansatz} holds}. Importantly, since the Jacobian of the first linear layer does not appear in Equation \eqref{eq:tangentkernel}, an increase in the input-output Jacobian norm of a linear (kernel) model will not cause an increase in sharpness according to Ansatz \ref{ansatz}. The same is true of wide networks, which approximate kernel models increasingly well as width is sent to infinity \cite{leewide}; thus wide networks will also be expected to sharpen less severely during training, as has been observed empirically.

\begin{figure}
    \centering
    \setkeys{Gin}{width=\linewidth}
    \begin{subfigure}{0.3\columnwidth}
        \includegraphics{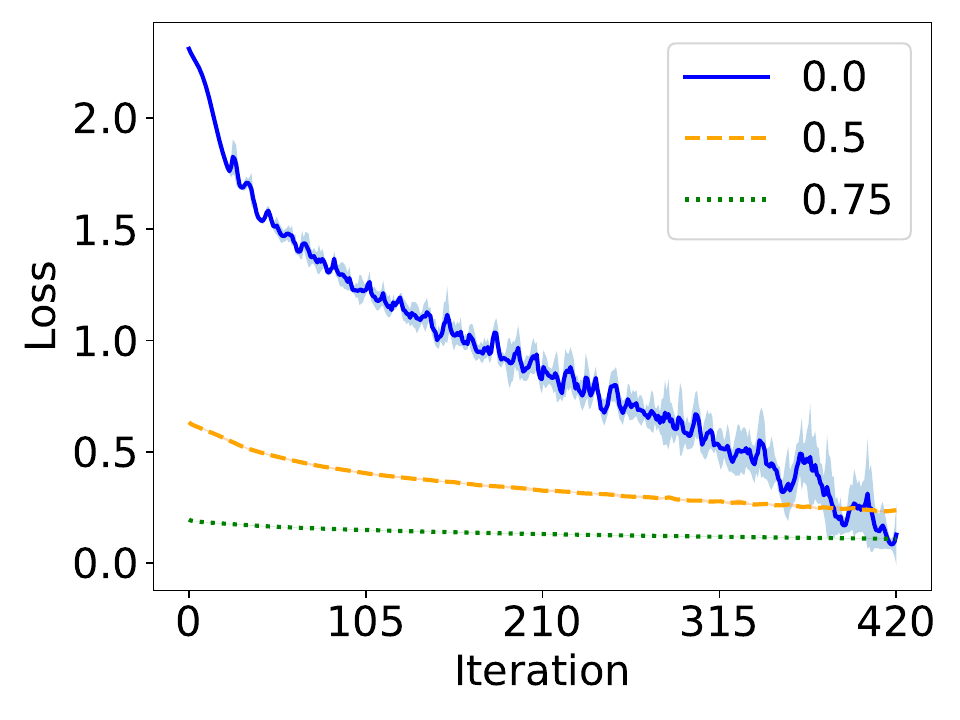}
        \caption{Loss}
    \end{subfigure}\hfill
    \begin{subfigure}{0.3\columnwidth}
        \includegraphics{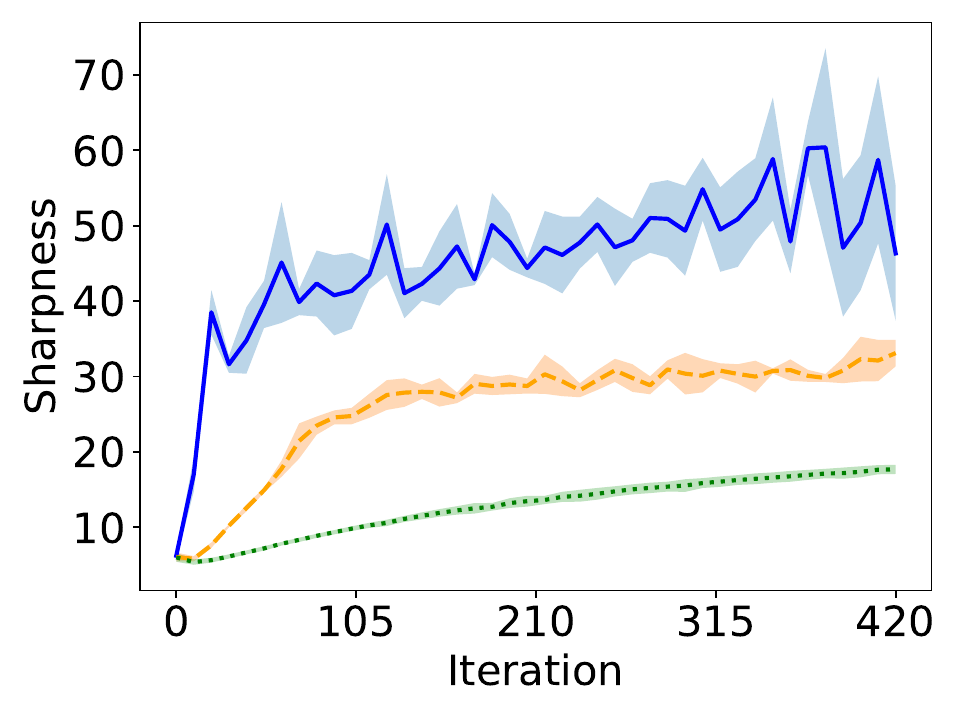}
        \caption{Sharpness}
    \end{subfigure}\hfill
    \begin{subfigure}{0.3\columnwidth}
    \includegraphics{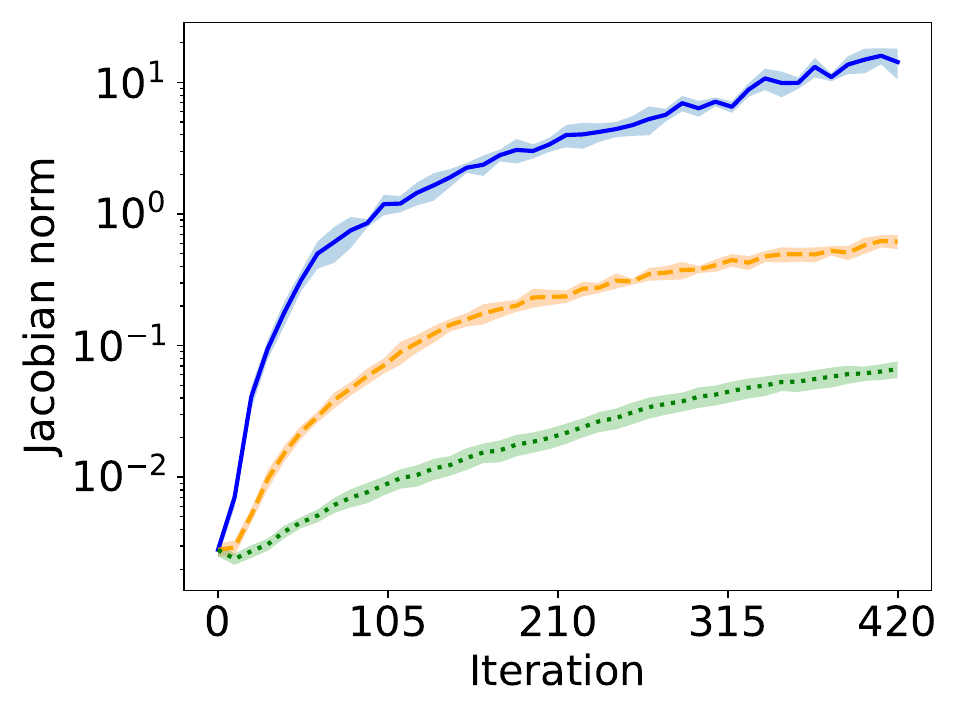}
    \caption{Jacobian norm (log $y$ axis)}
    \end{subfigure}
    \caption{Plots of cross entropy loss (with label entropies subtracted), loss sharpness and softmaxed-model Jacobian norm with varying degrees of label smoothing, for VGG11 trained with gradient descent on CIFAR10, with a learning rate of 0.08. Increasing label smoothing (indicated by line style) brings targets closer together, meaning less growth is necessary in the Jacobian norm to reduce loss. This coincides with less sharpening of the Hessian during training, in line with Ansatz \ref{ansatz}. Similar plots at different learning rates, and with ResNet18, are in Appendix \ref{app:fullbatch}.  Note log $y$ axis on Jacobian norm for ease of distinction.}
    \label{fig:labelspacing}
\end{figure}


\section{Flat minima and generalisation}\label{sec:generalisation}

We argue in this section that loss flatness implies good generalisation only via encouraging smaller Jacobian norm, through Ansatz \ref{ansatz}. Indeed, our final theorem, whose proof is inspired by that of \cite[Theorem 6]{maetal}, assures us rigorously that on models that fit training data, sufficiently small empirical Jacobian norm and sufficiently large sample size suffice to guarantee good generalisation, as has long been suspected in the literature \cite{drucker}.

\begin{theorem}\label{thm:generalisationbound}
    Let $\PB$ be $\delta$-good and let $f^*:\supp(\PB)\rightarrow\RB^{d'}$ be a target function, which we assume to be Lipschitz. Let $N\in\NB$ and let $0<\epsilon<1$. Then with probability at least $1-\epsilon$ over all i.i.d. samples $(x_1,\dots,x_N)$ from $\PB$, any Lipschitz function $f:\RB^d\rightarrow\RB^{d'}$ which coincides with $f^*$ on $(x_1,\dots,x_N)$, satisfies:
    \begin{equation}
        \EB_{x\sim\PB}\|f(x)-f^*(x)\|_2\leq \delta(N,\epsilon)\big(\|f^*\|_{Lip,\PB}+\max_i\big(\|Jf(x_i)\|_2+V_{B(x_i,\delta(N,\epsilon))}(Jf)\big)\big).
    \end{equation}
\end{theorem}

Our proof of Theorem \ref{thm:generalisationbound} is similar to that of \cite[Theorem 6]{maetal}. Our bound is an improvement on that of \cite{maetal} in having better decay in $N$. Additionally, in bounding in terms of the empirical Jacobian norm at convergence instead of training hyperparameters at convergence, our bound is sensitive to implicit Jacobian regularisation throughout training in a way that the bound of \cite{maetal} is not. Our bound is thus in principle sensitive to the superior generalisation of networks trained with an \emph{initially} high learning rate that is gradually decayed to a small learning rate (such networks experience more implicit Jacobian regularisation at the edge of stability), over networks trained with a small learning rate from initialisation.  In contrast, the bound of \cite{maetal} is not sensitive to this distinction.

Like \cite[Theorem 6]{maetal}, however, our bound exhibits \emph{inferior} $O\big((\log(N)N^{-1})^{\frac{1}{d}}\big)$ decay in $N$, where $d$ is the intrinsic dimension of the data distribution, when compared with the more common $O(N^{-\frac{1}{2}})$ rates in the literature. The reason for this is that all bounds of which we are aware in the literature, with the exception of \cite[Theorem 6]{maetal}, derive from consideration of \emph{hypothesis complexity}, rather than \emph{data complexity} as we and \cite{maetal} consider.  Although the method we adopt implies this worse decay in $N$, our method has the advantage of not needing to invoke the large hypothesis complexity terms, such as Rademacher complexity and KL-divergence, that make bounds based on such terms so loose when applied to deep learning \cite{zhang}. Moreover, since standard datasets in deep learning are intrinsically low dimensional \cite{miyato}, the $O\big((\log(N)N^{-1})^{\frac{1}{d}}\big)$ rate in our bounds can still be expected to be nontrivial in practice.


In what follows we present experimental results measuring the generalisation gap (absolute value of train loss minus test loss), sharpness and Jacobian norm at the end of training with varying degrees of various regularisation techniques. Our results confirm that while loss sharpness is neither necessary nor sufficient for good generalisation, empirical Jacobian norm consistently correlates better with generalisation than does sharpness for all regularisation measures we studied. Moreover, whenever loss sharpness \emph{does} correlate with generalisation, this correlation tends to go hand-in-hand with smaller Jacobian norm, as predicted by Ansatz \ref{ansatz}. Note that the Jacobians we plot were all computed in evaluation mode, meaning all BN layers had their statistics fixed: that the relationship with the train mode loss Hessian can still be expected to hold is justified in Appendix \ref{sec:bnjac}.

\begin{figure}
    \centering
    \includegraphics[scale=0.45, trim=8pt 12pt 0pt 24pt, clip]{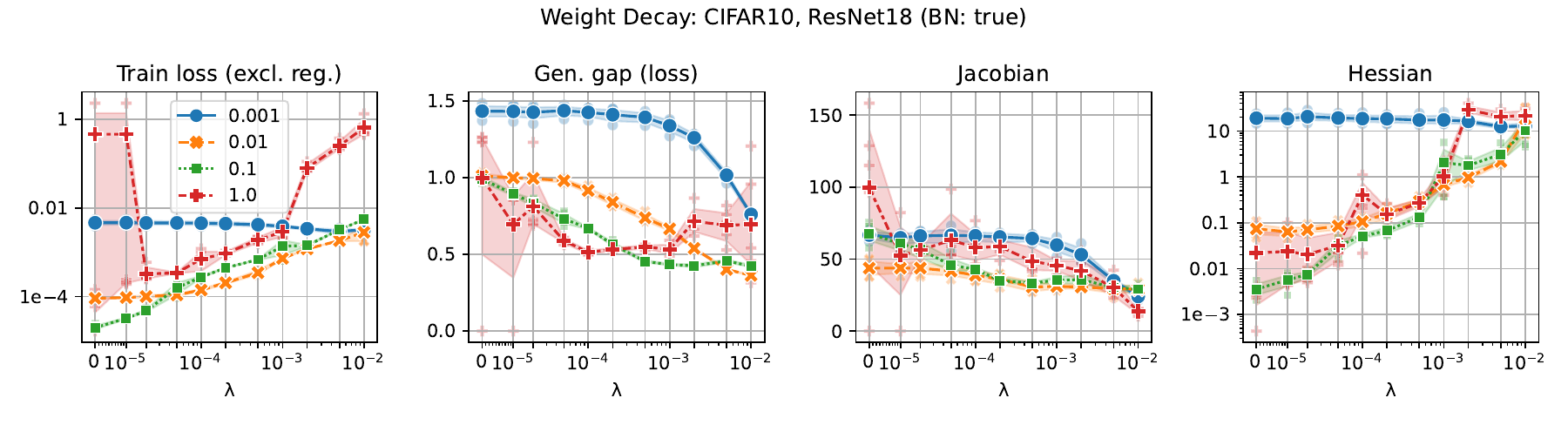}
    \caption{The impact of weight decay ($x$ axis) on Jacobian norm, sharpness and generalisation gap when training ResNet18 with cross-entropy loss at the end of training (90 epochs). Line style indicates learning rate. 
 Without weight decay, SGD finds a solution with near-zero loss and vanishing sharpness, but for which Jacobian norm and generalisation gap are relatively large.}
    \label{fig:weightdecay}
\end{figure}

\textbf{Cross-entropy and weight decay:} It was observed in \cite{granziol} that when training with the cross-entropy cost, networks trained with weight decay generalised better than those trained without, despite converging to sharper minima. We confirm this observation in Figure \ref{fig:weightdecay}.  Using Ansatz \ref{ansatz} and Theorem \ref{thm:generalisationbound}, we are able to provide a new explanation for this fact.  For the cross-entropy cost, both terms $D^2\gamma_Y$ and $D\gamma_Y$ appearing in the Hessian (Equation \eqref{eq:hessian}) vanish at infinity in parameter space.  By preventing convergence of the parameter vector to infinity in parameter space, weight-decay therefore encourages convergence to a sharper minimum than would otherwise be the case. However, since the Frobenius norm of a matrix is an upper bound on the spectral norm, weight decay also implicitly regularises the Jacobian of the model, leading to better generalisation by Theorem \ref{thm:generalisationbound}.

Since training with weight decay is the norm in practice, and since any correlations between \emph{vanishing} Hessian and generalisation will likely be unreliable, we used weight decay in all of what follows.

\textbf{Learning rate and batch size:} It is commonly understood that larger learning rates and smaller batch sizes serve to encourage convergence to flatter minima \cite{keskar, wudynamicalstability}.  By Theorem \ref{thm:generalisationbound} and Ansatz \ref{ansatz}, such hyperparameter choice can be expected to lead to better generalisation. We report on experiments testing this idea in Figures \ref{fig:lr}, \ref{fig:batch} (ResNet18 on CIFAR10) and Appendix \ref{app:batch_lr} (ResNet18 and VGG11 on CIFAR10 and CIFAR100). The results are mostly consistent with expectations, with the exception of Jacobian norm being reduced with larger batch size at the highest learning rate while generalisation gap increases. The phenomenon appears to be data-agnostic (the same occurs with ResNet18 on CIFAR100), but architecture-dependent (we do not observe this problem with VGG11 on either CIFAR10 or CIFAR100). This is not necessarily inconsistent with Theorem \ref{thm:generalisationbound}, which allows for the empirical Jacobian norm to be a poor estimator of the Lipschitz constant of the model via the local variation in the model's Jacobian. Our results indicate that this local variation term is increased during (impractical) large batch training of skip-connected architectures using a constant high learning rate.

\begin{figure}[t]
    \centering
    \setkeys{Gin}{width=\linewidth}
    \begin{subfigure}{0.24\columnwidth}
        \includegraphics{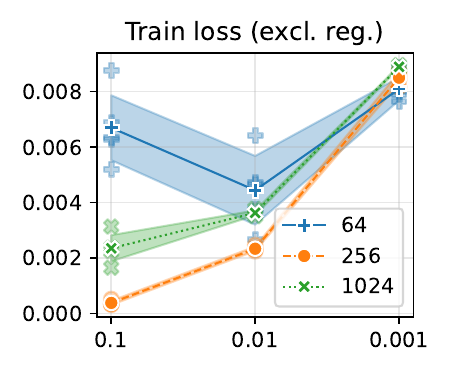}
    \end{subfigure}\hfill
    \begin{subfigure}{0.24\columnwidth}
        \includegraphics{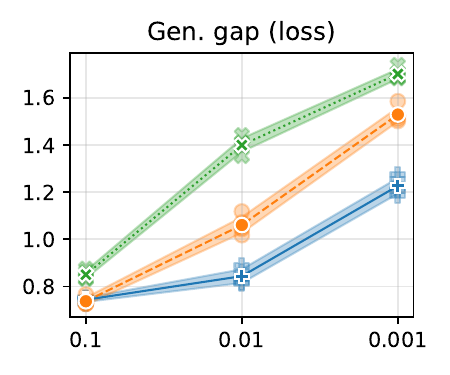}
    \end{subfigure}\hfill
    \begin{subfigure}{0.24\columnwidth}
        \includegraphics{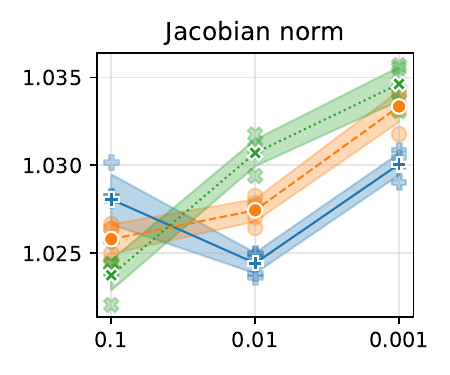}
    \end{subfigure}\hfill
    \begin{subfigure}{0.24\columnwidth}
    \includegraphics{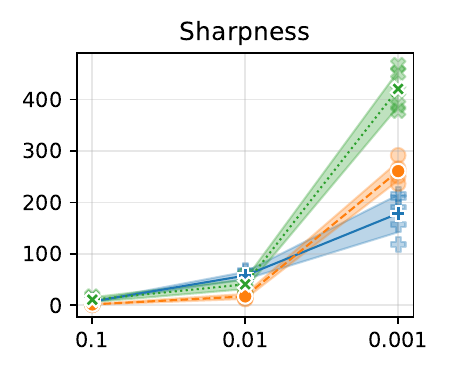}
    \end{subfigure}
    \caption{The impact of differing (constant) learning rates ($x$ axis) at end of training with ResNet18 on CIFAR10 trained with SGD (five trials).  Line style indicates batch size.  In line with conventional wisdom, larger learning rates typically have smaller sharpness and generalise better.  The downward pressure on sharpness and generalisation gap correlates with smaller Jacobian norms, as anticipated by Ansatz \ref{ansatz} and Theorem \ref{thm:generalisationbound}.  Models trained with weight decay (0.0001).}
    \label{fig:lr}
\end{figure}

\begin{figure}[t]
    \centering
    \setkeys{Gin}{width=\linewidth}
    \begin{subfigure}{0.24\columnwidth}
        \includegraphics{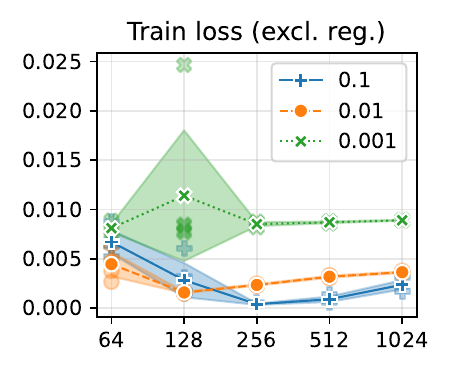}
    \end{subfigure}\hfill
    \begin{subfigure}{0.24\columnwidth}
        \includegraphics{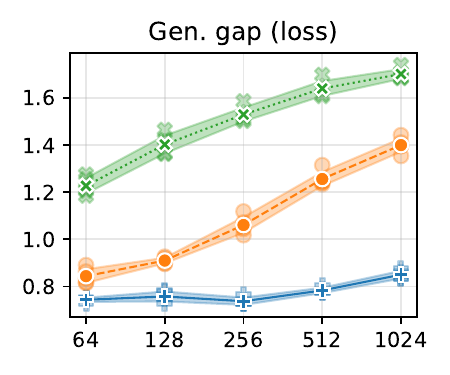}
    \end{subfigure}\hfill
    \begin{subfigure}{0.24\columnwidth}
        \includegraphics{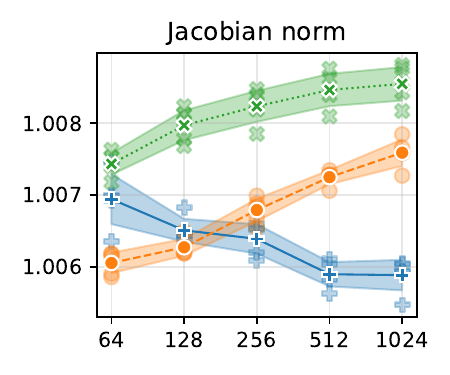}
    \end{subfigure}\hfill
    \begin{subfigure}{0.24\columnwidth}
    \includegraphics{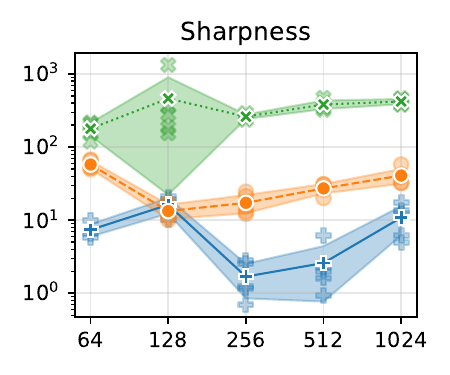}
    \end{subfigure}
    \caption{The impact of differing batch sizes ($x$ axis) at the end of training with ResNet18 on CIFAR10 trained with SGD (five trials).  Line style indicates (constant) learning rate.  At least for the smaller learning rates, increasing batch size typically increases sharpness, Jacobian norm and generalisation gap in line with Theorem \ref{thm:generalisationbound}.  At the largest learning rate, the relationship between Jacobian norm and generalisation gap is the inverse of what is expected, suggesting that larger learning rates tend to increase the local variation of the model's Jacobian, causing it to underestimate the Lipschitz constant of the model (Theorem \ref{thm:generalisationbound}). Note log scale on $y$ axis in Sharpness plot for ease of distinction.}
    \label{fig:batch}
\end{figure}

\textbf{Other regularisation techniques:} We also test to see the extent that other common regularisation techniques, including label smoothing, data augmentation, mixup \cite{mixup} and sharpness-aware minimisation (SAM) \cite{foret} regularise Hessian and Jacobian to result in better generalisation. We find that while Jacobian norm is regularised at least initially across all techniques, only in SAM is sharpness regularised. This validates our proposal that Jacobian norm is a key mediating factor between sharpness and generalisation, as well as the position of \cite{dinh} that flatness is not necessary for good generalisation.  Note that at least some of the gradual increase in Jacobian norm for mixup and data augmentation as the $x$-axis parameter is increased is to be expected due to our measurement scheme: see the caption in Figure \ref{fig:practical}. See Appendix \ref{app:practical} for experimental details and similar results for VGG11 and CIFAR100.

\begin{figure}[t]
\makebox[\textwidth][c]{\includegraphics[scale=0.45, trim=12pt 12pt 0pt 36pt, clip]{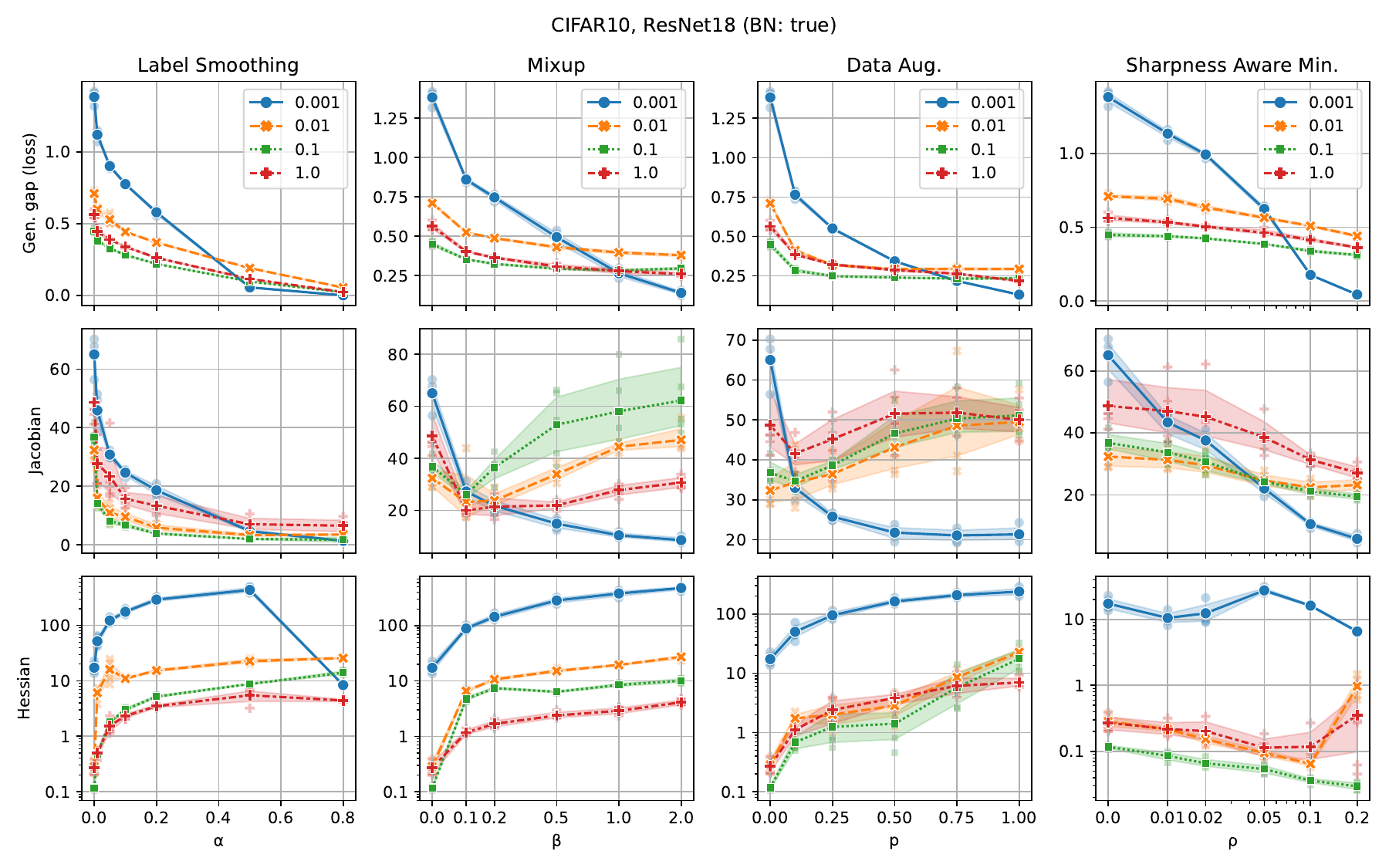}}
\caption{The impact of different regularisation strategies ($x$ axis) when training a ResNet18 model with batch-norm on CIFAR10 (five trials).  The norm of the input-output Jacobian is much better correlated with the generalisation gap than that of the Hessian; all regularisation strategies were effective in controlling the Jacobian norm without necessarily controlling the sharpness.
Line style indicates initial learning rate.
Besides the regularisation strategy being studied, all experiments include weight decay (0.0005). Note that the decorrelation between Jacobian and generalisation for Mixup and Data Aug as $x$-axis parameter is increased is to be expected. In both cases, we measure the Jacobian norm over \emph{only the pure training set}, which makes up less of a percentage of the total training set and is therefore subject to less implicit regularisation during training as the $x$-axis parameter is increased.
Refer to the appendix for extended results.}
\label{fig:practical}
\end{figure}

\section{Discussion}

One limitation of our work is that Ansatz \ref{ansatz} is not mathematically rigorous.  A close theoretical analysis would ideally provide rigorous conditions under which one can upper bound the model Jacobian norm in terms of loss sharpness \emph{throughout training}. These conditions would need to be compatible with the empirical counterexmaples we give in Figure \ref{fig:weightdecay} and Appendix \ref{app:mediating}.


A second limitation of our work is that we do not numerically evaluate our generalisation bound. Numerical evaluation of the bound would require both an estimate of the Lipschitz constant of a GAN which has generated the data, as well as the local variation of the Jacobian of the model defined in Equation \eqref{variation}. While existing techniques may be able to aid the first of these \cite{fazlyab}, we are unaware of any work studying the second, which is outside the scope of this paper. We leave this important evaluation to future work.


\section{Conclusion}

We proposed a new relationship between the loss Hessian of a deep neural network and the input-output Jacobian of the model. We proved several theorems allowing us to leverage this relationship in providing new explanations for progressive sharpening and the link between loss geometry and generalisation. An experimental study was conducted which validates our proposal.

\bibliography{references}

\newpage
\appendix

\noindent\textbf{Computational resources:}  
Exploratory experiments were conducted using a desktop machine with two Nvidia RTX A6000 GPUs.
The extended experimental results in the appendix were obtained on a shared HPC system, where a total of 3500 GPU hours were used across the lifetime of the project (Nvidia A100 GPUs).

\section{Proofs}

Here we give proofs of the theorems stated in the main body of the paper. To prove our first theorem, Theorem \ref{thm:satisfy}, we require the following differential-geometric lemma.

\begin{lemma}\label{diffgeom}
    Let $M$ be any compact, embedded, $d$-dimensional submanifold of $\RB^n$, possibly with boundary and corners. Equip $M$ with the Riemannian structure inherited from this embedding and let $\PB$ be the normalised volume measure of $M$ induced by its Riemannian structure. Then there exists $\kappa>0$ and $C>0$ such that
    \[
    \PB[B(x,r)]\geq C\,r^d
    \]
    for all $r\leq\kappa$ and $x\in M$, where $B(x,r)$ denotes the closed Euclidean ball around $x$ of radius $r$.
\end{lemma}

\begin{proof}
    Since the Riemannian structure on $M$ is inherited from its embedding into $\RB^n$, the embedding is bi-Lipschitz and hence there exists a constant $c$ such that
    \[
    \|x-y\|_2\leq c^{-1}g(x,y)
    \]
    for all $x,y\in M$, where $g(x,y)$ denotes the geodesic distance induced on $M$ by the Riemannian structure. It follows that for all $r>0$, one has $B_g(x,cr)\subset B(x,r)$, where $B_g(x,r)$ is the closed geodesic ball of radius $r$ about $x$. Hence
    \begin{equation}\label{eq:esti1}
    \PB[B(x,r)]\geq \PB(B_g(x,cr))
    \end{equation}
    and it remains only to lower-bound $B_g(x,cr)$ for $r$ sufficiently small. This we achieve by considering the geometry of $M$.

    Given $x\in M$, let $\exp_x:T_xM\rightarrow M$ be the Riemannian exponential map and let $(x,0)\in T_xM$ denote the zero tangent vector. Recall that the \emph{injectivity radius} of $M$ is the largest number $R$ for which the restriction of $\exp_x$ to the $R$-ball about zero in $T_xM$ is a diffeomorphism onto its image for all $x\in M$. The injectivity radius exists and is strictly positive by compactness of $M$. Let $\kappa>0$ be any number that is strictly smaller than $c^{-1}R$.
    
    Let us now fix $x\in M$. Since $\exp_x$ is a radial isometry, it maps $B((x,0),r)$ onto $B_g(x,r)$ for any $r\leq\kappa$. Let $d\xi$ denote the standard Euclidean volume element in $B((x,0),r)$ and $\omega$ the volume form on $M$. Since $\exp_x|_{B((x,0),r)}$ is a diffeomorphism, there exists a nonvanishing, smooth function $f$ on $B((x,0),r)$ for which $\exp_x^{*}\omega = f_x\,d\xi$. Letting $c'_x$ denote the minimum of $f_x$, we then have
    \begin{equation}\label{eq:esti2}
    \PB[B_g(x,r)] = \int_{B_g(x,r)}\omega = \int_{B((x,0),r)}\exp_x^*\omega\geq c'_x\int_{B((x,0),r)}d\xi = c'_x\,c'' r^d.
    \end{equation}
    Here $c''$ is the usual scaling factor one sees in the volume of a Euclidean $d$-ball. Using compactness of $M$, we define $c':=\inf_{x\in M}c'_x >0$ and $C:=c'\,c''\,c^d$. It then follows from the estimates \eqref{eq:esti1} and \eqref{eq:esti2} that
    \[
    \PB[B(x,r)]\geq C\,r^d
    \]
    for all $x\in M$ and $0<r\leq\kappa$ as claimed.
\end{proof}

\begin{proof}[Proof of Theorem \ref{thm:satisfy}]
    

    The proof we give is derived from \cite[Theorem 2.1]{reznikov}, however we wish to be more precise with our constants than is the case in \cite{reznikov}. We thus give a full proof here.

    First, using the fact that $\supp(\PB)$ satisfies the hypotheses of Lemma \ref{diffgeom}, fix $C>0$ and $\kappa>0$ for which $\PB[B(x,r)]\geq C\,r^d$ for all $(x,r)\in M\times[0,\kappa]$, and define $\Phi(r):=C\,r^d$. This function $\Phi$ will play a key role in the proof.


    Now fix $\epsilon>0$, let $X_N:=(x_1,\dots,x_N)$ be a set of i.i.d. points sampled from $\PB$, and let $\rho(X_N):=\sup_x\min_i\|x-x_i\|$ denote the covering radius of $X_N$. Suppose that $\rho(X_N)>2t$ for some real number $t$, and let $E_t$ be any maximal set of points in $\supp(\PB)$ whose distinct pairwise distances are all bounded below by $t$. Then we can find $x\in E_t$ such that $B(x,t)\cap X_N=\emptyset$. Indeed, by hypothesis on $\rho(X_N)$, there exists $y\in\supp(\PB)$ such that $X_N\cap B(y,2t) = \emptyset$. There also exists $x\in E_t$ such that $\|x-y\|< t$, since otherwise we could add $y$ to $E_t$ and contradict the maximality of $E_t$. One then has $B(x,t)\subset B(y,2t)$, so that $B(x,t)$ does not intersect $X_N$.
    
    It has thus been shown that
    \begin{equation}
        \PB[\rho(X_N)> 2t]\leq\PB[\exists x\in E_t:B(x,t)\cap X_N=\emptyset].
    \end{equation}
    Since $\PB(B(x,t))\geq \Phi(t)$ for all $x$, $(1-\Phi(t))$ is an upper bound on the probability that a randomly selected $x'$ will not lie within $t$ of a given $x$. Letting $\#(E_t)$ denote the cardinality of $E_t$, the independence of the $x_i$ therefore permits an upper bound
    \begin{equation}
        \PB[\rho(X_N)> 2t]\leq \#(E_t)\,(1-\Phi(t))^N.
    \end{equation}
    The cardinality of $E_t$ can be further bounded via
    \begin{equation}
        1\geq\sum_{x\in E_t}\PB[B(x,t)]\geq\#(E_t)\Phi(t),
    \end{equation}
    so that one has
    \begin{equation}
        \PB[\rho(X_N)>2t]\leq\Phi(t)^{-1}(1-\Phi(t))^N.
    \end{equation}

    Using the invertibility of $\Phi$, one now sets $\Phi(t) = \alpha\log(N)N^{-1}$, with $\alpha = (1-\log_N(\epsilon))$. The McLaurin series for $\log(1-y)$ with $y = \alpha\log(N)N^{-1}$ reveals that
    \begin{equation}
        (1-\alpha\log(N)N^{-1})^{N}\leq N^{-\alpha},
    \end{equation}
    so that the bound becomes
    \begin{equation}
        \PB[\rho(X_N)> 2\Phi^{-1}(\alpha\log(N)N^{-1})]\leq \frac{1}{\alpha}\frac{N}{\log(N)}N^{-\alpha}\leq N^{1-\alpha}.
    \end{equation}
    Substituting $\alpha = (1-\log_N(\epsilon))$ yields
    \begin{equation}\label{eq:deltaeqn}
        \PB\bigg[\rho(X_N)> 2\Phi^{-1}\bigg(\frac{\log(N\epsilon^{-1})}{N}\bigg)\bigg]\leq \epsilon,
    \end{equation}
    from which the general claim follows.

    The result is specialised to the unit hypercube $[0,1]^d$ simply by observing that in this case, one can take
    \begin{equation}
        \Phi(r):=\frac{\pi^{d/2}}{2^d\Gamma(d/2+1)}r^d,   
    \end{equation}
    which is the volume of the intersection of the ball of radius $r$, centered at a corner, with the cube.
    
    For the unit hypersphere $S^d\subset\RB^{d+1}$, fix $x\in S^d$ and consider the Euclidean ball $B(x,r)\subset\RB^{d+1}$ centered at $x$. It is clear that the intersection $B_r:=S^d\cap B(x,r)$ is a geodesic for the inherited Riemannian metric on $S^d$. In particular, for $r<2$,  $B_r$ is a hyperspherical cap. Elementary trigonometry shows that the angle subtended by the line connecting $x$ to $0$ and the line connecting any point on the boundary of $B_r$ to $0$ is given by $2\sin^{-1}(r/2)$. It then follows from \cite[p. 67]{hyperspherical} and the elementary trigonometric identity $\sin(2\theta) = 2\sin(\theta)\cos(\theta)$ that one has
    \[
    \PB[B(x,r)] = C^{-1}\frac{\pi^{\frac{d+1}{2}}}{\Gamma\big(\frac{d+1}{2}\big)}I_{r^2-r^4/4}\bigg(\frac{d}{2},\frac{1}{2}\bigg),
    \]
    where
    \[
    I_{x}(a,b) = \frac{\Gamma(a+b)}{\Gamma(a)\Gamma(b)}\int_{0}^xt^{a-1}(1-t)^{b-1}\,dt
    \]
    is the regularised incomplete beta function and $C = 2\pi^{\frac{d}{2}}\Gamma(d/2)^{-1}$ is the volume of the unit $d$-sphere. For any $0<t<1$ one has $(1-t)^{1-\frac{1}{2}}\geq 1$, thus one has the estimate
    \[
    I_{r^2-r^4/4}\bigg(\frac{d}{2},\frac{1}{2}\bigg)\geq \frac{\Gamma\big(\frac{d+1}{2}\big)}{\Gamma\big(\frac{d}{2}\big)\Gamma\big(\frac{1}{2}\big)}(r^2-r^4/4)^{\frac{d}{2}}.
    \]
    Now, fixing $0<\gamma<1$ and assuming $r^2/4<\gamma$, we therefore have
    \[
    \PB(B(x,r))\geq \frac{1}{2}(1-\gamma)^{\frac{d}{2}}r^d.
    \]
    We can then use $\Phi(r) = \frac{1}{2}(1-\gamma)^{\frac{d}{2}}r^d$ in defining $\delta(N,\epsilon) = 2\Phi^{-1}(\log(N\epsilon^{-1})N^{-1})$ as in Equation \eqref{eq:deltaeqn} provided only that the assumption $\delta^2/4<\gamma$ is satisfied. Setting $\delta = 2^{1+\frac{1}{d}}(1-\gamma)^{-\frac{1}{2}}(\log(N\epsilon^{-1})N^{-1})^{\frac{1}{d}}$, this requirement reduces to having $N$ sufficiently large that $2^{\frac{2}{d}}(\log(N\epsilon^{-1})N^{-1})^{\frac{2}{d}}<\gamma-\gamma^2$.
    
    For the final result, when pushing forward a good distribution by a Lipschitz function, the stated result follows from the Lipschitz identity.
\end{proof}

\begin{proof}[Proof of Theorem \ref{thm:jacobianmaximum}]
    Fix $N\in\NB$ and $0<\epsilon<1$. By $\delta$-goodness of the data distribution, with probability at least $1-\epsilon$ over i.i.d. samples $(x_1,\dots,x_N)$, the balls $\{B(x_i,\delta(N,\epsilon))\}_{i=1}^N$ cover $\supp(\PB)$. For notational convenience, denote $B(x_i,\delta(N,\epsilon))$ by simply $B_i$. Conditioning on the event that the $B_i$ cover $\supp(\PB)$, one has
    \begin{align*}
        \|f\|_{Lip,\PB}&\leq\sup_{x\in\bigcup_{i=1}^NB_i}\|Jf(x)\|_2 = \max_i\,\sup_{x\in B_i}\|Jf(x)\|_2\\ &\leq \max_i\,\sup_{x\in B_i}\|Jf(x_i)-Jf(x_i)+Jf(x)\|_2\\ &\leq\max_i\bigg(\|Jf(x_i)\|_2+\sup_{x\in B_i}\|Jf(x)-Jf(x_i)\|_2\bigg)\\ &\leq \max_i\big(\|Jf(x_i)\|_2+V_{B_i}(Jf)\big)
    \end{align*}
    The result follows.
\end{proof}

\begin{proof}[Proof of Theorem \ref{thm:progressivesharpening}]
    Recall that by hypothesis on $c$, one has $c(z_1,z_2)\geq \alpha\|z_1-z_2\|_2^2$ for all $z_1,\,z_2$ in  the domain of $c$.  Thus, since $c(f(x_i),f^*(x_i))<\ell^2\alpha$ for all $i$, we therefore have
    \begin{equation}
        \|f(x_i)-f^*(x_i)\|_2\leq\ell
    \end{equation}
    for all $i$. That is, for each $i$, $f(x_i)$ is contained in the $\ell$-ball centred on the target value $f^*(x_i)$.

    Since the smallest distance between any two $f(x_i)$ and $f(x_j)$ is realised when these points lie on the straight line between $f^*(x_i)$ and $f^*(x_j)$, one has the lower bound
    \begin{equation}
        \|f(x_i)-f(x_j)\|_2\geq \|f^*(x_i)-f^*(x_j)\|_2 - 2\ell
    \end{equation}
    for all $i,\,j$. Invoking the Lipschitz identity, and using the fact that $x_i=x_j$ for $i\neq j$ with probability zero, then gives
    \begin{equation}
        \|f\|_{Lip,\PB}\geq \max_{i\neq j}\frac{\|f^*(x_i)-f^*(x_j)\|_2-2\ell}{\|x_i-x_j\|_2}.
    \end{equation}
    Conditioning now on the event that $\supp(\PB)\subset\bigcup_{i=1}^{N}B(x_i,\delta(\epsilon,N))$, which occurs with probability at least $1-\epsilon$ by $\delta$-goodness of $\PB$, and invoking Theorem \ref{thm:jacobianmaximum} together with the subadditivity of the component-wise maximum function, gives the result.
\end{proof}

\begin{proof}[Proof of Theorem \ref{thm:generalisationbound}]
    Fix $N\in\NB$ and $0<\epsilon<1$. Since $\PB$ is $\delta$-good, with probability at least $1-\epsilon$ over i.i.d. samples $(x_1,\dots,x_N)$, for each $x\in\supp(\PB)$ there exists $j$ such that $\|x-x_j\|_2\leq\delta(N,\epsilon)$. Conditioning on this event and fixing $x\in\supp(\PB)$ with corresponding nearby $x_j$, the triangle inequality applies to yield
    \begin{equation}
        \|f(x)-f^*(x)\|_2\leq \|f(x)-f(x_j)\|_2+\|f(x_j)-f^*(x_j)\|_2+\|f^*(x_j)-f^*(x)\|_2,
    \end{equation}
    which gives
    \begin{equation}
        \|f(x)-f^*(x)\|_2\leq \delta(N,\epsilon)\big(\|f\|_{Lip,\PB}+\|f^*\|_{Lip}\big)
    \end{equation}
    after applying the Lipschitz identities and invoking the hypothesis that $f$ and $f^*$ agree on training data. The proof of Theorem \ref{thm:jacobianmaximum} now applies to show that, conditioned on this same event, one has
    \begin{equation}
        \|f(x)-f^*(x)\|_2\leq \delta(N,\epsilon)\big(\|f^*\|_{Lip}+\max_i\big(\|Jf(x_i)\|_2+V_{B_i}(Jf)\big)\big).
    \end{equation}
    Taking the $\PB$-expectation of both sides yields the result.
\end{proof}

\subsection{A note on Jacobian computations}\label{sec:bnjac}

In our experiments, the Jacobian norm and sharpnesss are estimated using the power method and the Hessian/Jacobian-vector product functions in PyTorch. Recalling that we use $\RB^{d\times N}$ to denote the set of data matrices, consisting of a batch of data column vectors concatenated together, there are two different kinds of functions $f:\RB^{d_0\times N}\rightarrow\RB^{d_1\times N}$ whose Jacobians we would like to compute. The first of these are functions defined columnwise by some other map $g:\RB^{d_0}\rightarrow\RB^{d_1}$: that is
\begin{equation}
    f(X)_{j} = g(X_j),
\end{equation}
where lower indices index columns so that $X_j\in\RB^{d_0}$ for all $j$.  In this case, flattening (columns first) shows that $Jf(X)$ is just a block-diagonal matrix, whose $j^{th}$ block is $Jf(X_j)$. Hence $\|Jf(X)\|_2 = \max_j\|Jf(X_j)\|_2$. For a network in evaluation mode, \emph{every} layer Jacobian is of this form.

On the other hand, for a network in \emph{train} mode (as is the case when evaluating the Hessian of the loss), batch normalisation (BN) layers do not have this form.  BN layers in train mode are \emph{nonlinear} transformations defined by
\begin{equation}\label{bn}
    bn(X):=\frac{X-\EB X}{(\epsilon + \sigma^2X)^{\frac{1}{2}}},
\end{equation}
where $\EB$ and $\sigma^2$ denote mean and variance computed across the column index (again we omit the parameters as they are not essential to the discussion).  The skeptical reader would rightly question whether the Jacobian of BN in train mode (to which Equation \eqref{eq:tangentkernel} applies) is sufficiently close to the Jacobian of BN in evaluation mode (to which Theorem \ref{thm:generalisationbound} applies) for Ansatz \ref{ansatz} to be valid for BN networks.  Our next and final theorem says that for sufficiently large datasets, this is not a problem.

\begin{theorem}\label{thm:batchnorm}
    Let $bn_T$ and $bn_V$ denote a batch normalisation layer in train and evaluation mode respectively.  Given a data matrix $X\in\RB^{d\times N}$, assume the coordinates of $X$ are $O(1)$.  Then $Jbn_T = Jbn_V + O(N^{-1})$.
\end{theorem}

\begin{proof}[Proof of Theorem \ref{thm:batchnorm}]
    Given a data matrix $X$, let its row-wise mean and variance, thought of as \emph{constant vectors}, be denoted by $\EB$ and $\sigma^2$.  The evaluation mode BN map $bn_V:\RB^{d\times N}\rightarrow\RB^{d\times N}$ is simply an affine transformation, with Jacobian coordinates given by
    \begin{equation}
        \frac{\partial (bn_V)^i_j}{\partial x^k_l} = \frac{\delta^i_k\delta^i_l}{(\epsilon+\sigma^2)^{\frac{1}{2}}}.
    \end{equation}
    By \cite[Lemma 8.3]{macdonald}, the train mode BN map $bn_T:\RB^{d\times N}\rightarrow\RB^{d\times N}$ can be thought of as the composite $v\circ m$, where $m,v:\RB^{d\times N}\rightarrow\RB^{d\times N}$ are defined by
    \begin{equation}
        m(X):=X - \frac{1}{N}X1_{N\times N}
    \end{equation}
    and
    \begin{equation}
        v(Y):=(\epsilon+N^{-1}\|Y\|^2_{row})^{-\frac{1}{2}}Y
    \end{equation}
    respectively.  Here $1_{N\times N}$ denotes the $N\times N$ matrix of 1s, and $\|Y\|_{row}$ denotes the vector of row-wise norms of $Y$.

    As in \cite[Lemma 8.3]{macdonald}, one has
    \begin{equation}
        \frac{\partial v^i_j}{\partial y^k_l}(Y) = \frac{\delta^i_k}{(\epsilon+N^{-1}\|Y^i\|^2)^{\frac{1}{2}}}\bigg(\delta^j_l - \frac{y^i_l y^i_j}{(\epsilon+N^{-1}\|Y^i\|^2)}\bigg),\qquad \frac{\partial m^i_j}{\partial x^k_l}(Y) = \delta^i_k(\delta^i_l - N^{-1})
    \end{equation}
    for any data matrix $Y$.  Thus, invoking the chain rule and simplifying, we therefore get
    \begin{equation}
        \frac{\partial (bn_T)^i_j}{\partial x^k_l}(X) = \frac{\delta^i_k\delta^i_l}{(\epsilon + \sigma^2)^{\frac{1}{2}}} - \frac{1}{N}\frac{\delta^i_k (x^i_l-\EB^i) (x^i_l-\EB^i)}{(\epsilon + \sigma^2)^{\frac{3}{2}}} + O(N^{-1})
    \end{equation}
    Since the components $x$ are $O(1)$ by hypothesis, the result follows.
\end{proof}

\section{Additional experimental results and details}

\subsection{Progressive sharpening}\label{app:fullbatch}

We trained ResNet18 and VGG11 (superficially modifying \url{https://github.com/kuangliu/pytorch-cifar/blob/master/models/resnet.py} and \url{https://github.com/chengyangfu/pytorch-vgg-cifar10/blob/master/vgg.py} respectively, with VGG11 in addition having its dropout layers removed, but BN layers retained) with full batch gradient descent on CIFAR10 with learning rates $0.08$, $0.04$ and $0.02$, using label smoothings of $0.0$, $0.5$ and $0.75$.  The models with no label smoothing were trained to 99 percent accuracy, and the number of iterations required to do this were then used as the number of iterations to train the models using nonzero label smoothing. Sharpness and Jacobian norm were computed every 5, 10, or 20 iterations depending on the number of iterations required for convergence of the non-label-smoothed model.  Five trials were conducted in each case, with mean and standard deviation plotted.  Line style indicates degree of label smoothing.

The dataset was standardised so that each vector component is centered, with unit standard deviation. 
 Full batch gradient descent was approximated by averaging the gradients over 10 ``ghost batches'' of size 5000 each, as in \cite{coheneos}.  This is justified with BN networks by Theorem \ref{thm:batchnorm}.  The Jacobian norm we plot is for the Jacobian of the softmaxed model, as required by our derivation of Equation \eqref{eq:lipschitz}. The sharpness and Jacobian norms were estimated using a randomly chosen subset of 2500 data points. In all cases, the smoother labels are associated with less severe increase in Jacobian and sharpness during training, as predicted by Equation \eqref{eq:lipschitz}. Refer to \texttt{fullbatch.py} in the supplementary material for the code we used to run these experiments.

\begin{figure}[H]
    \centering
    \setkeys{Gin}{width=\linewidth}
    \begin{subfigure}{0.3\columnwidth}
        \includegraphics{progressivesharpening/vgg11/0.08/loss.pdf}
        \caption{Loss}
    \end{subfigure}\hfill
    \begin{subfigure}{0.3\columnwidth}
        \includegraphics{progressivesharpening/vgg11/0.08/hess.pdf}
        \caption{Sharpness}
    \end{subfigure}\hfill
    \begin{subfigure}{0.3\columnwidth}
    \includegraphics{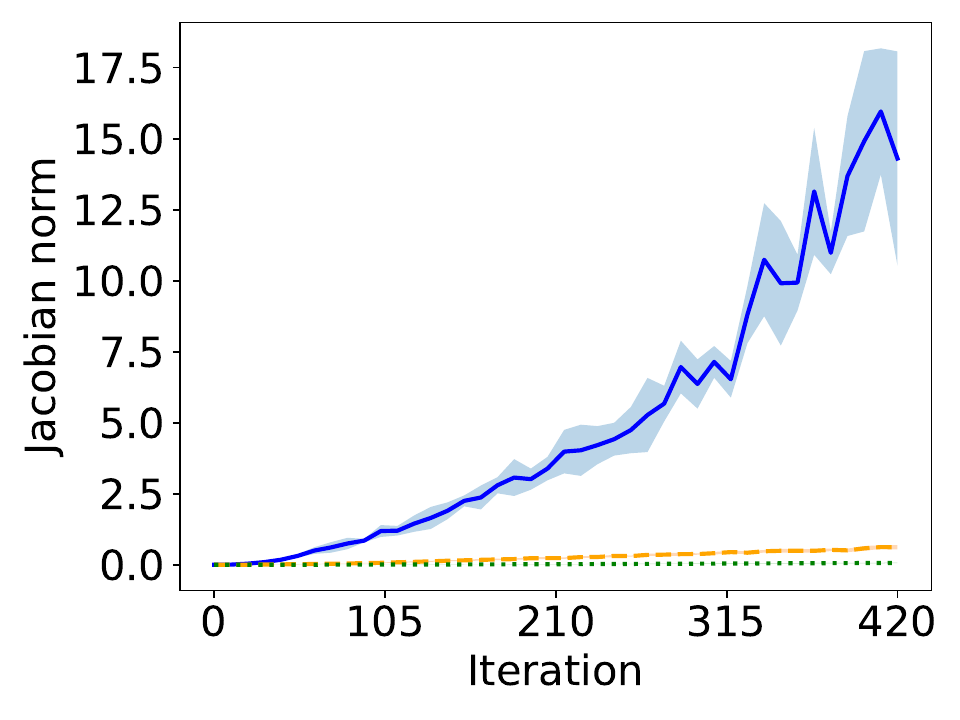}
    \caption{Jacobian norm}
    \end{subfigure}
    \caption{VGG11, learning rate 0.08}
\end{figure}

\begin{figure}[H]
    \centering
    \setkeys{Gin}{width=\linewidth}
    \begin{subfigure}{0.3\columnwidth}
        \includegraphics{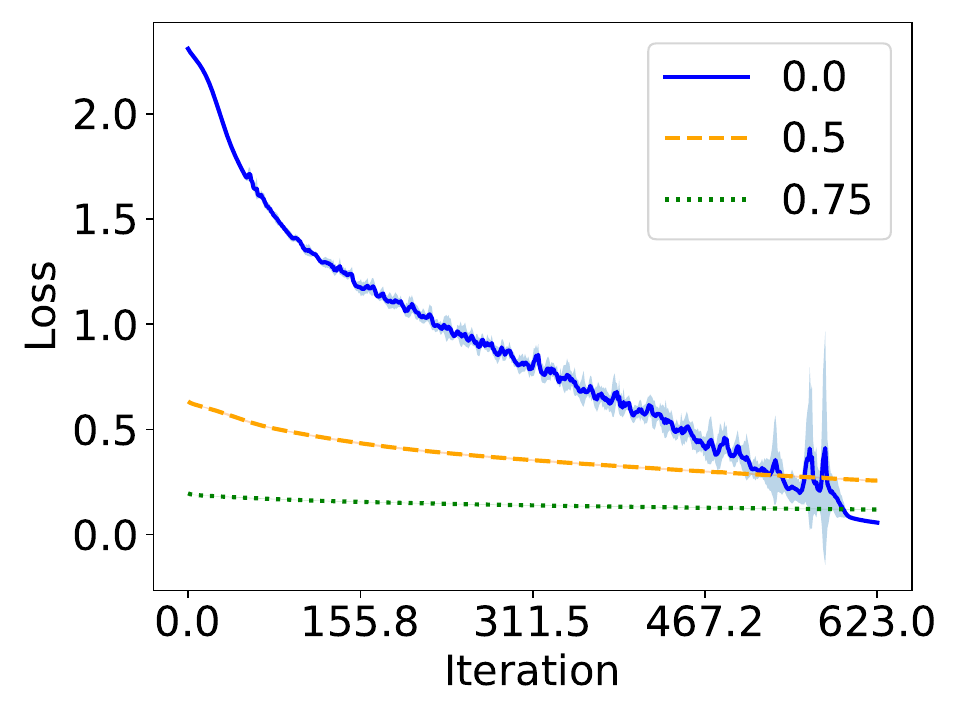}
        \caption{Loss}
    \end{subfigure}\hfill
    \begin{subfigure}{0.3\columnwidth}
        \includegraphics{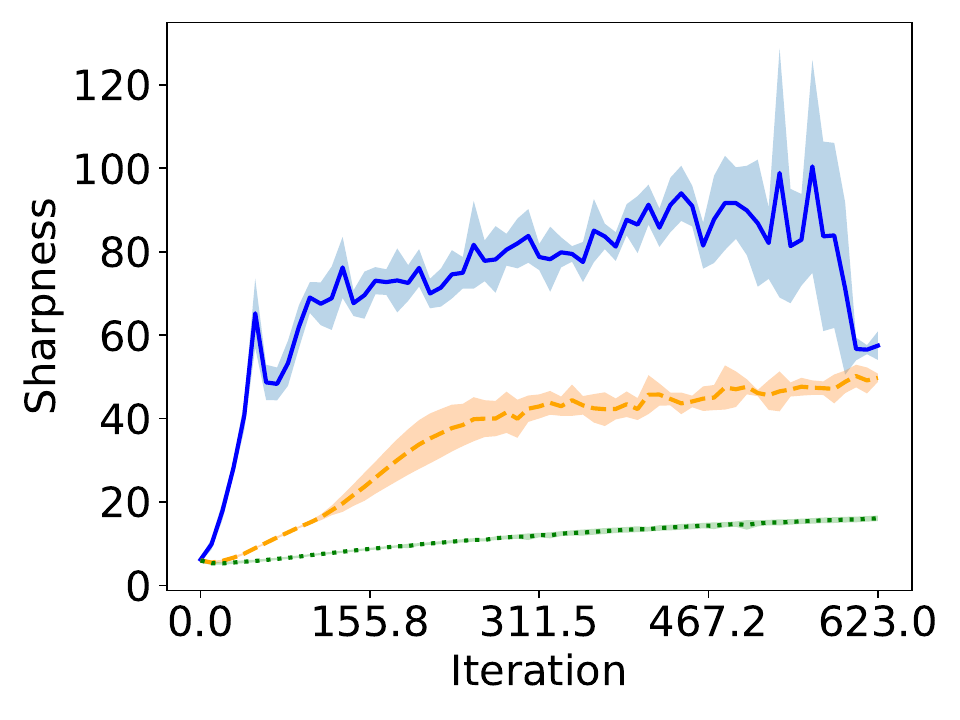}
        \caption{Sharpness}
    \end{subfigure}\hfill
    \begin{subfigure}{0.3\columnwidth}
    \includegraphics{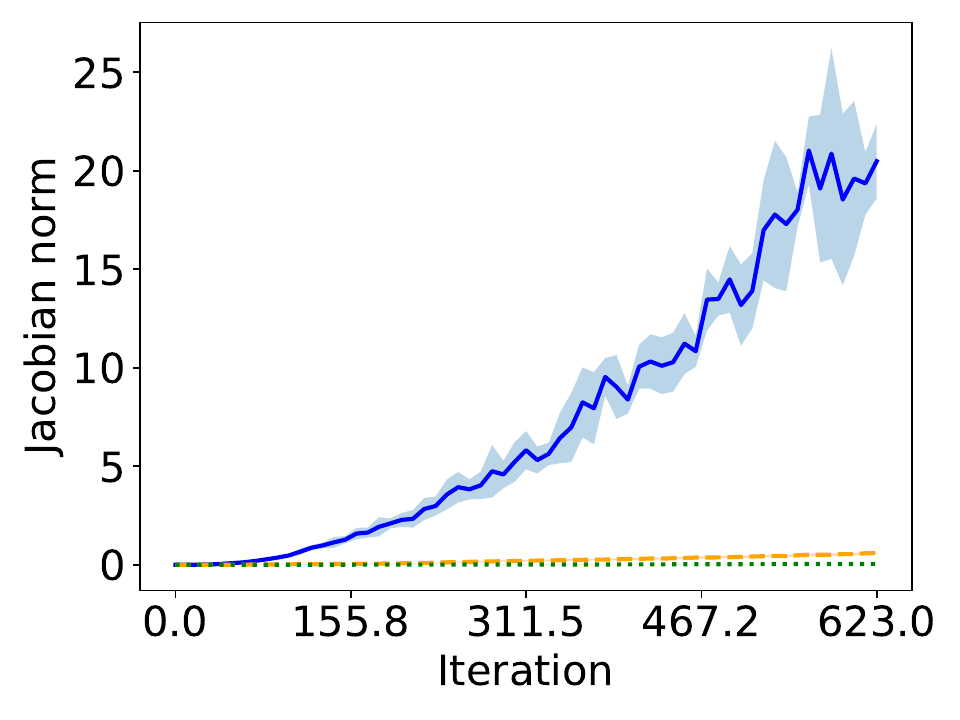}
    \caption{Jacobian norm}
    \end{subfigure}
    \caption{VGG11, learning rate 0.04}
\end{figure}

\begin{figure}[H]
    \centering
    \setkeys{Gin}{width=\linewidth}
    \begin{subfigure}{0.3\columnwidth}
        \includegraphics{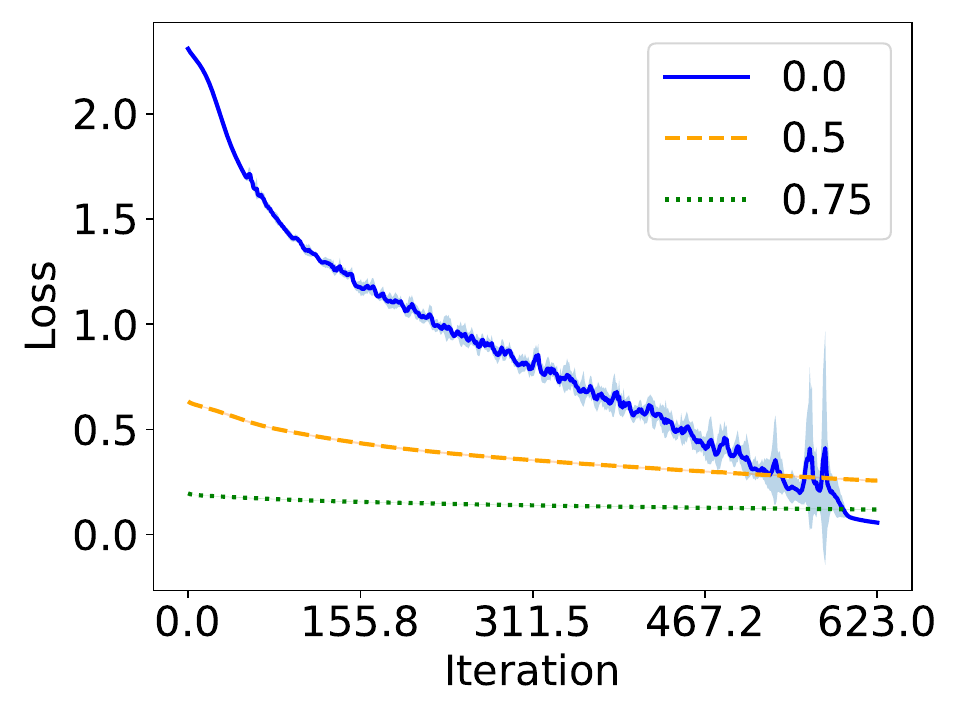}
        \caption{Loss}
    \end{subfigure}\hfill
    \begin{subfigure}{0.3\columnwidth}
        \includegraphics{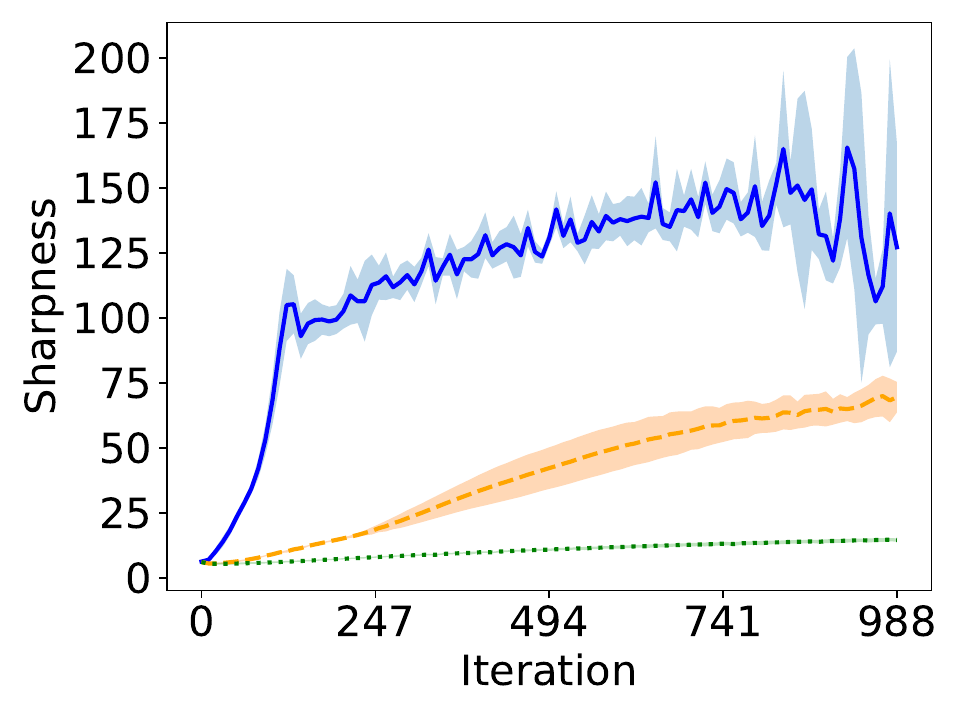}
        \caption{Sharpness}
    \end{subfigure}\hfill
    \begin{subfigure}{0.3\columnwidth}
    \includegraphics{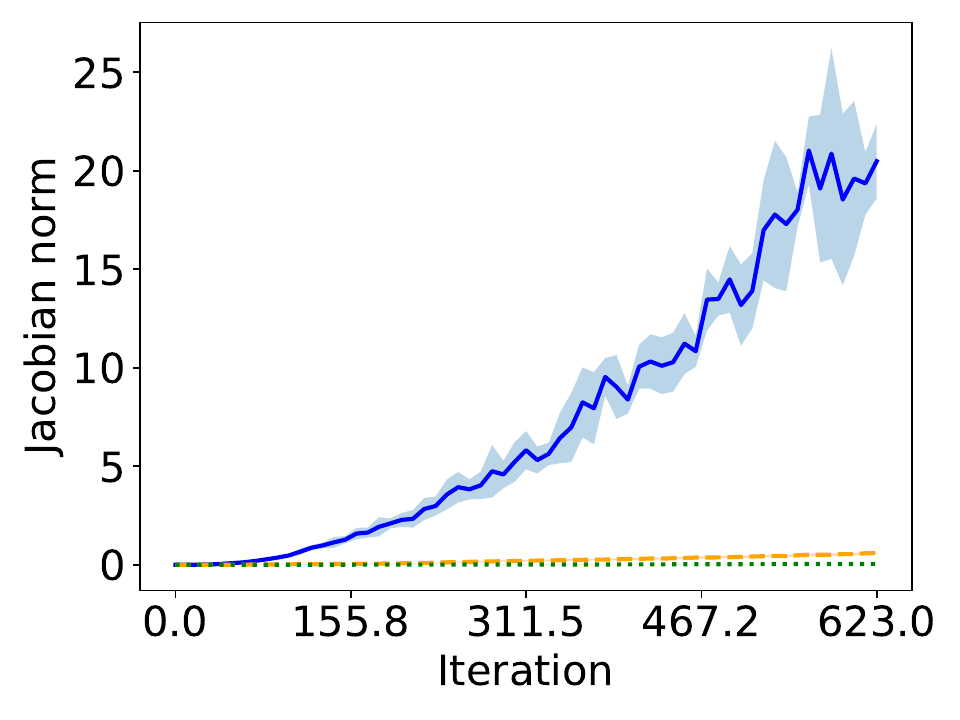}
    \caption{Jacobian norm}
    \end{subfigure}
    \caption{VGG11, learning rate 0.02}
\end{figure}

\begin{figure}[H]
    \centering
    \setkeys{Gin}{width=\linewidth}
    \begin{subfigure}{0.3\columnwidth}
        \includegraphics{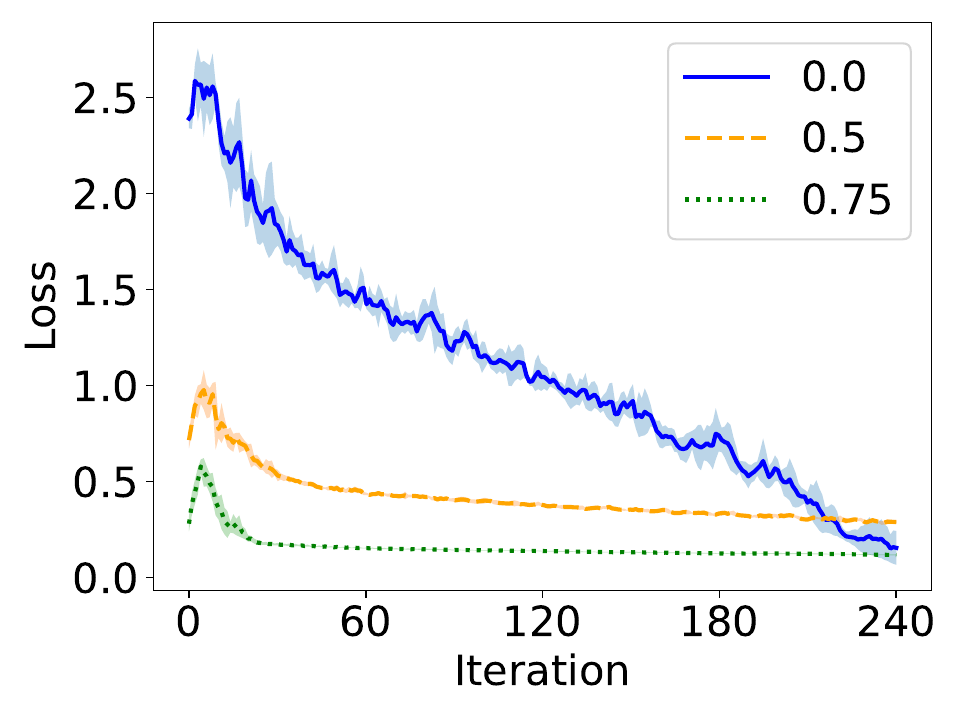}
        \caption{Loss}
    \end{subfigure}\hfill
    \begin{subfigure}{0.3\columnwidth}
        \includegraphics{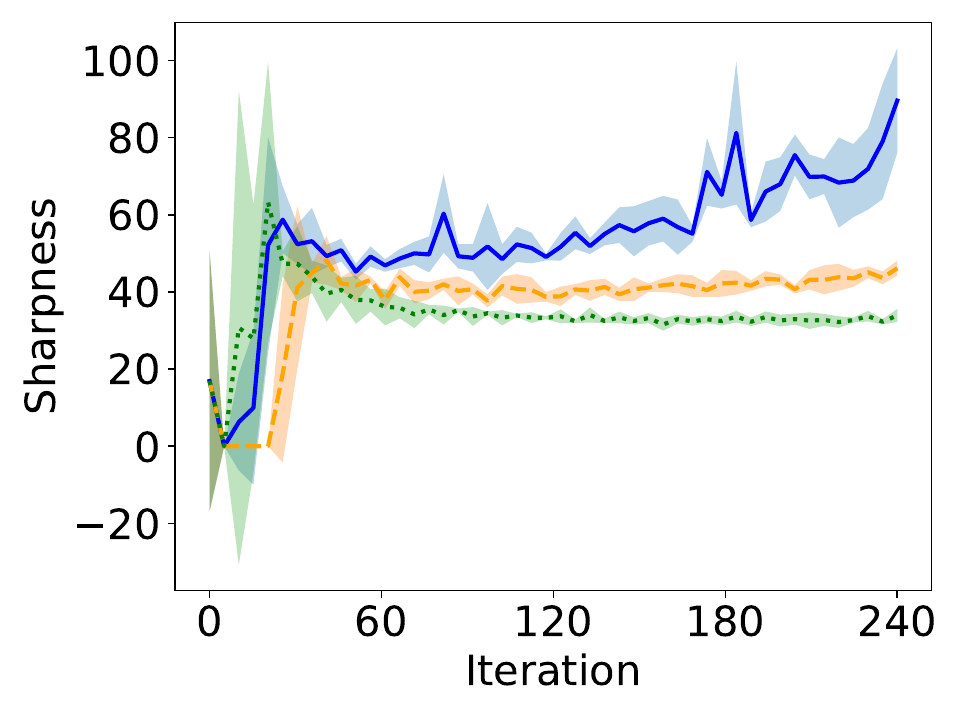}
        \caption{Sharpness}
    \end{subfigure}\hfill
    \begin{subfigure}{0.3\columnwidth}
    \includegraphics{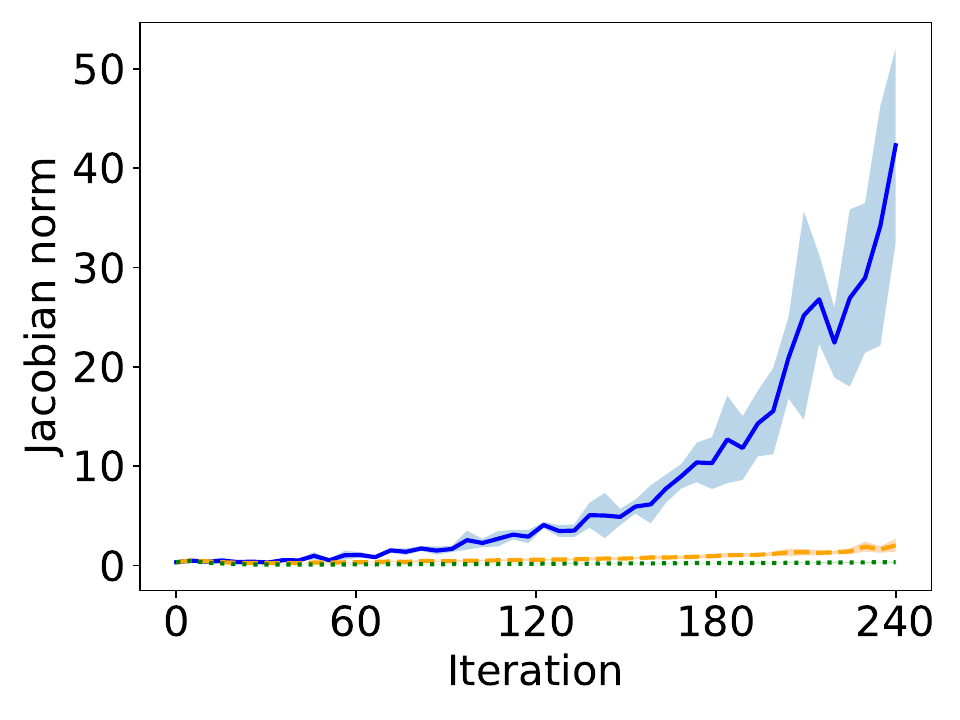}
    \caption{Jacobian norm}
    \end{subfigure}
    \caption{ResNet18, learning rate 0.08}
\end{figure}

\begin{figure}[H]
    \centering
    \setkeys{Gin}{width=\linewidth}
    \begin{subfigure}{0.3\columnwidth}
        \includegraphics{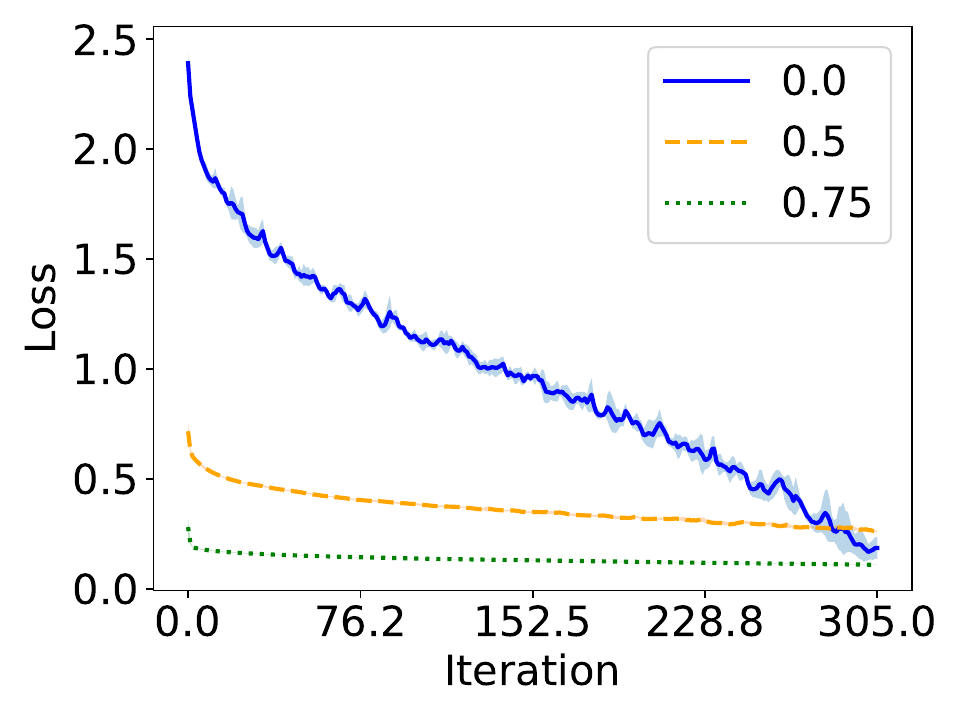}
        \caption{Loss}
    \end{subfigure}\hfill
    \begin{subfigure}{0.3\columnwidth}
        \includegraphics{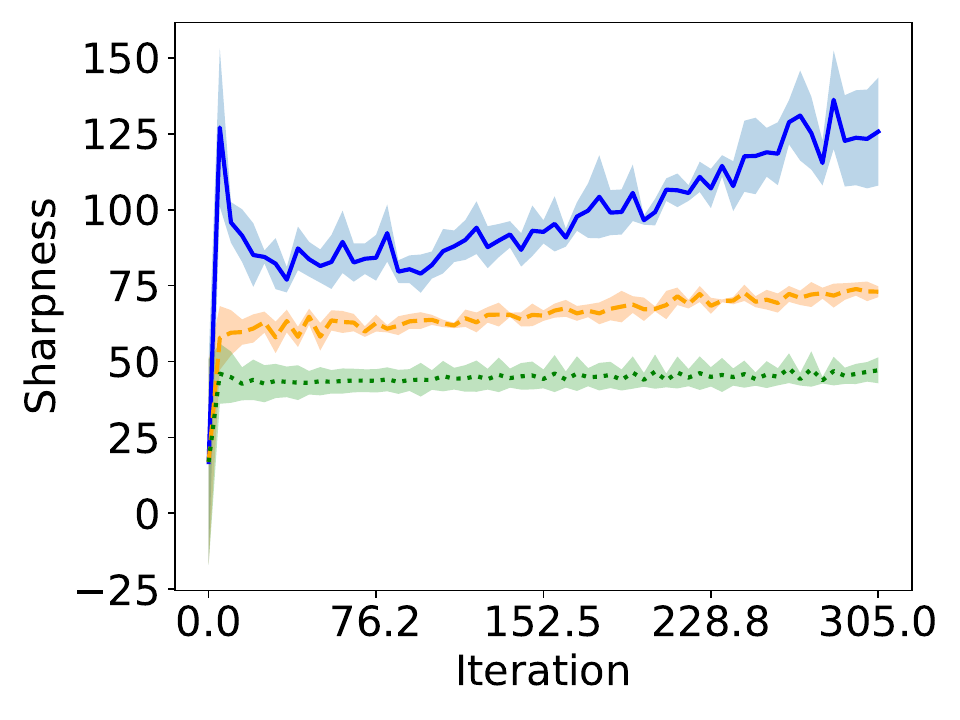}
        \caption{Sharpness}
    \end{subfigure}\hfill
    \begin{subfigure}{0.3\columnwidth}
    \includegraphics{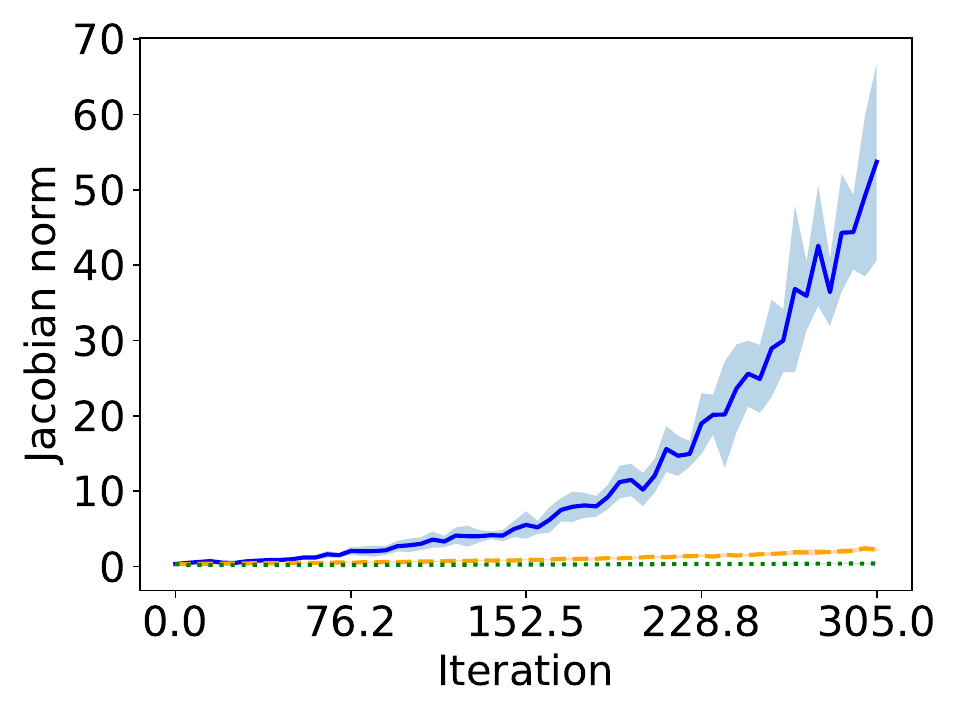}
    \caption{Jacobian norm}
    \end{subfigure}
    \caption{ResNet18, learning rate 0.04}
\end{figure}

\begin{figure}[H]
    \centering
    \setkeys{Gin}{width=\linewidth}
    \begin{subfigure}{0.3\columnwidth}
        \includegraphics{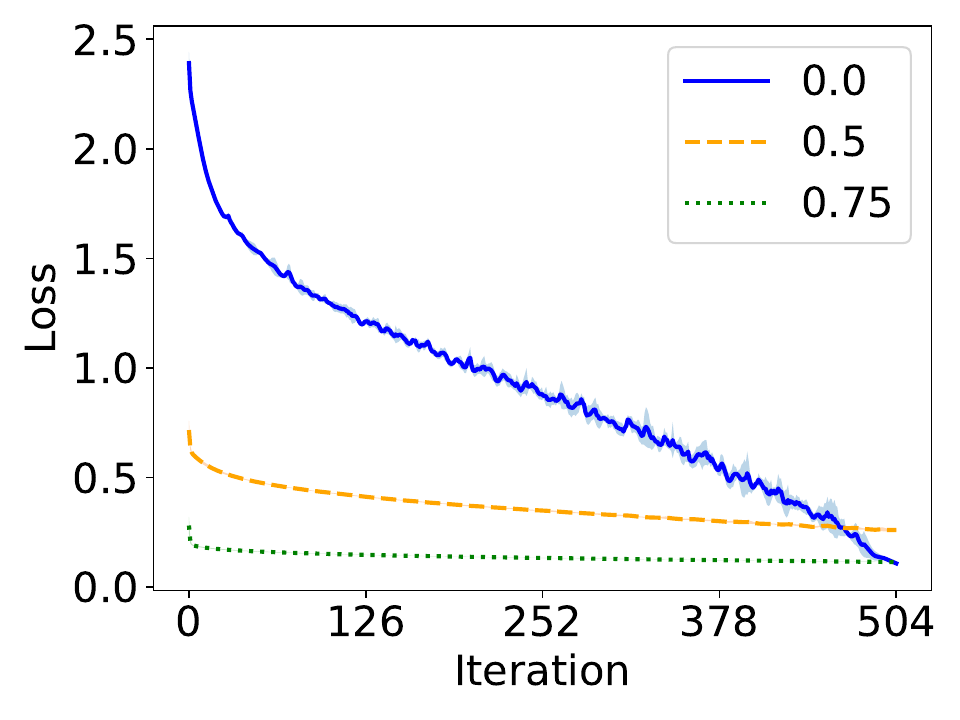}
        \caption{Loss}
    \end{subfigure}\hfill
    \begin{subfigure}{0.3\columnwidth}
        \includegraphics{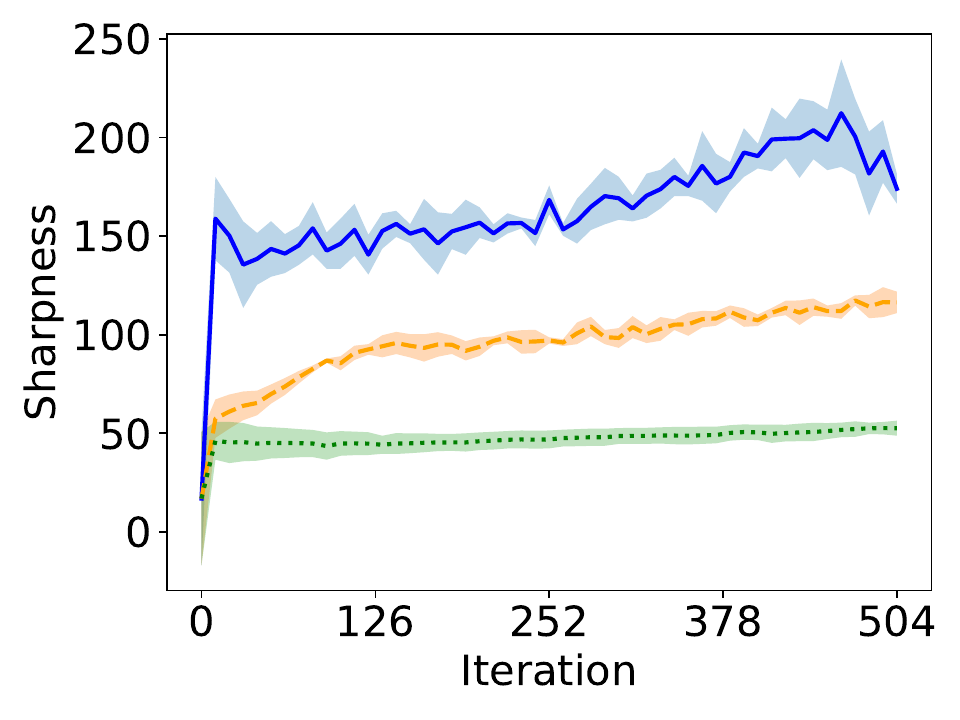}
        \caption{Sharpness}
    \end{subfigure}\hfill
    \begin{subfigure}{0.3\columnwidth}
    \includegraphics{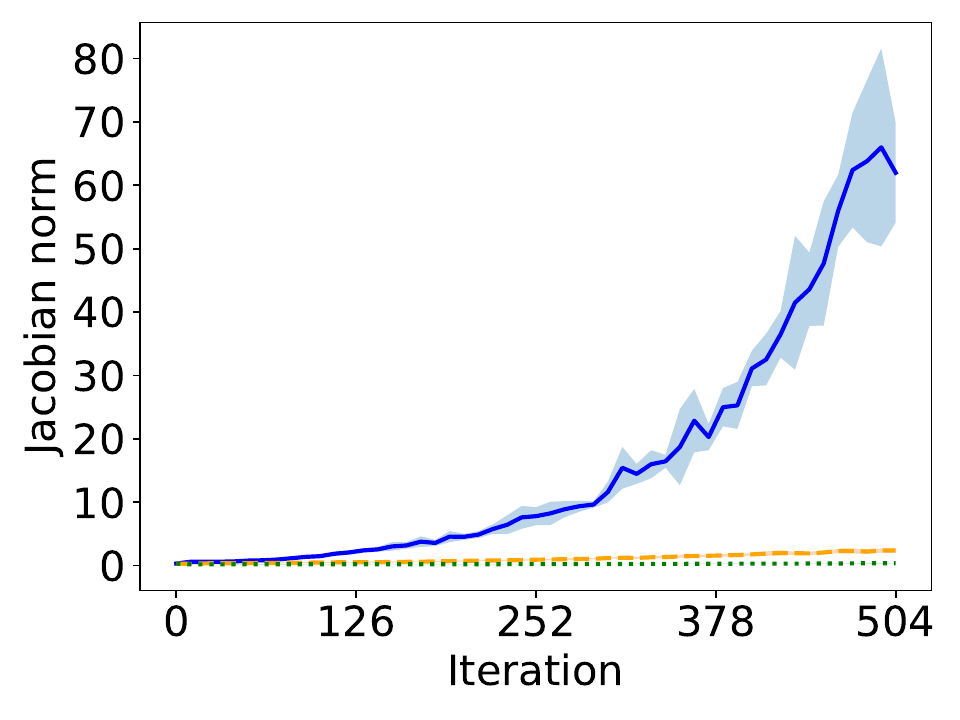}
    \caption{Jacobian norm}
    \end{subfigure}
    \caption{ResNet18, learning rate 0.02}
\end{figure}

\subsection{Batch size and learning rate}\label{app:batch_lr}

We trained a VGG11 and ResNet18 (superficially modifying \url{https://github.com/chengyangfu/pytorch-vgg-cifar10/blob/master/vgg.py} and \url{https://github.com/kuangliu/pytorch-cifar/blob/master/models/resnet.py} respectively, additionally removing dropout and retaining BN in the VGG11) on CIFAR10 and CIFAR100 at constant learning rates 0.1, 0.01, and 0.001, with batch sizes of 64, 128, 256, 512 and 1024.  The data were normalised by their RGB-channelwise mean and standard deviation, computed across all pixel coordinates and all elements of the training set.  Five trials of each were conducted.  Training loss on the full data set was computed every 100 iterations, and training was terminated when the average of the most recent 10 such losses was smaller than 0.01 for CIFAR10 and 0.02 for CIFAR100.  Weight decay regularisation of 0.0001 was used throughout. Refer to \texttt{batch\_lr.py} in the supplementary material.

At the termination of training, training loss and test loss were computed, with estimates of the Jacobian norm and sharpness computed using a randomly selected subset of size 2000 from the training set. The same seed was used to generate the parameters for a given trail as was used to generate the random subset of the training set.

We see that, on the whole, the claim that smaller batch sizes and larger learning rates lead to flatter minima and better generalisation is supported by our experiments, with Jacobian norm in particular being well-correlated to generalisation gap as anticipated by Theorem \ref{thm:generalisationbound} and Ansatz \ref{ansatz}. As noted in the main body of the paper, the notable exception to this is ResNet models trained with the largest learning rate, where Jacobian norm appears to underestimate the Lipschitz constant of the model and hence the generalisation gap. We speculate that this is due to large learning rate training of skip connected models leading to larger Jacobian Lipschitz constant, making the Jacobian norm a poorer estimate of the Lipschitz constant of the model and hence of generalisation (Theorems \ref{thm:jacobianmaximum} and \ref{thm:generalisationbound}).

\begin{figure}[H]
    \centering
    \setkeys{Gin}{width=\linewidth}
    \begin{subfigure}{0.24\columnwidth}
        \includegraphics{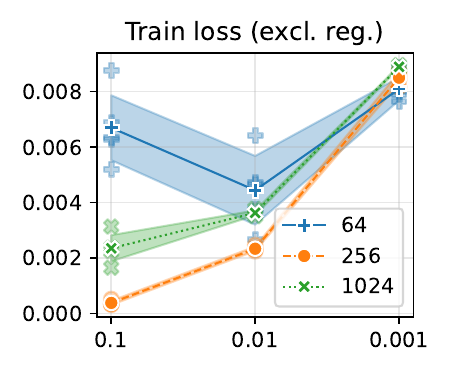}
    \end{subfigure}\hfill
    \begin{subfigure}{0.24\columnwidth}
        \includegraphics{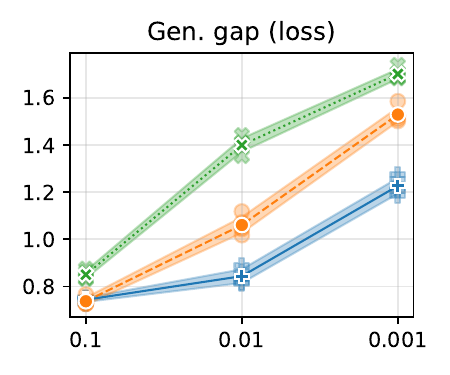}
    \end{subfigure}\hfill
    \begin{subfigure}{0.24\columnwidth}
        \includegraphics{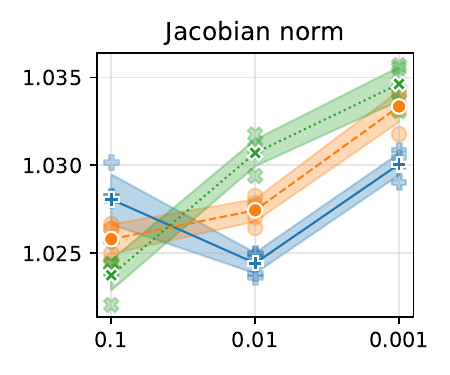}
    \end{subfigure}\hfill
    \begin{subfigure}{0.24\columnwidth}
    \includegraphics{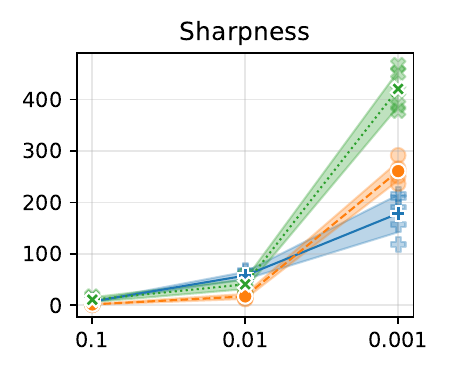}
    \end{subfigure}
    \caption{ResNet18 on CIFAR10, learning rate on $x$ axis. Line style indicates batch size.}
\end{figure}

\begin{figure}[H]
    \centering
    \setkeys{Gin}{width=\linewidth}
    \begin{subfigure}{0.24\columnwidth}
        \includegraphics{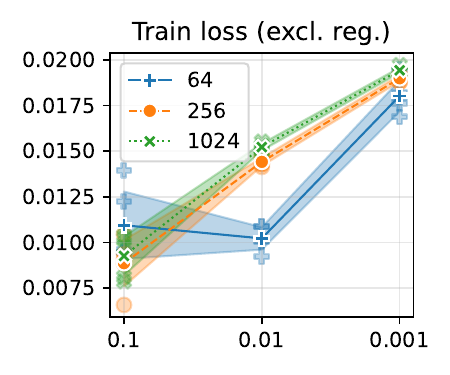}
    \end{subfigure}\hfill
    \begin{subfigure}{0.24\columnwidth}
        \includegraphics{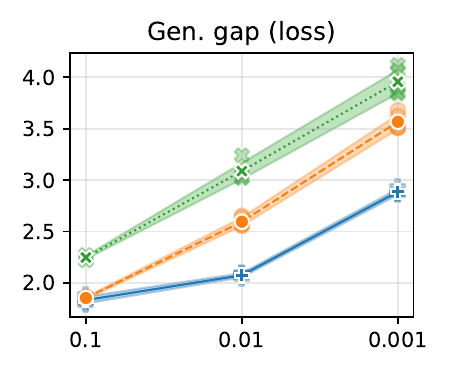}
    \end{subfigure}\hfill
    \begin{subfigure}{0.24\columnwidth}
        \includegraphics{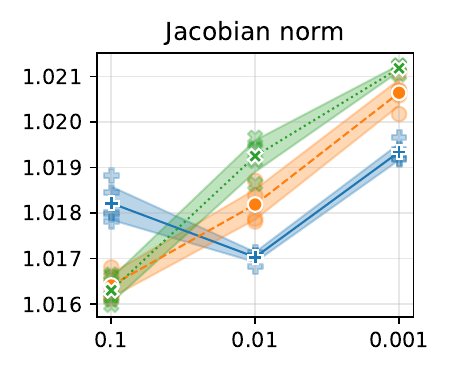}
    \end{subfigure}\hfill
    \begin{subfigure}{0.24\columnwidth}
    \includegraphics{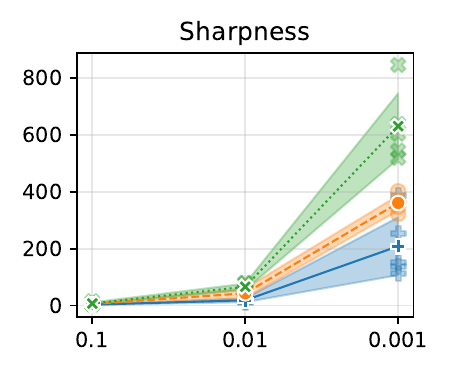}
    \end{subfigure}
    \caption{ResNet18 on CIFAR100, learning rate on $x$ axis. Line style indicates batch size.}
\end{figure}

\begin{figure}[H]
    \centering
    \setkeys{Gin}{width=\linewidth}
    \begin{subfigure}{0.24\columnwidth}
        \includegraphics{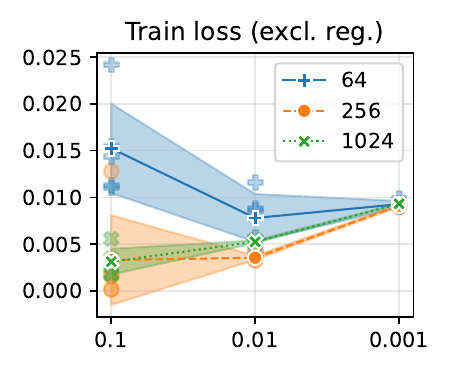}
    \end{subfigure}\hfill
    \begin{subfigure}{0.24\columnwidth}
        \includegraphics{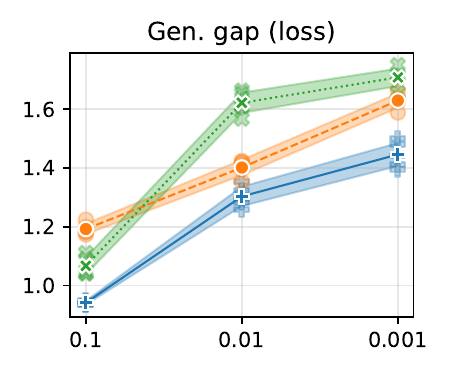}
    \end{subfigure}\hfill
    \begin{subfigure}{0.24\columnwidth}
        \includegraphics{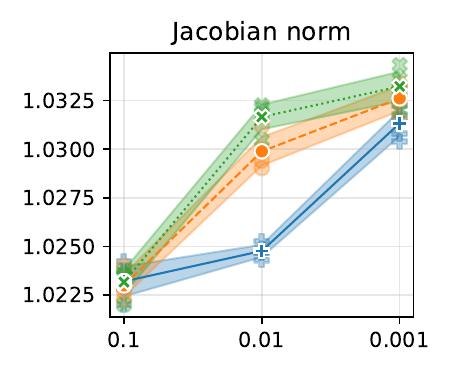}
    \end{subfigure}\hfill
    \begin{subfigure}{0.24\columnwidth}
    \includegraphics{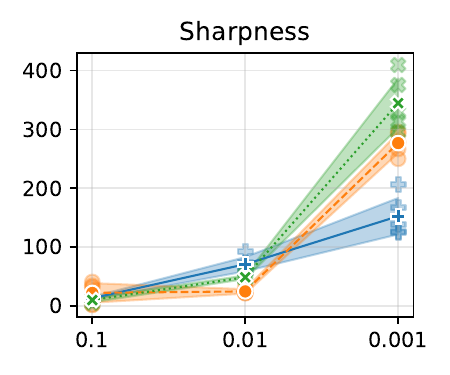}
    \end{subfigure}
    \caption{VGG11 on CIFAR10, learning rate on $x$ axis. Line style indicates batch size.}
\end{figure}

\begin{figure}[H]
    \centering
    \setkeys{Gin}{width=\linewidth}
    \begin{subfigure}{0.24\columnwidth}
        \includegraphics{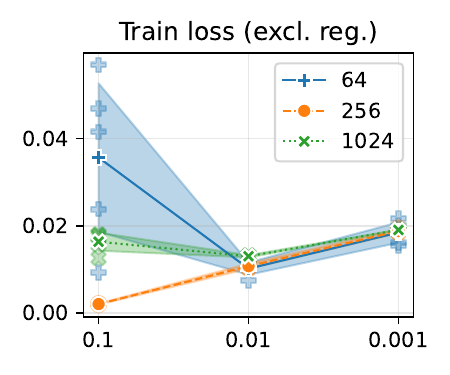}
    \end{subfigure}\hfill
    \begin{subfigure}{0.24\columnwidth}
        \includegraphics{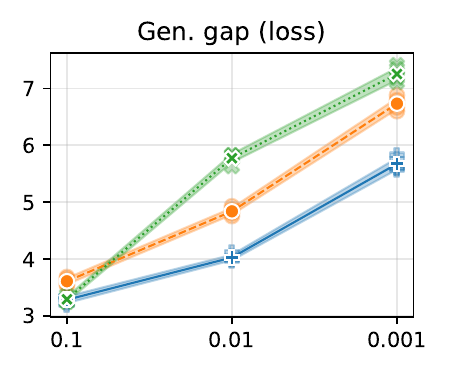}
    \end{subfigure}\hfill
    \begin{subfigure}{0.24\columnwidth}
        \includegraphics{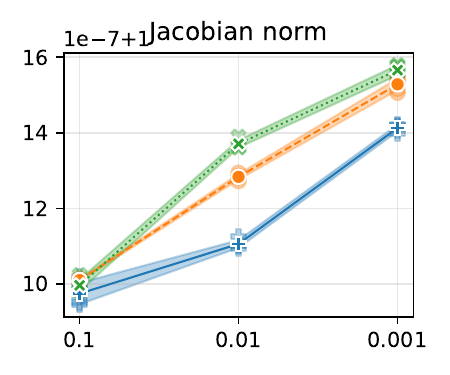}
    \end{subfigure}\hfill
    \begin{subfigure}{0.24\columnwidth}
    \includegraphics{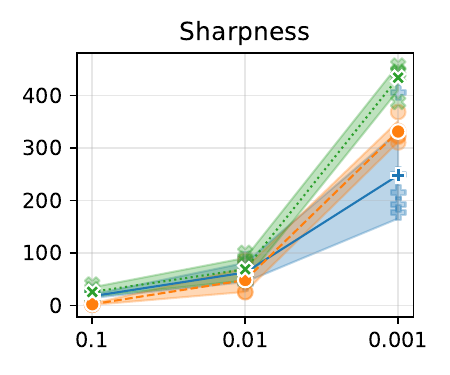}
    \end{subfigure}
    \caption{VGG11 on CIFAR100, learning rate on $x$ axis.  Line style indicates batch size.}
\end{figure}

\begin{figure}[H]
    \centering
    \setkeys{Gin}{width=\linewidth}
    \begin{subfigure}{0.24\columnwidth}
        \includegraphics{batchsize/cifar10/resnet/trainlossbatch.pdf}
    \end{subfigure}\hfill
    \begin{subfigure}{0.24\columnwidth}
        \includegraphics{batchsize/cifar10/resnet/gapbatch.pdf}
    \end{subfigure}\hfill
    \begin{subfigure}{0.24\columnwidth}
        \includegraphics{batchsize/cifar10/resnet/jac1evalbatch.pdf}
    \end{subfigure}\hfill
    \begin{subfigure}{0.24\columnwidth}
    \includegraphics{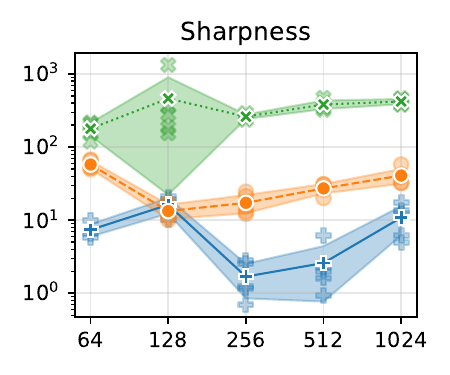}
    \end{subfigure}
    \caption{ResNet18 on CIFAR10, batch size on $x$ axis.  Line style indicates learning rate.}
\end{figure}

\begin{figure}[H]
    \centering
    \setkeys{Gin}{width=\linewidth}
    \begin{subfigure}{0.24\columnwidth}
        \includegraphics{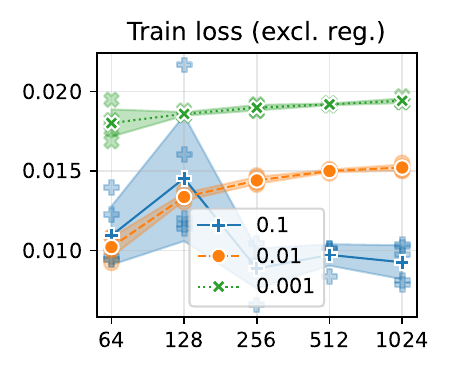}
    \end{subfigure}\hfill
    \begin{subfigure}{0.24\columnwidth}
        \includegraphics{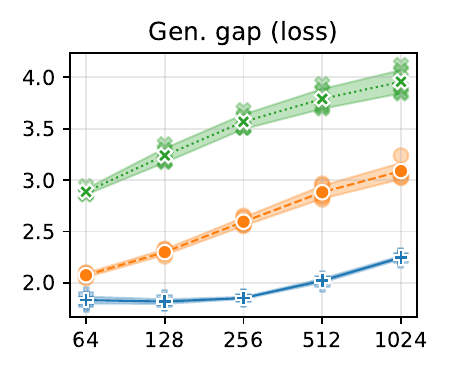}
    \end{subfigure}\hfill
    \begin{subfigure}{0.24\columnwidth}
        \includegraphics{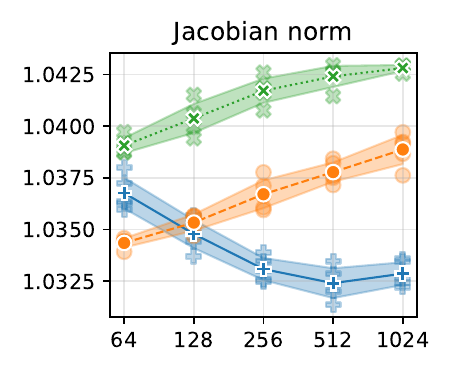}
    \end{subfigure}\hfill
    \begin{subfigure}{0.24\columnwidth}
    \includegraphics{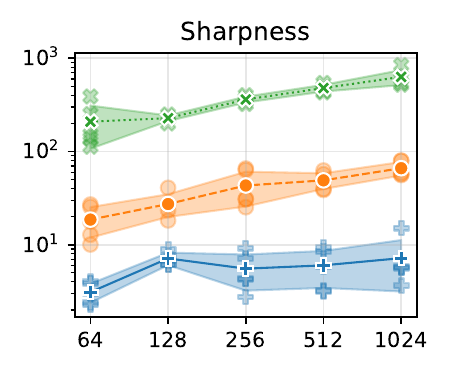}
    \end{subfigure}
    \caption{ResNet18 on CIFAR100, batch size on $x$ axis.  Line style indicates learning rate.}
\end{figure}

\begin{figure}[H]
    \centering
    \setkeys{Gin}{width=\linewidth}
    \begin{subfigure}{0.24\columnwidth}
        \includegraphics{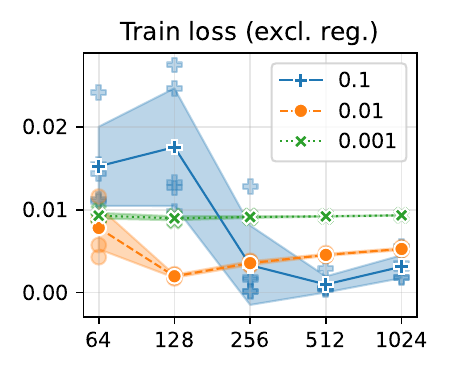}
    \end{subfigure}\hfill
    \begin{subfigure}{0.24\columnwidth}
        \includegraphics{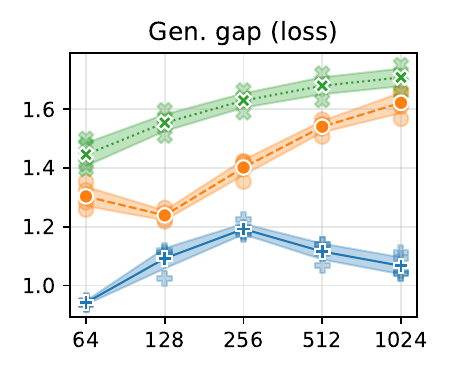}
    \end{subfigure}\hfill
    \begin{subfigure}{0.24\columnwidth}
        \includegraphics{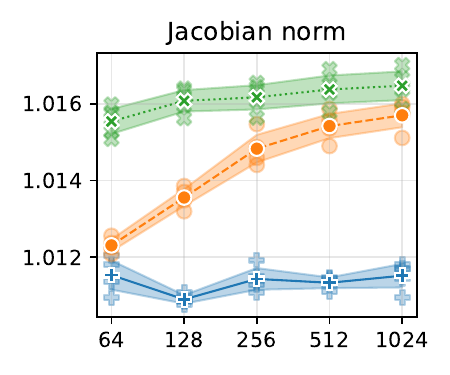}
    \end{subfigure}\hfill
    \begin{subfigure}{0.24\columnwidth}
    \includegraphics{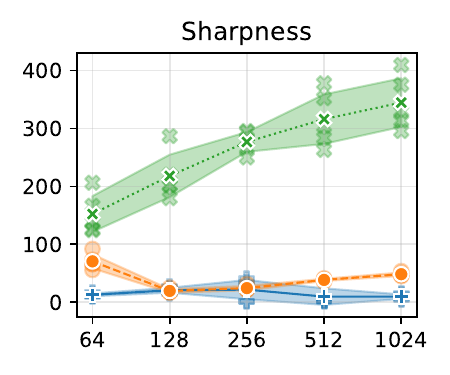}
    \end{subfigure}
    \caption{VGG11 on CIFAR10, batch size on $x$ axis.  Line style indicates learning rate.}
\end{figure}

\begin{figure}[H]
    \centering
    \setkeys{Gin}{width=\linewidth}
    \begin{subfigure}{0.24\columnwidth}
        \includegraphics{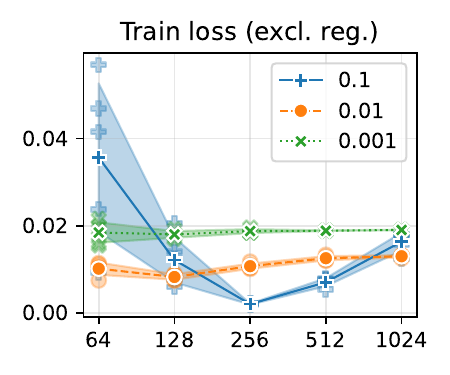}
    \end{subfigure}\hfill
    \begin{subfigure}{0.24\columnwidth}
        \includegraphics{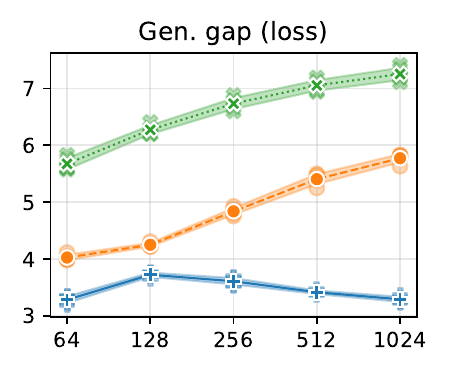}
    \end{subfigure}\hfill
    \begin{subfigure}{0.24\columnwidth}
        \includegraphics{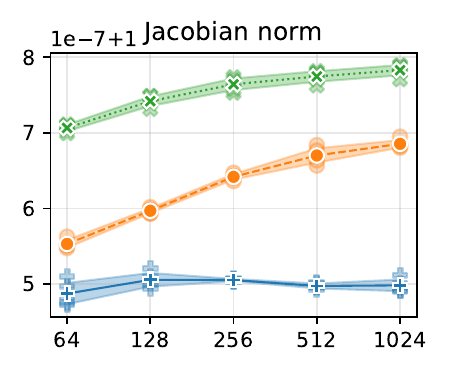}
    \end{subfigure}\hfill
    \begin{subfigure}{0.24\columnwidth}
    \includegraphics{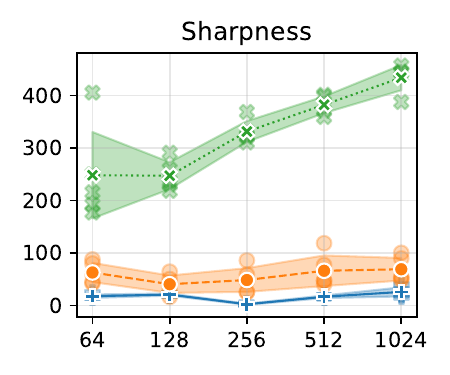}
    \end{subfigure}
    \caption{VGG11 on CIFAR100, batch size on $x$ axis.  Line style indicates learning rate.}
\end{figure}

\subsection{Extended practical results}\label{app:practical}

Here we provide the experimental details and additional experimental results for Section \ref{sec:generalisation}.
In addition to the generalisation gap, sharpness, and (input-output) Jacobian norm, these plots include the operator norm of the Gauss-Newton matrix, the loss on the training set, and the accuracy on the validation set.
Note that the figures in this section present the \emph{final} metrics at the conclusion of training as a function of the regularisation parameter, and hence progressive sharpening (as a function of time) will not be observed directly.
Rather, the purpose is to examine the impact of different regularisation strategies on the Hessian and model Jacobian.

In all cases, the model Jacobian is correlated with the generalisation gap, whereas the Hessian is often uncorrelated.
Moreover, the norm of the Gauss-Newton matrix is (almost) always very similar to that of Hessian, in line with previous empirical observation \cite{papyan1, papyan2, coheneos}.
The results support the arguments that (i) the model Jacobian is strongly related to the generalisation gap and (ii) while the Hessian is connected to the model Jacobian, and forcing the Hessian to be sufficiently small likewise appears to force the Jacobian to be smaller (see the SAM plots), the Hessian is also influenced by other factors and can be large even without the Jacobian being large (see other plots).
All of these phenomena are consistent with our ansatz.

The figures in this section present results for \{CIFAR10, CIFAR100\} $\times$ \{VGG11, ResNet18\} $\times$ \{batch-norm, no batch-norm\}.
Line colour and style denote different initial learning rates.
The models were trained for 90 epochs with a minibatch size of 128 examples using Polyak momentum of 0.9.
The learning rate was decayed by a factor of 10 after 50 and 80 epochs.
Models were trained using the softmax cross-entropy loss with weight decay of 0.0005.
(If weight decay were disabled, we would often observe the Hessian to collapse to zero as the training loss went to zero, due to the vanishing second gradient of the cost function in this region \cite[Appendix C]{coheneos}. This would make the correlation between sharpness and generalisation gap even worse.)
Refer to \texttt{train\_jax.py} and \texttt{slurm/launch\_wd\_sweep.sh} for the default configuration and experiment configuration, respectively.

For each regularisation strategy, the degree of regularisation is parametrised as follows.
Label smoothing considers smoothed labels $(1 - \alpha) y + \alpha (1 / n)$ with $\alpha \in [0, 1]$.
Mixup takes a convex combination of two examples $(1-\theta) (x, y) + \theta (x', y')$.
The coefficient $\theta$ is drawn from a symmetric beta distribution $\theta \sim \operatorname{Beta}(\beta, \beta)$ with $\beta \in [0, \infty)$, which corresponds to a Bernoulli distribution when $\beta = 0$ and a uniform distribution when $\beta = 1$.
Data augmentation modifies the inputs~$x$ with probability~$p \in [0, 1]$; the image transforms which we adopt are the standard choices for CIFAR (four-pixel padding, random crop, random horizontal flip).
Sharpness Aware Minimisation (SAM) restricts the distance~$\rho$ between the current parameter vector and that which is used to compute the update, which simplifies to SGD when $\rho = 0$.

For all results presented in this section, the loss and generalisation gap were computed using ``clean'' training examples; that is, without mixup or data augmentation in the respective experiments (despite this requiring an additional evaluation of the model for each step).
The Hessian, Gauss-Newton and input-output Jacobian norms were similarly computed using a random batch of 1000 clean training examples.
The losses and matrix norms in this section similarly exclude weight decay.
Label smoothing was interpreted as a modification of the loss function rather than the training distribution, and therefore \emph{was} included in the calculation of the losses and matrix norms.
Batch-norm layers were used in ``train mode'' (statistical moments computed from batch) for the Hessian and Gauss-Newton matrix, and in ``eval mode'' (statistical moments are constants) for the Jacobian matrix, although this difference was observed to have negligible effect on the Jacobian norm.

\subsubsection{CIFAR10}

\begin{figure}[H]
\centering
\makebox[\textwidth][c]{\includegraphics[height=190mm]{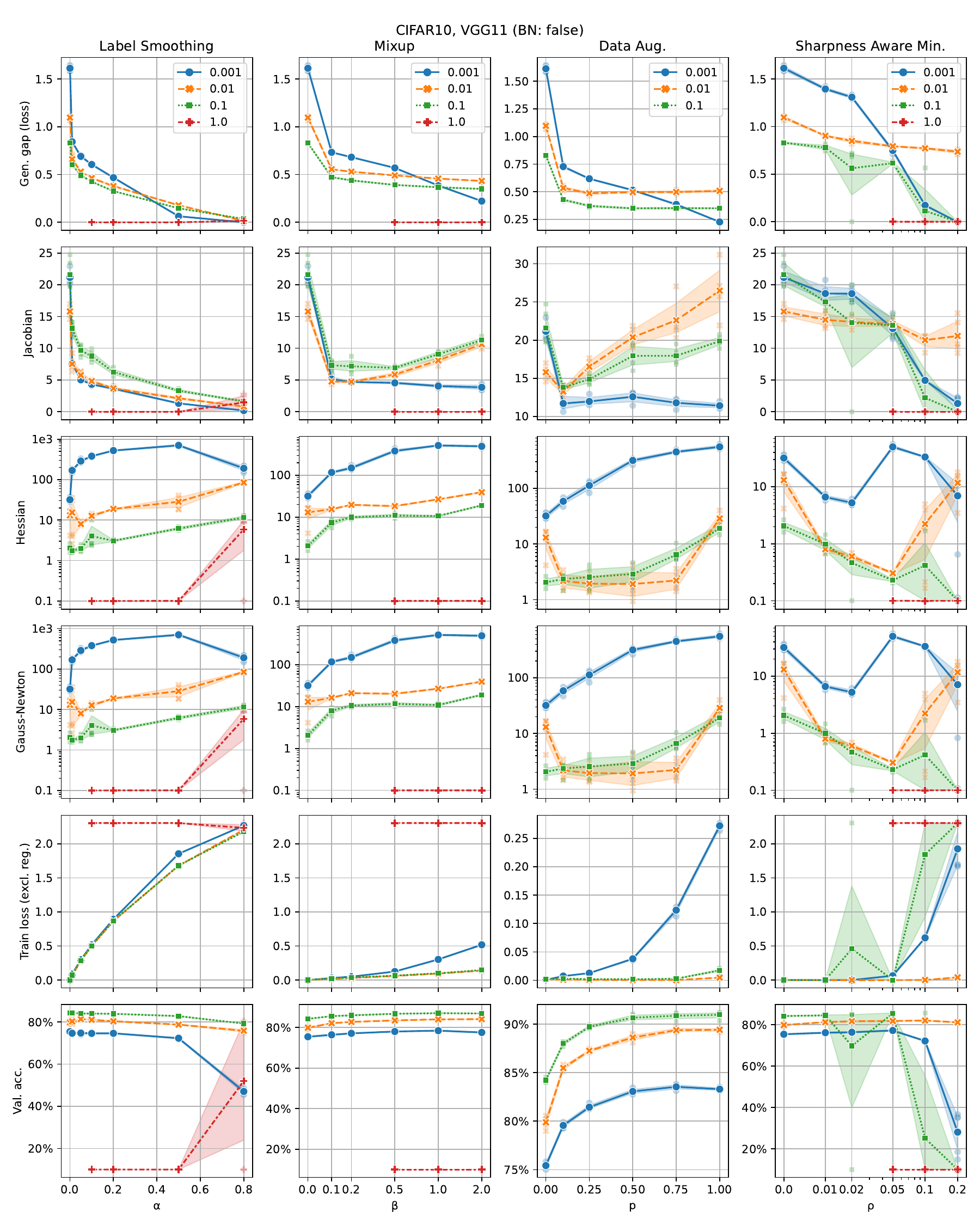}}
\caption{VGG11 without batch-norm on CIFAR10.}
\end{figure}
\begin{figure}[H]
\centering
\makebox[\textwidth][c]{\includegraphics[height=190mm]{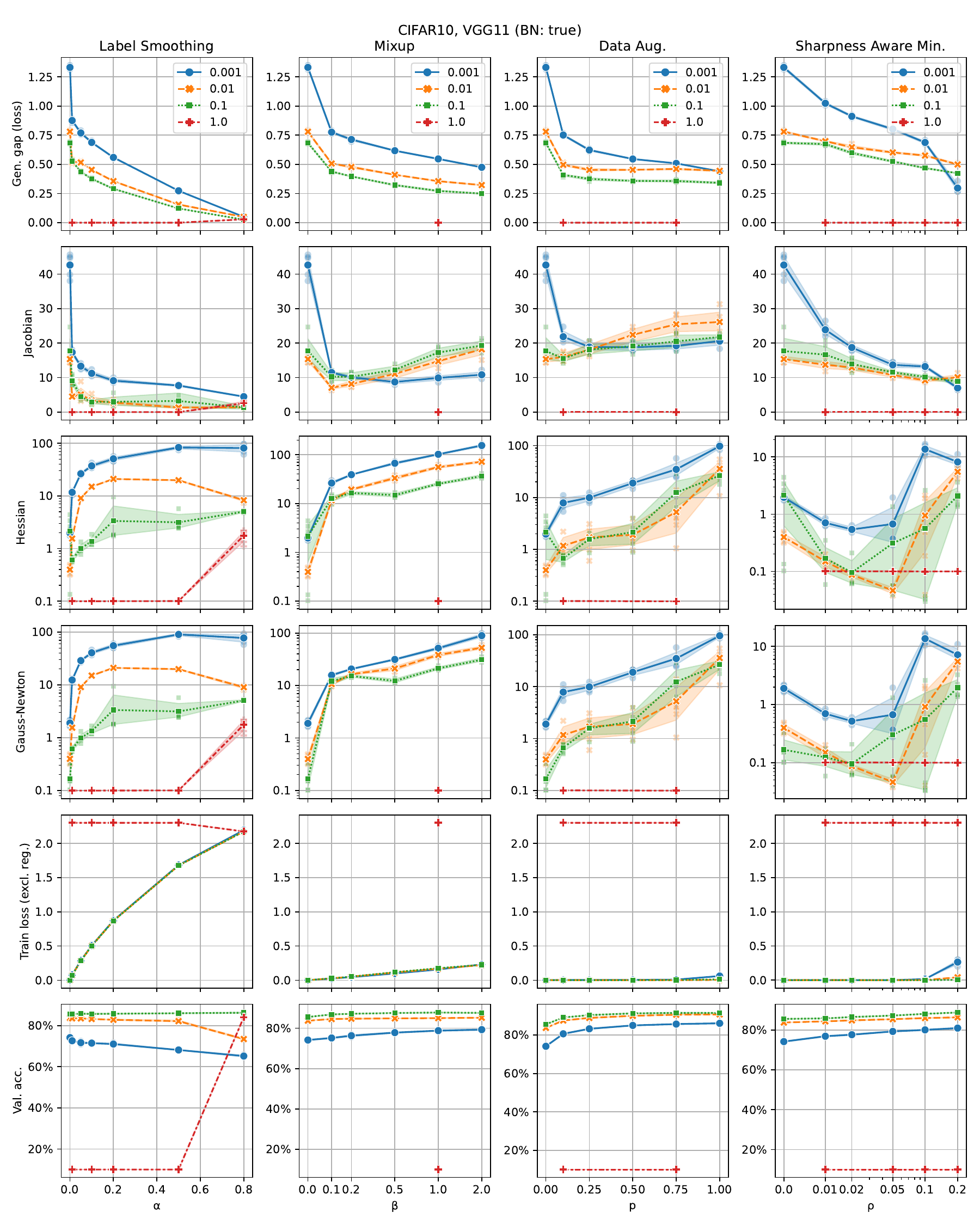}}
\caption{VGG11 with batch-norm on CIFAR10.}
\end{figure}

\begin{figure}[H]
\centering
\makebox[\textwidth][c]{\includegraphics[height=190mm]{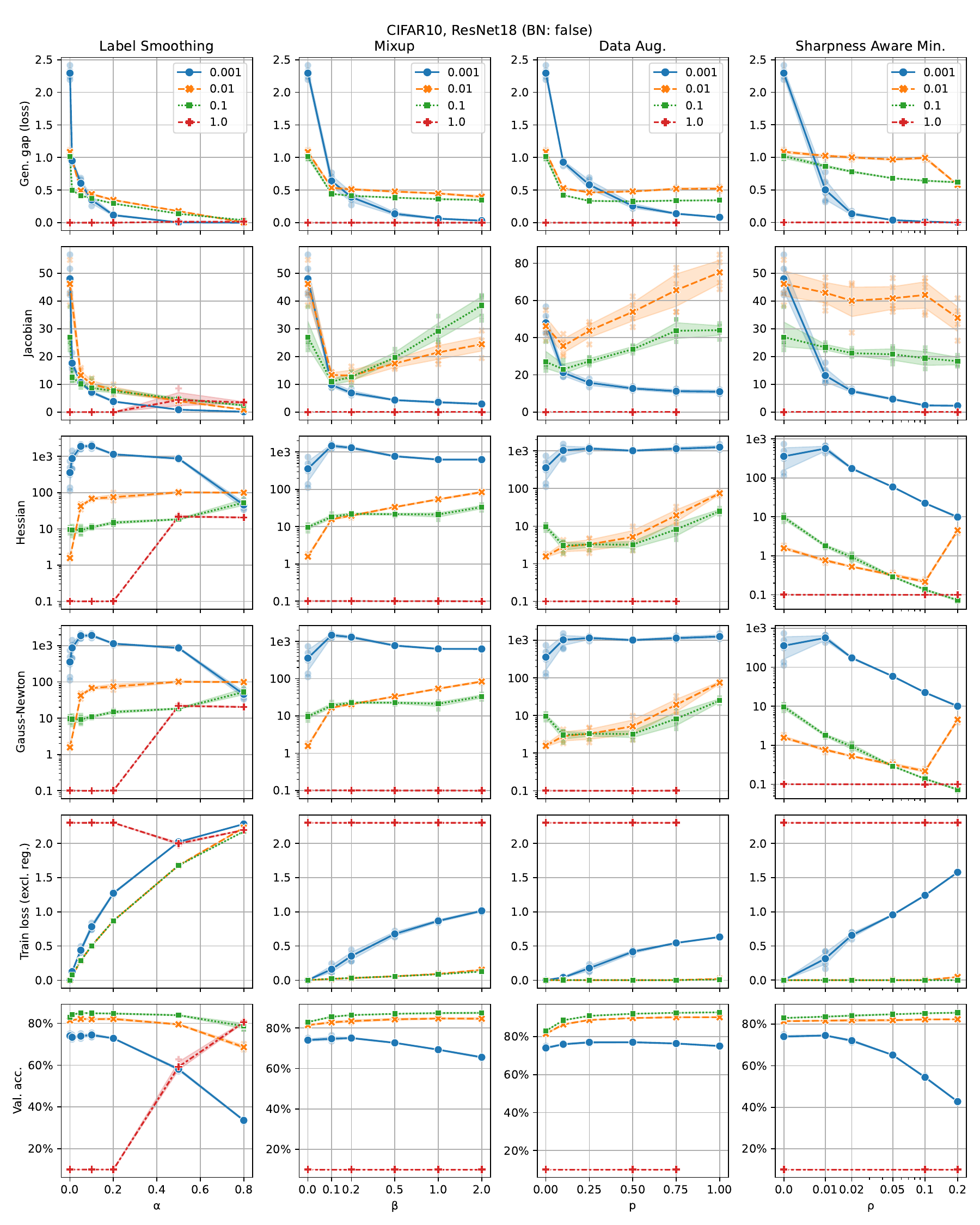}}
\caption{ResNet18 without batch-norm on CIFAR10.}
\end{figure}
\begin{figure}[H]
\centering
\makebox[\textwidth][c]{\includegraphics[height=190mm]{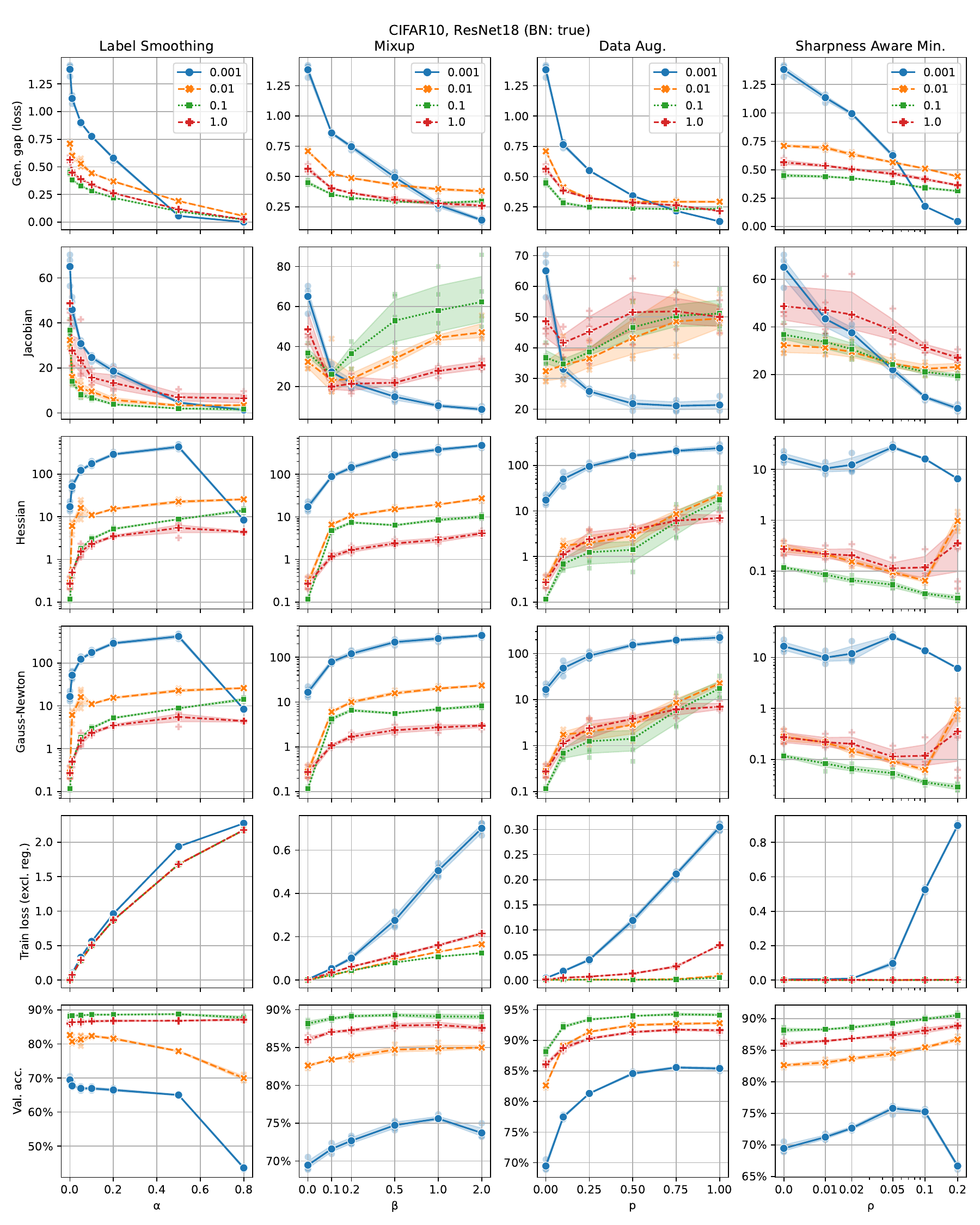}}
\caption{ResNet18 with batch-norm on CIFAR10.}
\end{figure}

\clearpage
\subsubsection{CIFAR100}

\begin{figure}[H]
\centering
\makebox[\textwidth][c]{\includegraphics[height=190mm]{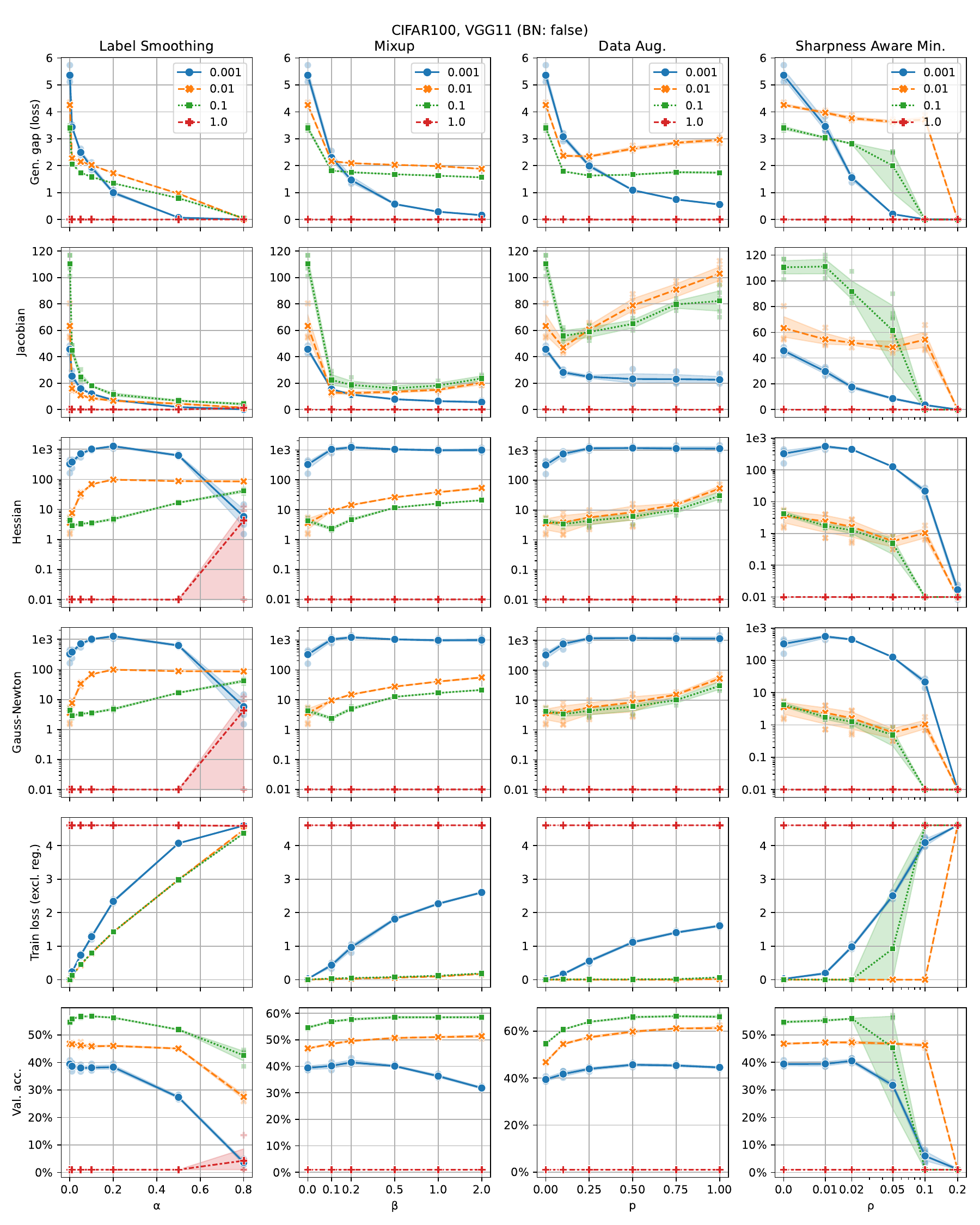}}
\caption{VGG11 without batch-norm on CIFAR100.}
\end{figure}
\begin{figure}[H]
\centering
\makebox[\textwidth][c]{\includegraphics[height=190mm]{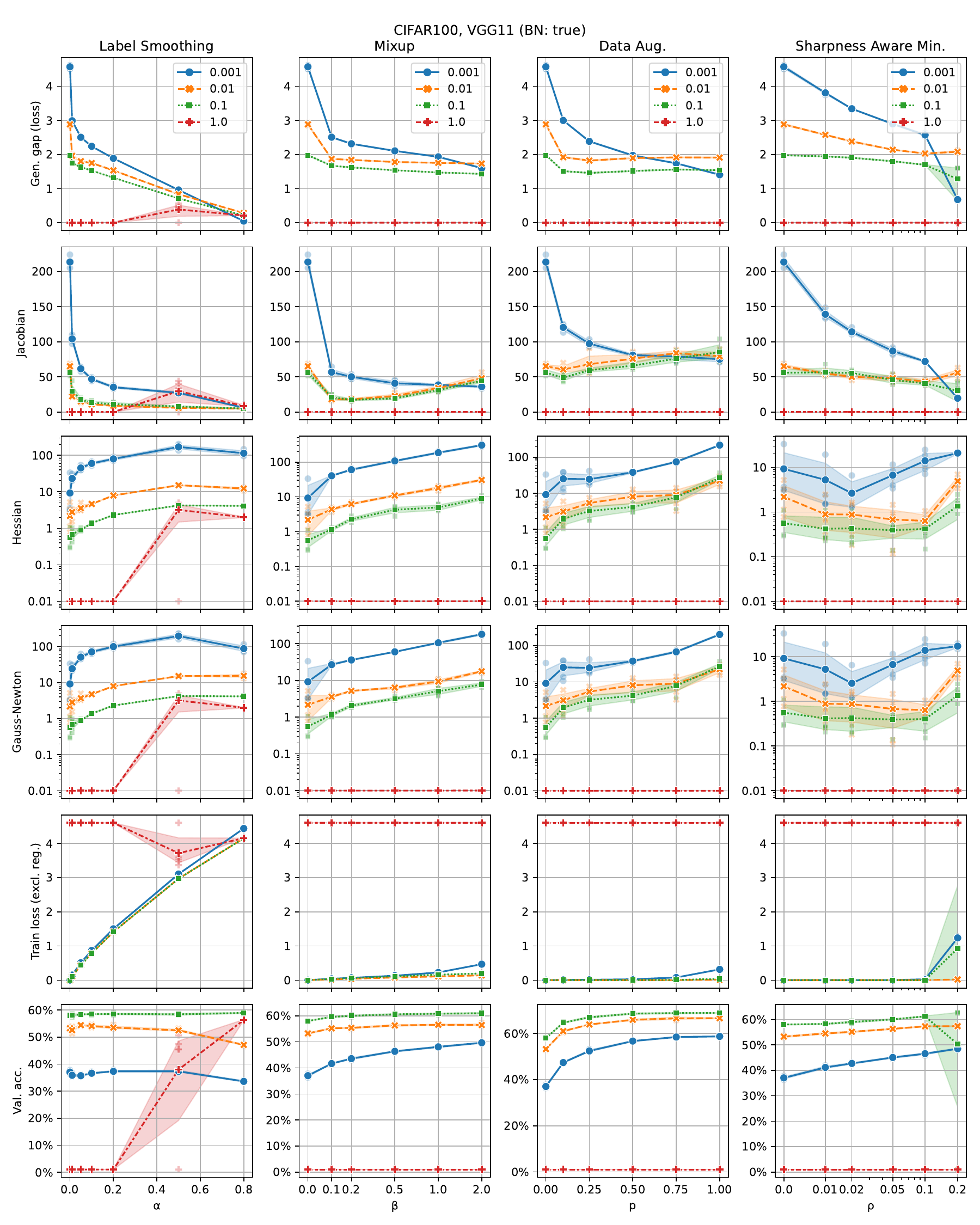}}
\caption{VGG11 with batch-norm on CIFAR100.}
\end{figure}

\begin{figure}[H]
\centering
\makebox[\textwidth][c]{\includegraphics[height=190mm]{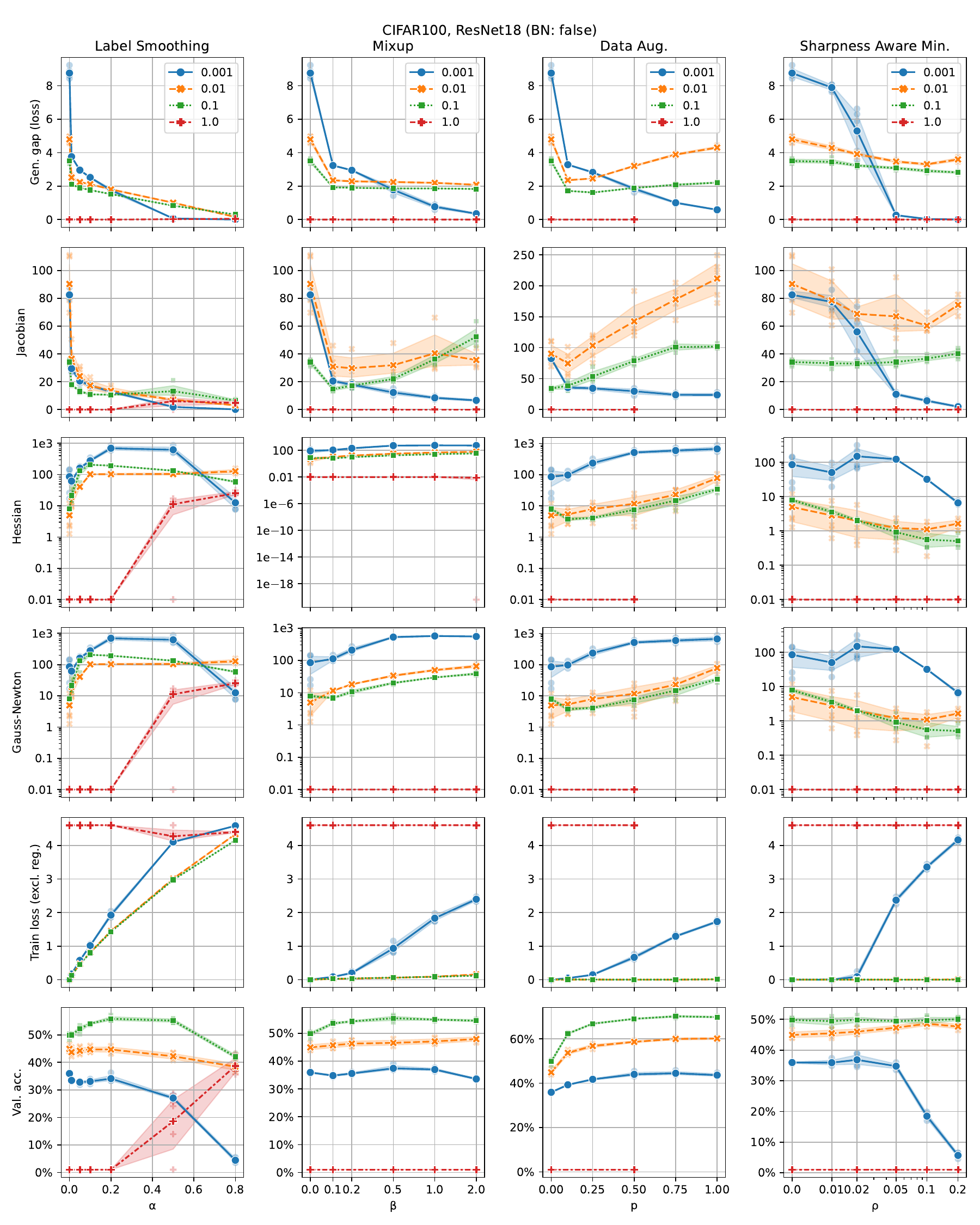}}
\caption{ResNet18 without batch-norm on CIFAR100.}
\end{figure}
\begin{figure}[H]
\centering
\makebox[\textwidth][c]{\includegraphics[height=190mm]{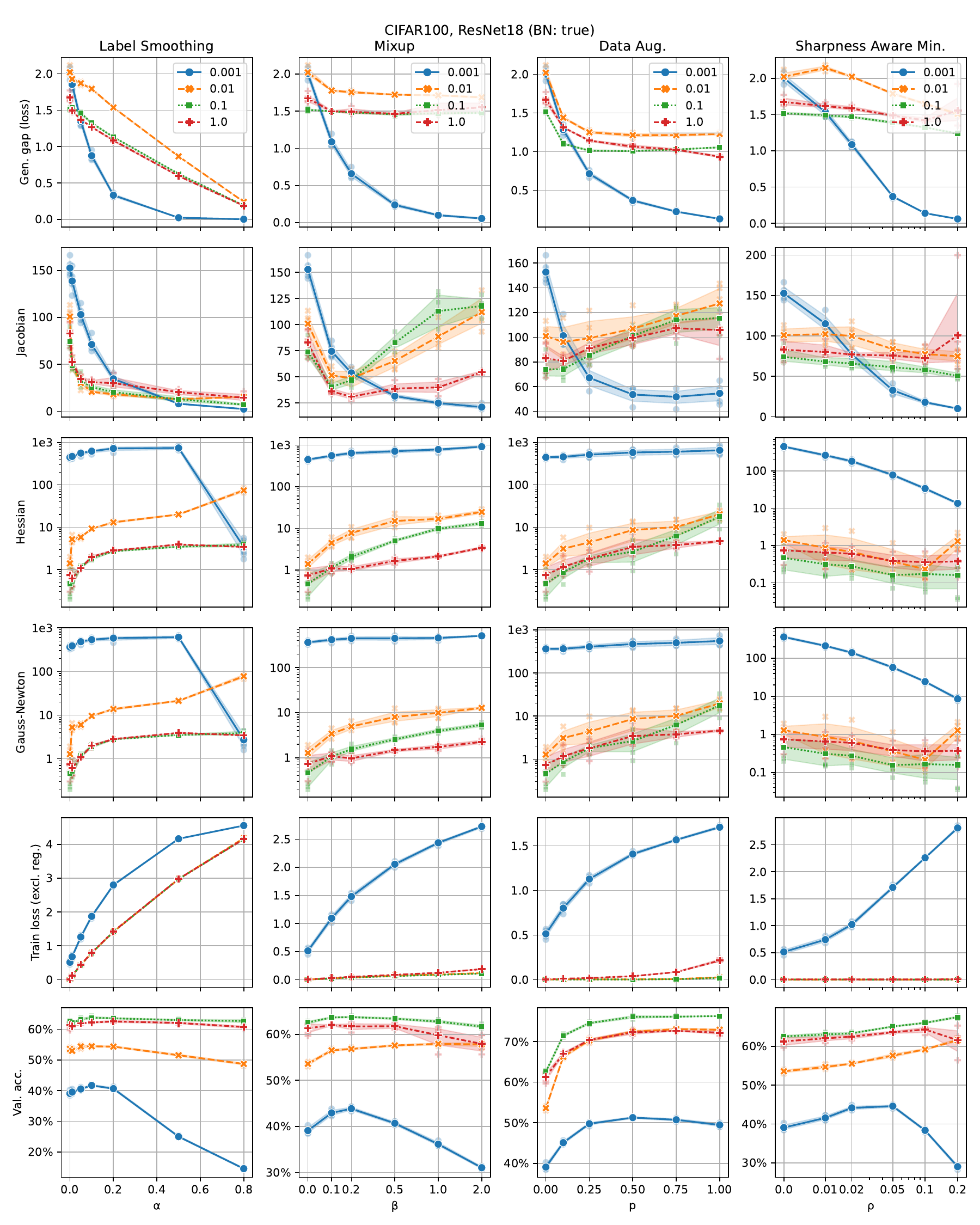}}
\caption{ResNet18 with batch-norm on CIFAR100.}
\end{figure}
\newpage

\section{On mediating factors in the ansatz}\label{app:mediating}

In this section we provide classification and regression experiments to show that the relationship between loss curvature and input-output Jacobians is not always simple.  Recall again Equation \eqref{eq:tangentkernel}:
\begin{equation}
    C\,DF_X\,DF_X^T\,C^T = C\,\bigg(\sum_{l=1}^{L}\bigg(Jf_L\cdots Jf_{l+1}\,Df_l\,Df_l^T\,Jf_{l+1}^T\cdots Jf_L^T\bigg)\bigg)C^T.
\end{equation}
Due to the presence of the parameter derivatives $Df_l$ and the square root $C$ of the Hessian of the cost function, as well as the absence of the first layer Jacobian in Equation \eqref{eq:tangentkernel}, it is not always true that the magnitude of the Hessian and that of the model's input-output Jacobian are correlated.

\subsection{The cost function}\label{sec:vgg_sgd}

We trained a VGG11, again using \url{https://github.com/chengyangfu/pytorch-vgg-cifar10/blob/master/vgg.py} with drouput layers removed and BN layers retained, on CIFAR10 using SGD with a batch size of 128, using differing degrees of label smoothing. Data was standardised so that each RGB channel has zero mean and unit standard deviation over all pixel coordinates and training samples.  In Figure \ref{fig:vgg_sgd1} we plot the sharpness and Jacobian norm every five iterations in the first epoch, observing exactly the same relationship between label smoothing, Jacobian norm and sharpness as predicted in Section \ref{sec:progressivesharpening}. Refer to \texttt{vgg\_sgd.py} in the supplementary material for the code to run these experiments.

\begin{figure}[H]
    \centering
    \setkeys{Gin}{width=\linewidth}
    \begin{subfigure}{0.3\columnwidth}
        \includegraphics{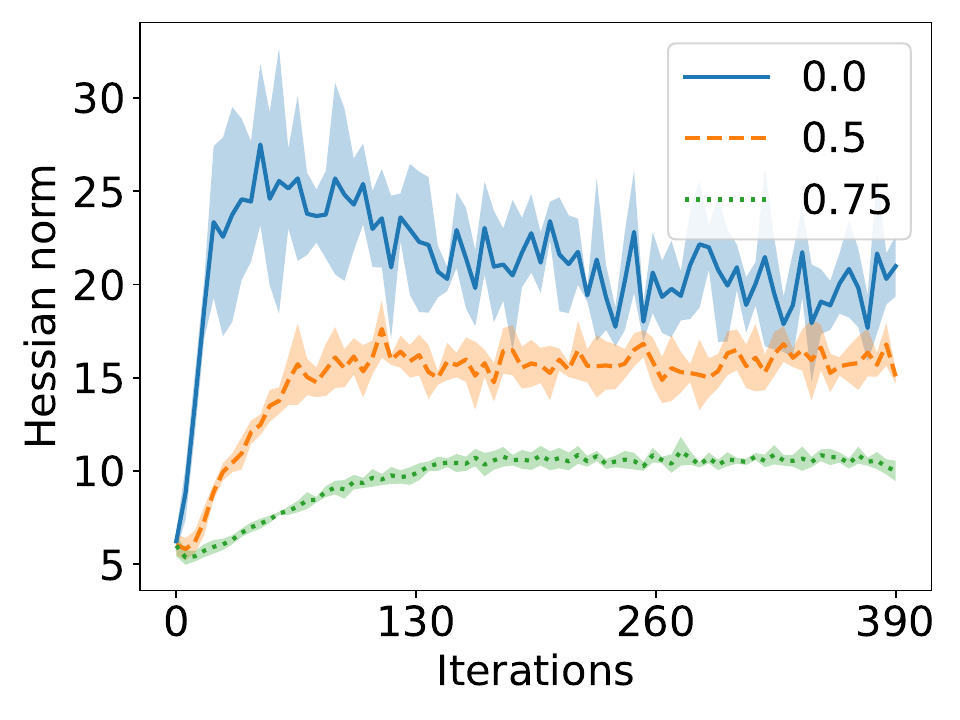}
        \caption{Sharpness}
    \end{subfigure}\hfill
    \begin{subfigure}{0.3\columnwidth}
        \includegraphics{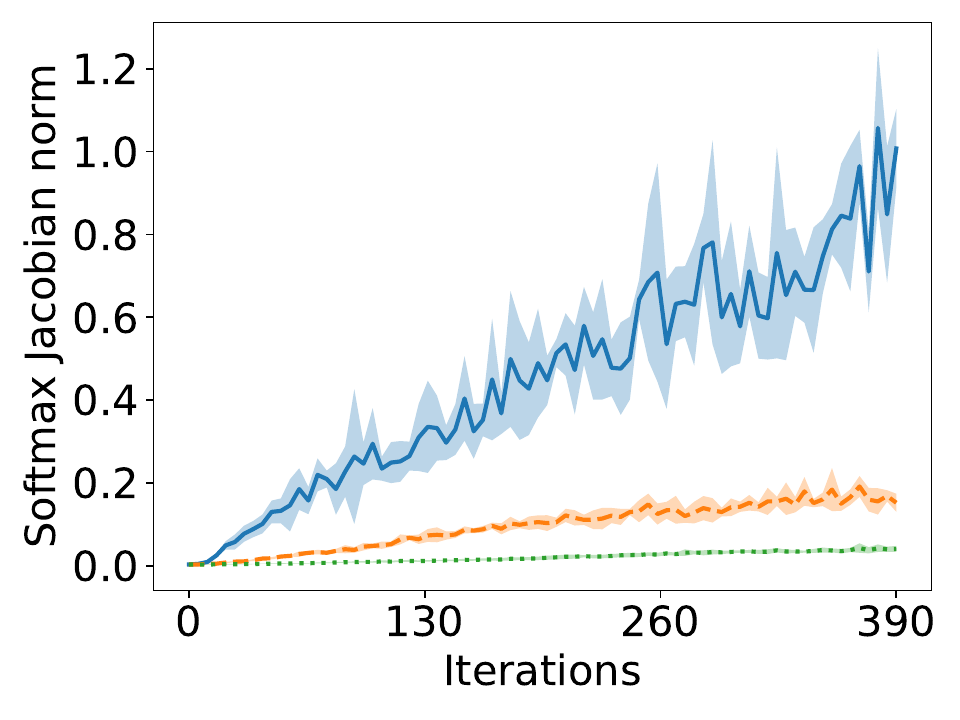}
        \caption{Jacobian norm}
    \end{subfigure}\hfill
    \begin{subfigure}{0.3\columnwidth}
        \includegraphics{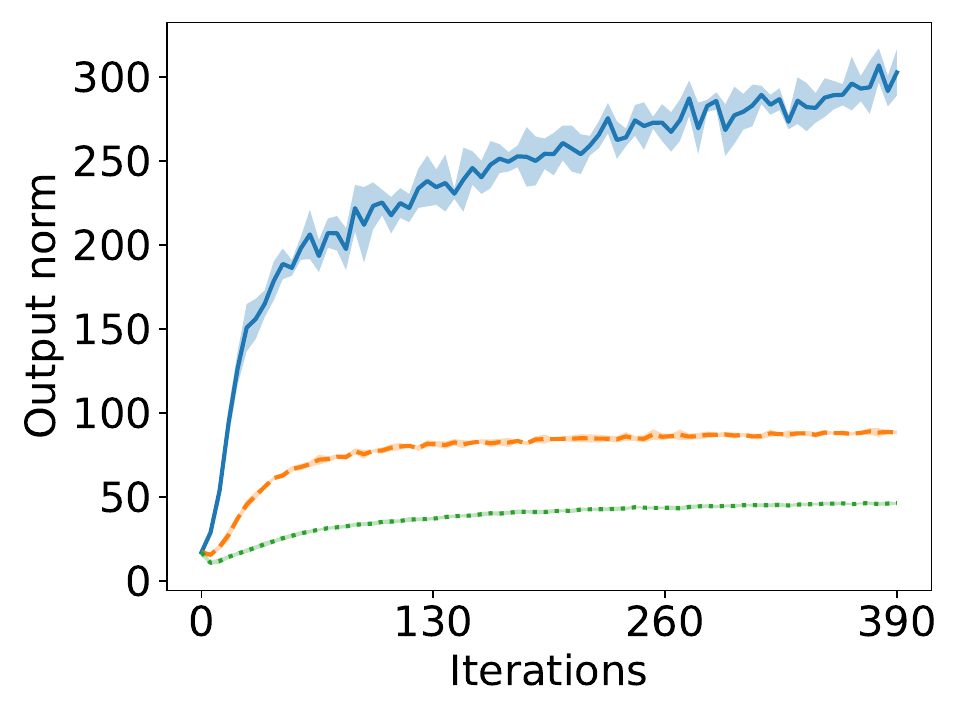}
        \caption{Output norm}
    \end{subfigure}
    \caption{Effect of label smoothing on sharpness, Jacobian norm and output Frobenius norm during the first epoch of SGD. Line style indicates label smoothing. Exactly as in the full batch GD case, the smoother labels are associated with less severe increase of the Jacobian and less severe progressive sharpening.}
    \label{fig:vgg_sgd1}
\end{figure}

Zooming out, however, to the full training run over 90 epochs, Figure \ref{fig:vgg_sgd2} shows that the Hessian ultimately behaves markedly differently.

\begin{figure}[H]
    \centering
    \setkeys{Gin}{width=\linewidth}
    \begin{subfigure}{0.3\columnwidth}
        \includegraphics{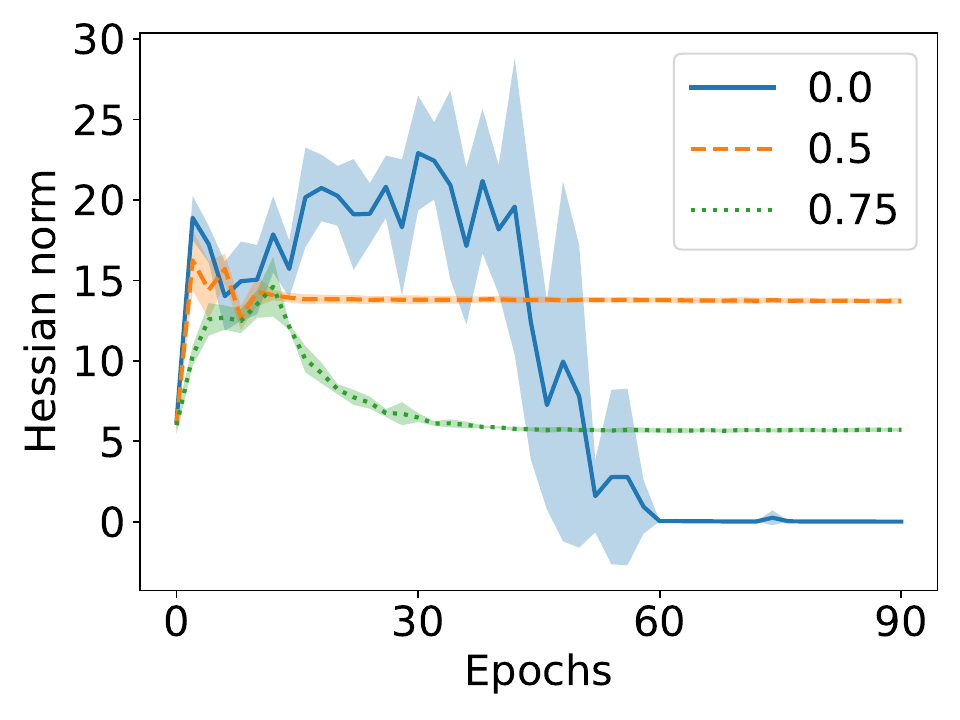}
        \caption{Sharpness}
    \end{subfigure}\hfill
    \begin{subfigure}{0.3\columnwidth}
        \includegraphics{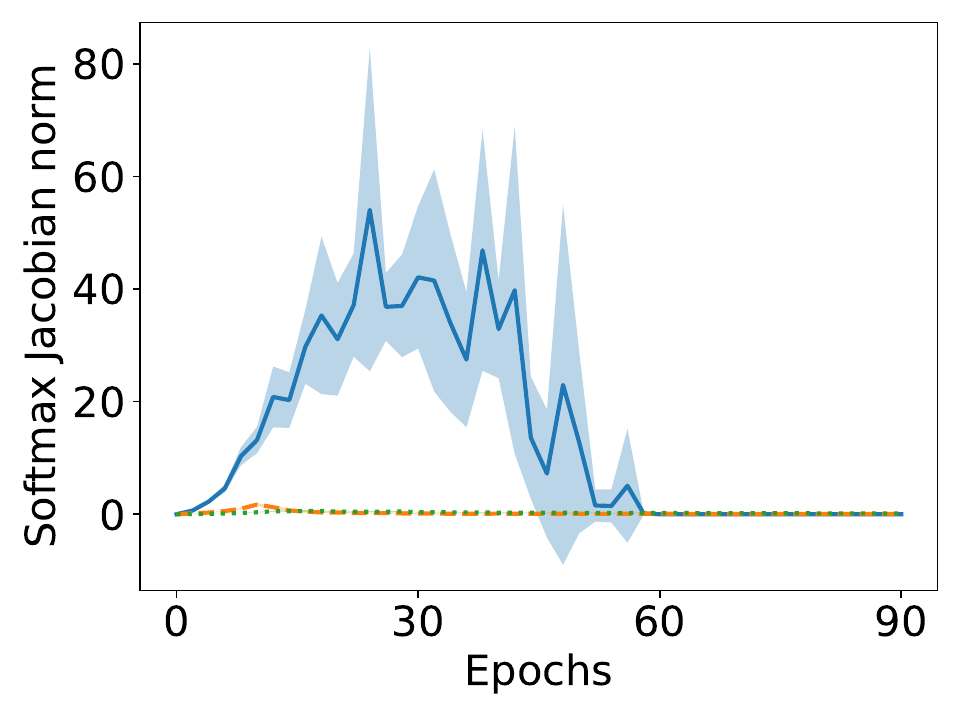}
        \caption{Jacobian norm}
    \end{subfigure}\hfill
    \begin{subfigure}{0.3\columnwidth}
        \includegraphics{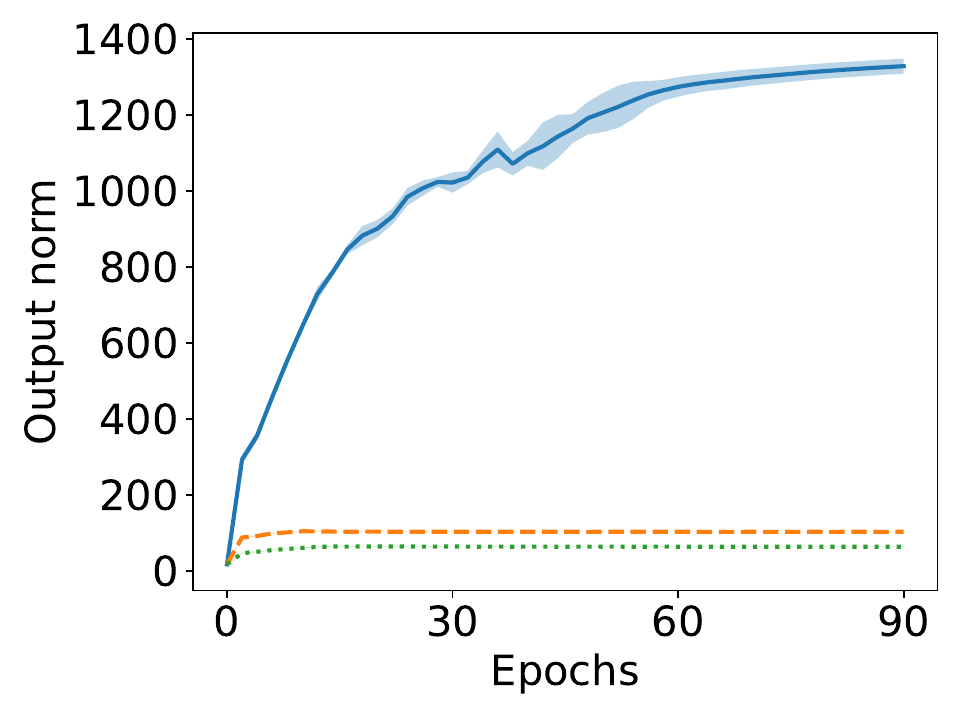}
        \caption{Output norm}
    \end{subfigure}
    \caption{Effect of label smoothing on sharpness, Jacobian norm and output Frobenius norm during 90 epochs of SGD. Line style indicates label smoothing.  When the output norm gets too large, the decay of the $C$ terms in Equation \eqref{eq:tangentkernel} begins to overtake growth in Jacobian norm, ultimately causing the sharpness with unsmoothed labels to collapse to zero, where as sharpness values for smoothed labels plateau at nonzero values.  Note also the the Jacobian norm presented is the Jacobian norm of the softmaxed model: the growth in output norm therefore also drives the Jacobian norm to zero.}
    \label{fig:vgg_sgd2}
\end{figure}
The reason for this is that the norm of the network output grows to infinity in the unsmoothed case (Figure \ref{fig:vgg_sgd2}), and this growth in output corresponds to a vanishing of the $C$ term \cite[Appendix C]{coheneos} corresponding to the cross-entropy cost in Equation \eqref{eq:tangentkernel}.  Note that this growth is also to blame for the eventual collapse of the norm of the Jacobian of the softmaxed model, since the derivative of softmax vanishes at infinity also. This vanishing of the Jacobian should not be interpreted as saying that the \emph{Lipschitz} constant of the model has collapsed (since if this were the case the model would not be able to fit the data), but rather that the maximum Jacobian norm over training samples has ceased to be a good approximation of the Lipschitz constant of the model, due to the Jacobian itself having a large Lipschitz constant (Theorem \ref{thm:jacobianmaximum}).

\subsection{The parameter derivatives}\label{sec:param_deriv}

Recall Theorem \ref{thm:progressivesharpening}. We have already shown that decreasing the distance between labels decreases the severity of Jacobian growth, and, in line with Ansatz \ref{ansatz}, also reduces the severity of progressive sharpening. It seems that one could also manipulate this severity by increasing the distance between input data points, for instance by scaling all data by a constant.
We test this training simple, fully-connected, three layer networks, of width 200, with gradient descent on the first 5000 data points of CIFAR10.  The data is standardised to have componentwise zero mean and unit standard deviation (measured across the whole dataset). The networks are initialised using the uniform distribution on the interval with endpoints $\pm1/\sqrt{in\,features}$ (default PyTorch initialisation) and trained with a learning rate of $0.2$ for 300 iterations using the cross-entropy cost.  Both ReLU and tanh activations are considered. Refer to \texttt{small\_network.py} in the supplementary material for the code we used to run these experiments.

We scale the inputs by factors of 0.5, 1.0 and 1.5, bringing the data closer together at 0.5 and further apart at 1.5.  From Theorem \ref{thm:progressivesharpening}, we anticipate the Jacobian growth to go from most to least severe on these scalings respectively, with sharpness growth behaving similarly according to Ansatz \ref{ansatz}. Surprisingly, we find that while Jacobian growth \emph{is} more severe for the smaller scalings, the sharpness growth is \emph{less}.

The reason for this is the effect of data scaling on the parameter derivatives $Df_l$ in Equation \eqref{eq:tangentkernel}. It is easily computed (cf. \cite{macdonald}) that the parameter derivative $Df_l$ is simply determined by the matrix $f_{l-1}(X)$ of features from the previous layer. In the experimental settings we examined, these matrices are \emph{larger} for the larger scaling values, which pushes the sharpness upwards despite Jacobian norm being smaller. Figures \ref{fig:smallrelu} and \ref{fig:smalltanh} show this phenomenon on ReLU and tanh activated networks respectively.

\begin{figure}[H]
    \centering
    \setkeys{Gin}{width=\linewidth}
    \begin{subfigure}{0.24\columnwidth}
        \includegraphics{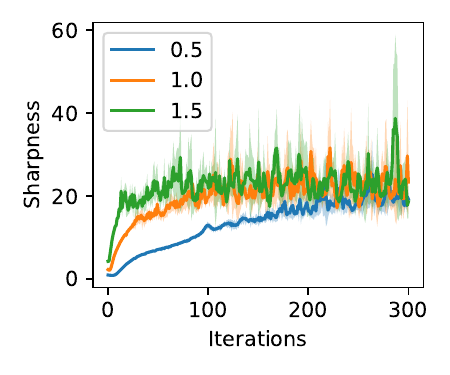}
        \caption{Sharpness}
    \end{subfigure}\hfill
    \begin{subfigure}{0.24\columnwidth}
        \includegraphics{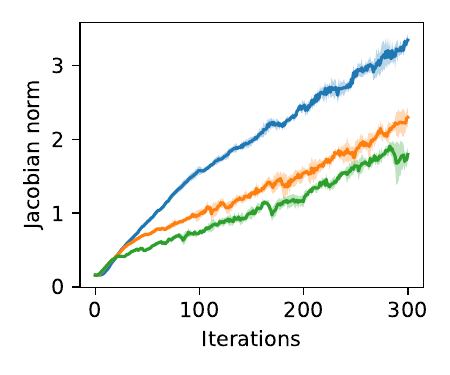}
        \caption{Jacobian norm}
    \end{subfigure}\hfill
    \begin{subfigure}{0.24\columnwidth}
        \includegraphics{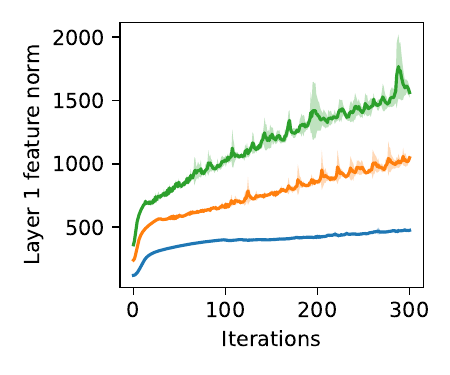}
        \caption{Layer 1 feature norm}
    \end{subfigure}\hfill
    \begin{subfigure}{0.24\columnwidth}
        \includegraphics{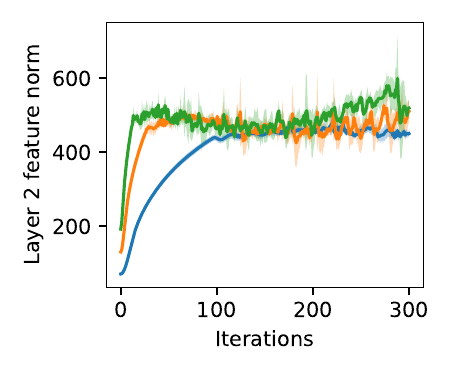}
        \caption{Layer 2 feature norm}
    \end{subfigure}
    \caption{Effect of input scaling on sharpness and Jacobian norm for a ReLU network (5 trials). Line style denotes input scaling factor. Scaling data closer together does cause more severe increase in Jacobian norm, but corresponds to \emph{less} severe increase in sharpness. This is due to the data scaling increasing the spectral norm of the feature maps.}
    \label{fig:smallrelu}
\end{figure}

\begin{figure}[H]
    \centering
    \setkeys{Gin}{width=\linewidth}
    \begin{subfigure}{0.24\columnwidth}
        \includegraphics{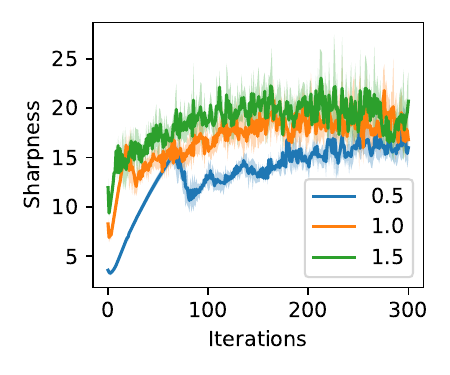}
        \caption{Sharpness}
    \end{subfigure}\hfill
    \begin{subfigure}{0.24\columnwidth}
        \includegraphics{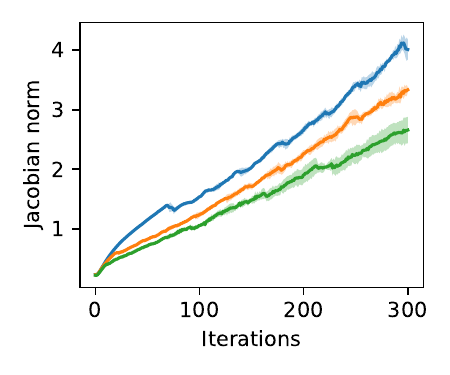}
        \caption{Jacobian norm}
    \end{subfigure}\hfill
    \begin{subfigure}{0.24\columnwidth}
        \includegraphics{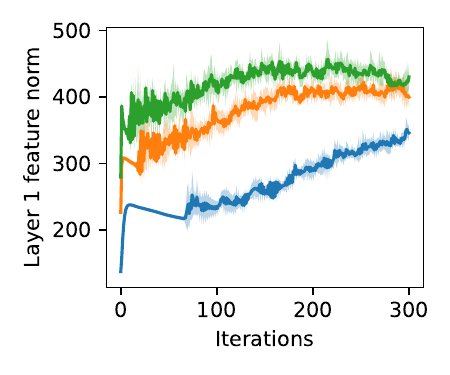}
        \caption{Layer 1 feature norm}
    \end{subfigure}\hfill
    \begin{subfigure}{0.24\columnwidth}
        \includegraphics{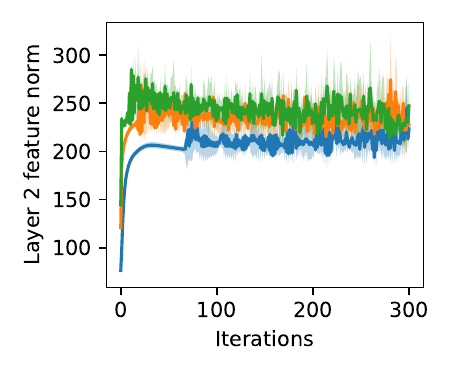}
        \caption{Layer 2 feature norm}
    \end{subfigure}
    \caption{Effect of input scaling on sharpness and Jacobian norm for a tanh network (5 trials). Line style denotes input scaling factor. Scaling data closer together does cause more severe increase in Jacobian norm, but corresponds to \emph{less} severe increase in sharpness. This is due to the data scaling increasing the spectral norm of the feature maps.}
    \label{fig:smalltanh}
\end{figure}

\subsection{The absence of the first layer Jacobian}\label{sec:firstlayerjac}

We replicate the experiments of Section 8 in \cite{sameera}, where it is demonstrated that flatness of minimum is not correlated with model smoothness in a simple regression task, and point to a possible explanation of this within our theory.

In detail, a four layer MLP with either Gaussian or ReLU activations is trained to regress 8 points in $\RB^2$.  Each network is trained from a \emph{high frequency} initialisation (obtained from default PyTorch initialisation on the Gaussian network, and by pretraining on $\sin(6\pi x)$ for the ReLU network) to yield a non-smooth fit of the target data, and a \emph{low frequency} initialisation (a wide Gaussian distribution for the Gaussian-activated network, and the default PyTorch initialisation for the ReLU network) to achieve a smooth fit of the data (see Figure \ref{fig:sameera}). Refer to \texttt{coordinate\_network.py} in the supplementary material for the code we used to run this experiment.

\begin{figure}[H]
    \centering
    \setkeys{Gin}{width=\linewidth}
    \begin{subfigure}{0.24\columnwidth}
        \includegraphics{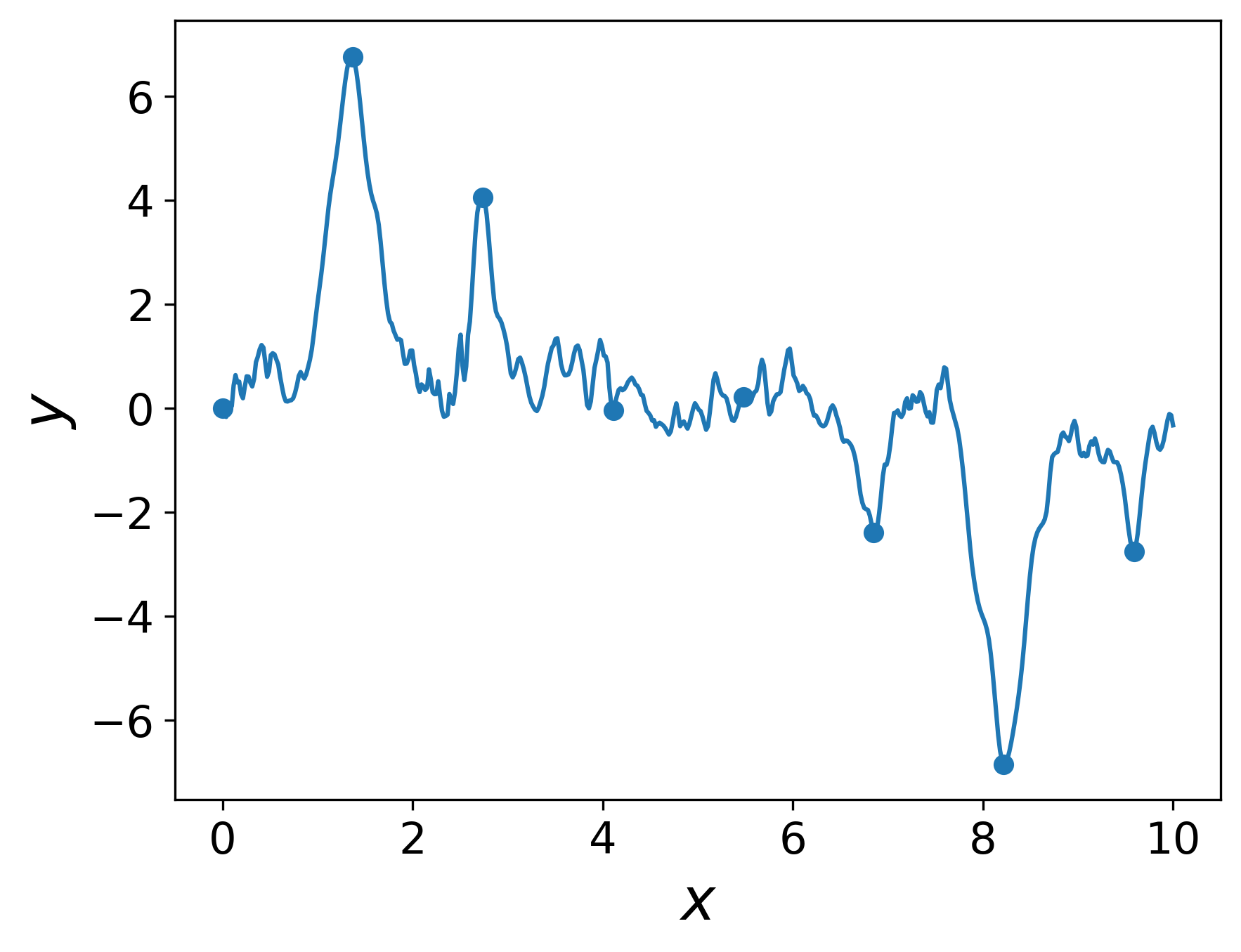}
        \caption{Gaussian, high freq.}
    \end{subfigure}\hfill
    \begin{subfigure}{0.24\columnwidth}
        \includegraphics{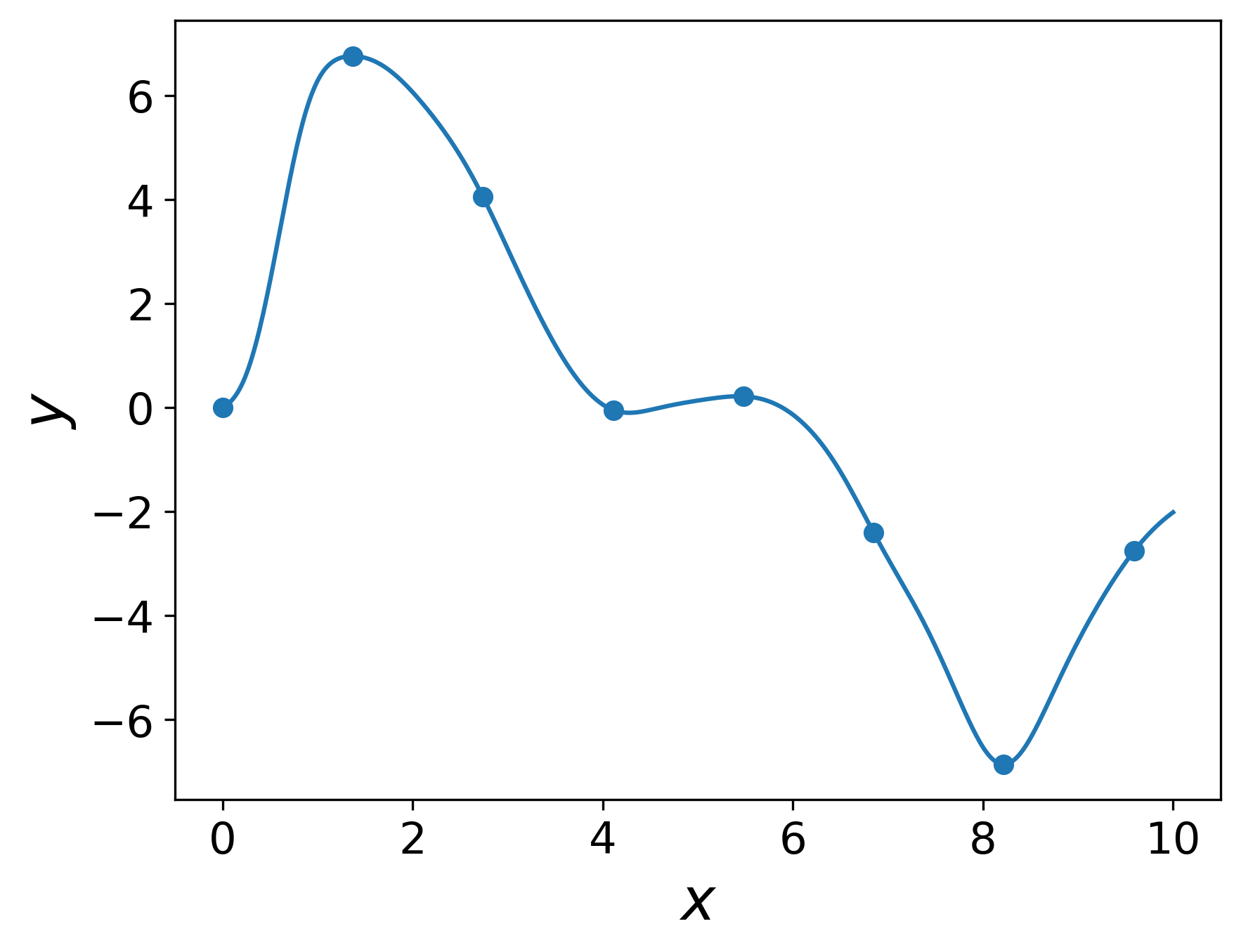}
        \caption{Gaussian, low freq.}
    \end{subfigure}\hfill
    \begin{subfigure}{0.24\columnwidth}
    \includegraphics{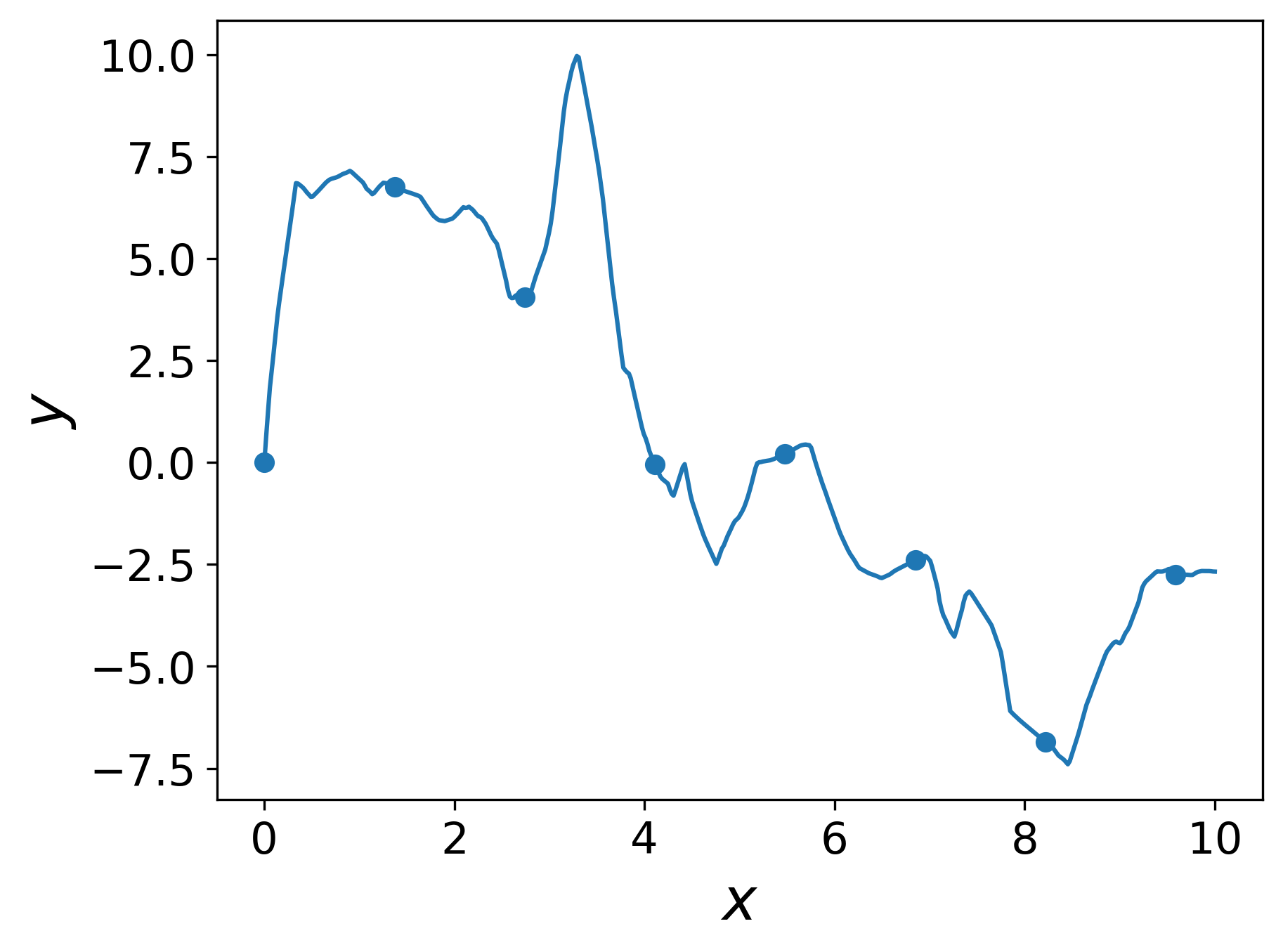}
    \caption{ReLU, high freq.}
    \end{subfigure}
    \begin{subfigure}{0.24\columnwidth}
    \includegraphics{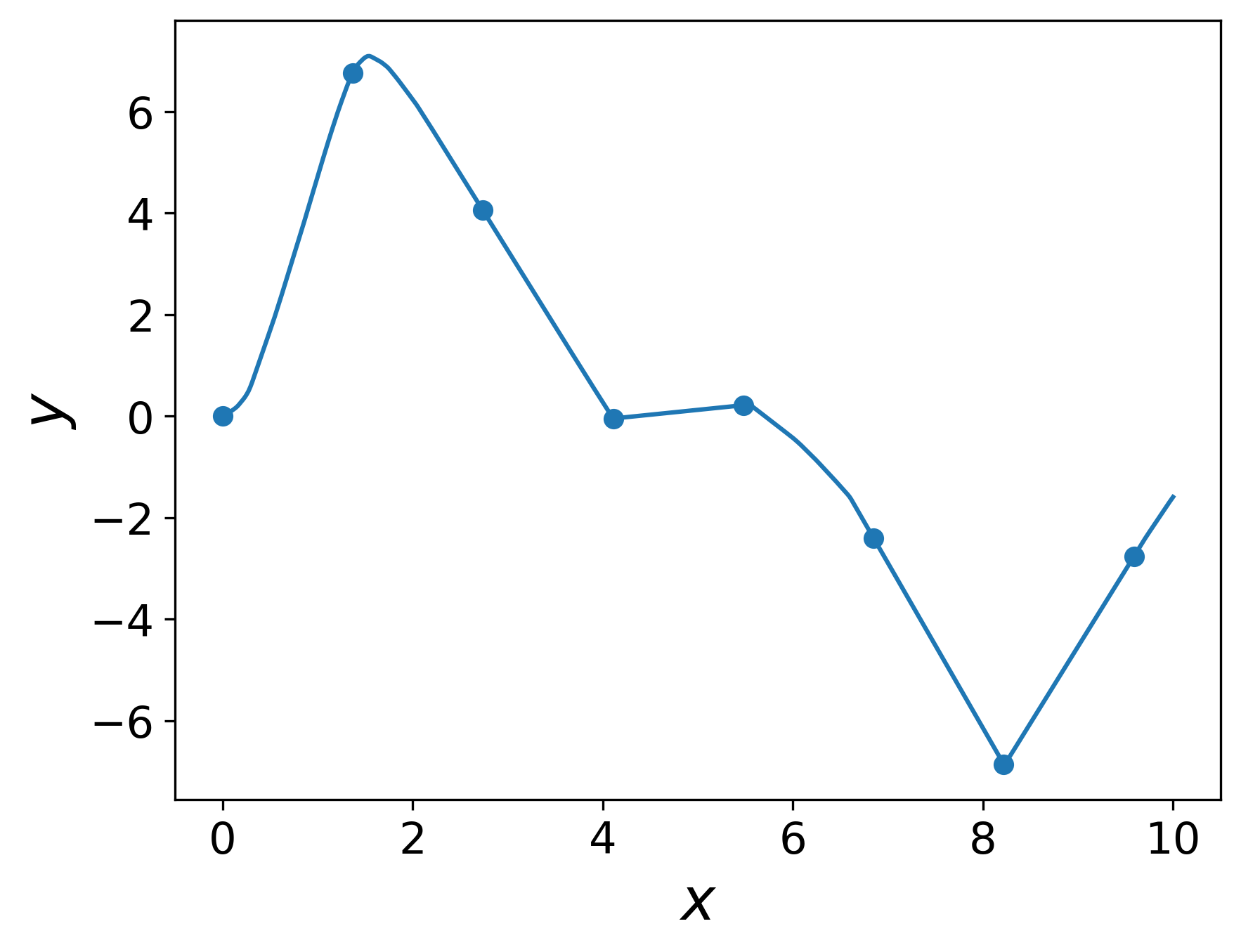}
    \caption{ReLU, low freq.}
    \end{subfigure}
    \caption{Interpolations of 8 points by networks started from high frequency and low frequency initialisations.  The smoother interpolations have lower Jacobian norm over the training data.}
    \label{fig:sameera}
\end{figure}

The models were trained with gradient descent using a learning rate of $1e-4$ and momentum of 0.9, for 10000 epochs for the Gaussian networks and 100000 epochs for the ReLU networks. The pretraining for the high frequency ReLU initialisation was achieved using Adam with a learning rate of $1e-4$ and default PyTorch settings.

\begin{table}[H]
  \caption{Norm values for regression networks, replicating result of \cite{sameera}}
  \label{table:sameera1}
  \centering
  \begin{tabular}{lll}
    \toprule
    & \multicolumn{2}{c}{Gaussian activated network} \\
    Norm   & High freq.     & Low freq. \\
    \midrule
    Jacobian norm            & $326.80\pm182.51$ & $84.07\pm49.30$\\
    sharpness            & $13517.61\pm3871.97$ & $18353.54\pm5751.98$\\
    First layer weight norm            & $9.09\pm0.34$ & $0.56\pm0.02$ \\
    \bottomrule
  \end{tabular}
\end{table}

\begin{table}
  \caption{Norm values for regression networks, replicating result of \cite{sameera}}
  \label{table:sameera2}
  \centering
  \begin{tabular}{lll}
    \toprule
    & \multicolumn{2}{c}{ReLU activated network} \\
    Norm   & High freq.     & Low freq. \\
    \midrule
    Jacobian norm            & $121.59\pm80.79$ & $42.96\pm7.82$\\
    sharpness            & $14211.05\pm4767.26$ & $9922.45\pm3330.37$\\
    First layer weight norm            & $9.61\pm0.28$ & $9.56\pm0.16$ \\
    \bottomrule
  \end{tabular}
\end{table}

Tables \ref{table:sameera1} and \ref{table:sameera2} record the means and standard deviations of loss Hessian and Jacobian norms of the Gaussian and ReLU networks over 10 trials.  As in \cite{sameera}, we find that while the smoothly interpolating ReLU models on average land in flatter minima than the non-smoothly interpolating models, the opposite is true for the Gaussian networks.

Keeping in mind that the high variances in these numbers make it difficult to come to a conclusion about the trend, we propose a possible explanation for why the Gaussian activated networks do not appear to behave according to Ansatz \ref{ansatz}.  From Equation \eqref{eq:tangentkernel}, the Hessian cannot be expected to relate to the Jacobian of the first layer. Note now the discrepancy between the first layer weight norms in the low frequency versus high frequency fits with the Gaussian networks, which does not occur for the ReLU networks. We believe this to be the cause of the larger Hessian for the low-frequency Gaussian models: with such a small weight matrix in the first layer, all higher layers must be relied upon to fit the data, making their Jacobians higher than they would otherwise need to be. These larger Jacobians would then push the Hessian magnitude higher by Equation \eqref{eq:tangentkernel}.

\section{Relationship to the parameter-output Jacobian}\label{app:leeetal}

In \cite{leeetal}, the \emph{parameter-output} Jacobian is advocated as the mediating link between loss sharpness and generalisation. Extensive experiments are provided to support this claim. In particular, it is shown that regularising by the Frobenius norm of the parameter-output Jacobian during training is sufficient to alleviate the poor generalisation of (toy) networks trained using a small learning rate. However, no theoretical reasons are given for why the parameter-output Jacobian should be related to generalisation.

We expect that our Theorem \ref{thm:generalisationbound} could provide the theoretical link missing in \cite{leeetal}. Indeed, the term in brackets in our Equation \eqref{eq:tangentkernel} is precisely the Gram matrix of the parameter-output Jacobian. Thus, our Ansatz \ref{ansatz} states that regularisation of the parameter-output Jacobian will also regularise the input-output Jacobian. If this is true, then we should see similar improvements in network performance when training with a small learning rate using an \emph{input-output} Jacobian regulariser to those seen in \cite{leeetal} when training with the \emph{parameter-output} Jacobian regulariser. Moreover, one should see that these improvements are correlated to input-output Jacobian norm in all cases. Figure \ref{fig:leeetal} demonstrates that this is indeed the case, using the same settings on data and models as were used for the regulariser experiments in Section \ref{sec:generalisation}.

\begin{figure}
\centering
\begin{subfigure}[b]{1.0\textwidth}
\includegraphics[scale = 0.45]{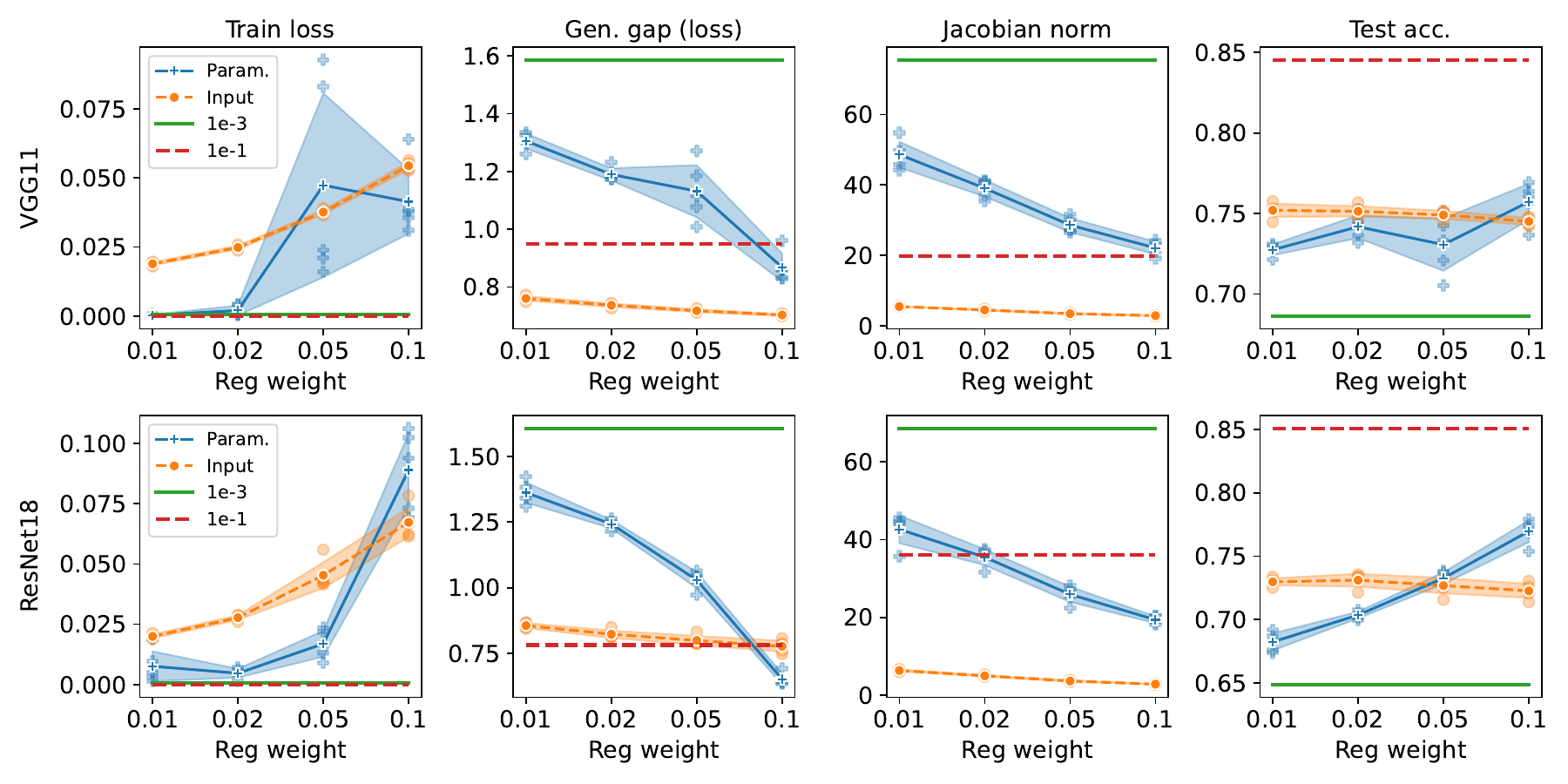}
    \caption{CIFAR10}
\end{subfigure}

\begin{subfigure}[b]{1.0\textwidth}
    \includegraphics[scale = 0.45]{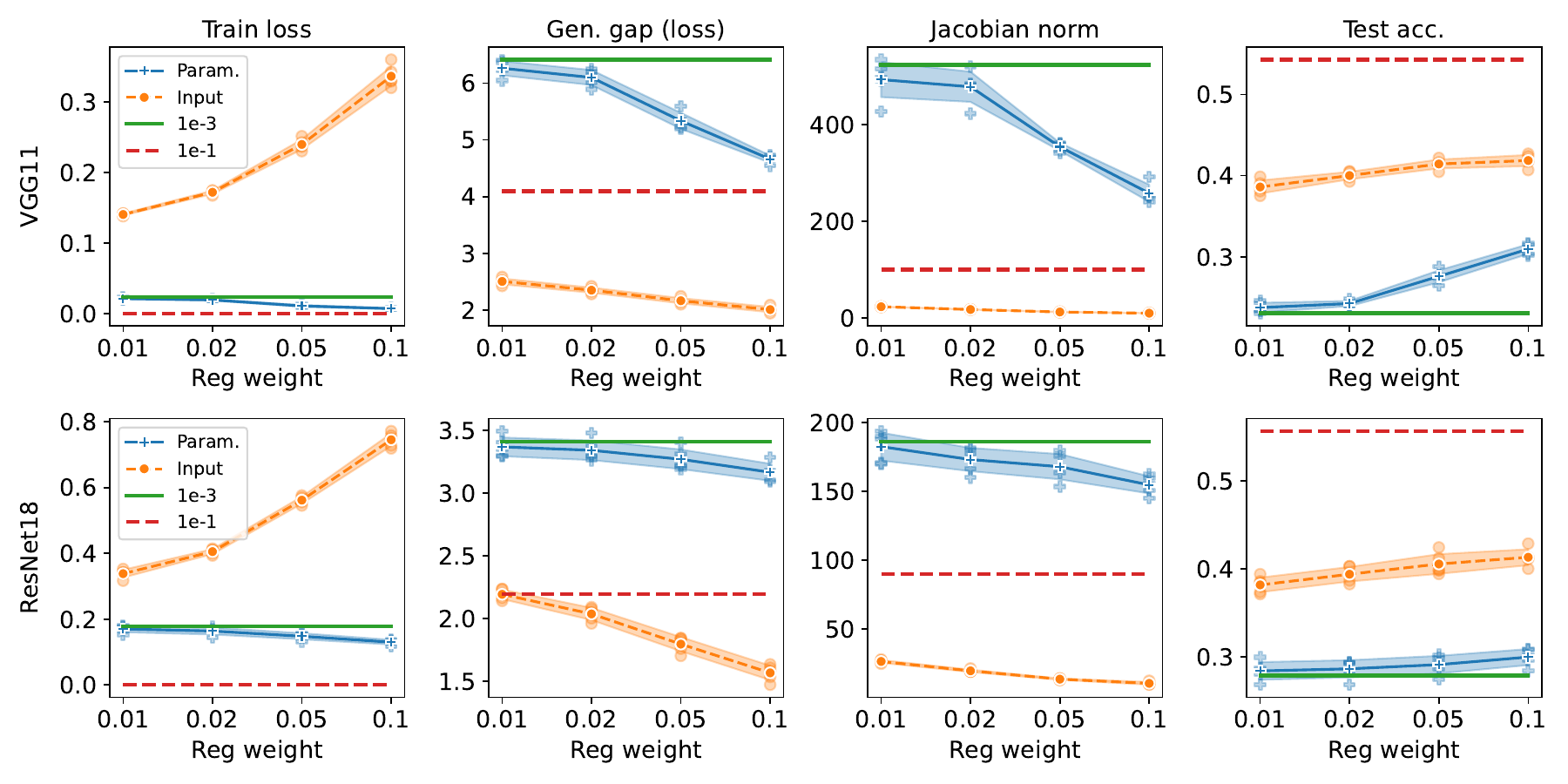}
    \caption{CIFAR100} 
\end{subfigure}
\caption{All models trained for 90 epochs. $x$-axis denotes regularisation coefficient, and ``Jacobian norm" refers to spectral norm. ``ref 1e-1" and ``ref 1e-3'' are means over 5 trials of unregularised networks (regularisation coefficient set to zero) trained at learning rates of 1e-1 and 1e-3 respectively, shown for reference. The unregularised models trained with learning rate 1e-1 had their learning rate decayed by a factor of 10 every 30 epochs, so that their final learning rates were 1e-3. Learning rate scheduling was not used for any other trials.``Param." is trained with the parameter-output regulariser of Lee et al, while ``Input" is trained by regularising the minibatch Jacobian Frobenius norm using a similar approximation as Lee et al, namely $\|J\|_{F}\approx \|u^T J\|_2$ for $u$ sampled uniformly from the $NC - 1$-sphere, where $N$ is the batch size and $C$ is the number of channels. 5 trials shown for the regularised trials, as well as means and 1 standard deviation shaded. Note the consistent downwards trend of both Jacobian norm and generalisation gap in all cases, as is consistent with our theory. Small Jacobian norm with large generalisation gap suggest, according to our theory, that the Jacobian of the model has ceased to be a good approximation of the Lipschitz constant of the model.}\label{fig:leeetal}
\end{figure}

\end{document}